%% file: main.tex
\documentclass{article} %
\usepackage{iclr2027_conference,times}

\input{math_commands.tex}

\usepackage{amsmath,amssymb,amsthm}
\usepackage{booktabs}
\usepackage{array}
\usepackage{multirow}
\usepackage{graphicx}
\usepackage{microtype}
\usepackage{xcolor}
\usepackage{tikz}
\usepackage{pgfplots}
\usepackage{makecell}
\usepackage{tabularx}
\pgfplotsset{compat=1.18}
\usetikzlibrary{arrows.meta,positioning,fit,backgrounds,patterns,decorations.pathreplacing,calc}
\usepackage{hyperref}
\usepackage{url}
\usepackage{placeins}
\hypersetup{hidelinks}

\theoremstyle{definition}
\newtheorem{proposition}{Proposition}
\newtheorem{corollary}{Corollary}

\definecolor{cOracle}{HTML}{222222}
\definecolor{cPKL}{HTML}{1F6FB4}
\definecolor{cTIP}{HTML}{E8762C}
\definecolor{cRisk}{HTML}{C0392B}
\definecolor{cDE}{HTML}{2E8B57}
\definecolor{cRand}{HTML}{8A8A8A}
\definecolor{cHelp}{HTML}{2E8B57}
\definecolor{cHarm}{HTML}{C0392B}
\definecolor{cPurple}{HTML}{7B4EA3}
\definecolor{cBrown}{HTML}{8C6D46}
\newcommand{\dumb}[4]{%
  \draw[#4, line width=1.1pt, -{Latex[length=1.6mm]}, shorten >=1.6pt] (axis cs:#2,#1) -- (axis cs:#3,#1);
  \fill[white] (axis cs:#2,#1) circle (2.0pt); \draw[#4, line width=0.9pt] (axis cs:#2,#1) circle (2.0pt);
  \fill[#4] (axis cs:#3,#1) circle (2.0pt);}
\newcommand{\dumbinf}[3]{%
  \fill[white] (axis cs:#2,#1) circle (2.0pt); \draw[#3, line width=0.9pt] (axis cs:#2,#1) circle (2.0pt);
  \node[#3, font=\scriptsize, anchor=west, inner sep=1pt, xshift=2.5pt] at (axis cs:#2,#1) {inf.};}

\newcommand{\bench}{DEEP}

\newcommand{\ndg}{\mathrm{nDG}}

\title{The Decision Value of Perception Compute}

\author{Hoang Pham Cong \\
Hanoi University of Science \\
and Technology (HUST) \\
Hanoi, Vietnam
\And
Ho Viet Duc Luong \\
Hanoi University of Science \\
and Technology (HUST) \\
Hanoi, Vietnam}
\iclrfinalcopy
\hypersetup{pdftitle={The Decision Value of Perception Compute},pdfauthor={Hoang Pham Cong, Ho Viet Duc Luong}}

\begin{document}
\maketitle
\lhead{Preprint}\renewcommand{\headrulewidth}{0.4pt}\thispagestyle{fancy}

\begin{abstract}
Adaptive perception spends extra computation where perception is expected to improve, but when perception feeds a decision system, a better perception output need not produce a better decision. We define the \emph{decision value} of perception compute as the single-step reduction in downstream loss from escalating an input from a cheap to an expensive perception mode. Because this value can be negative, escalating every input is a baseline rather than an upper bound, and the reference for allocation is a budget-constrained decision oracle. We introduce \bench{} ({\textbf{D}ecision \textbf{E}valuation for \textbf{E}scalated \textbf{P}erception}), a benchmark that scores allocators, which must decide before escalating, against this oracle under selection, latency and energy budgets, and charges each allocator for its own computation. {With deployed geometry on KITTI, nuScenes, nuPlan and NAVSIM, across braking and trajectory controllers and published planners, 34--62\% of the escalations that change the downstream loss make it worse. On all 850 nuScenes scenes, escalating every frame does not reliably reduce the braking loss, whereas the budgeted oracle removes 31\% of it while escalating at most one frame in five. On nuScenes, perception-level gains disagree in sign with decision value on 25--51\% of the frames where both are non-zero, and, in a post hoc comparison, a gate trained on decision value realizes more gain on test scenes than the same gate trained on published routing objectives. Learned allocators recover part of the oracle's gain (on nuScenes braking, also on unseen recording logs), but most lose their advantage once their own cost is charged. \bench{} makes this gap measurable; closing it, and linking single-step value to closed-loop driving, remain open.}
\end{abstract}

\section{Introduction}
Adaptive perception rests on a simple trade-off: more computation can improve perception, but it need not be spent on every input.
Existing methods therefore look for the inputs where extra computation is worthwhile, using signals such as uncertainty, difficulty, criticality or expected perception gain \citep{weiss2013dynamic,heo2022rtscale,kang2022dnnsam,liu2022selfcueing}.
Selective routing between a weak and a strong detector goes further: it sends an input to the strong model only when this is expected to improve prediction quality, recognizing that the stronger model is not better on every input \citep{qiu2024edge,geng2026budget}.
In all these approaches, extra perception compute is valued by the perception improvement it brings.

This criterion is incomplete when perception feeds a downstream decision system.
A more expensive detector may recover a missed object without changing the action, recover one that triggers a needed action, or add a false positive that triggers an unnecessary one.
In Figure~\ref{fig:concept}, the expensive mode recovers a missed bus in two nuScenes frames. In one, the braking controller decelerates as the scene requires; in the other, it brakes hard for a bus outside its path.
Whether extra compute is worthwhile therefore depends on its effect on the decision, for which perception gain is only an indirect signal.
We call this effect the \emph{decision value} of perception compute: the reduction in downstream loss from escalating an input to the expensive mode{, here the same detector at a higher input resolution}.

Such cases are not isolated. On KITTI \citep{geiger2012kitti} and nuScenes \citep{caesar2020nuscenes}, a substantial fraction of the escalations that change the downstream loss make it worse; harm also occurs for a published planner on nuPlan, and two downstream systems can assign opposite values to the same perception change. {The same holds on all 850 nuScenes scenes and for two planners on NAVSIM.} Decision value is therefore neither determined by perception improvement nor intrinsic to the perception model.

Because decision value can be negative, escalating every input absorbs harmful escalations along with beneficial ones: it is a practical baseline when the budget allows it, not an upper bound on allocation. The right reference is a \emph{budget-constrained decision oracle}, which spends the budget on the most valuable escalations and skips harmful ones.
The oracle cannot be deployed, because decision value is known only after both modes have run and their decisions have been scored against the scene. A deployable \emph{allocator} must decide before escalation, from the input, the cheap output and available context, and must pay for the computation it uses.

We introduce \bench{} (\textbf{Decision Evaluation for Escalated Perception}), a benchmark for this pre-escalation allocation problem.
\bench{} scores allocators by the decision value they realize, relative to the oracle, under selection, latency and energy budgets, and charges each allocator for its own computation.
Its core track uses detection-based driving pipelines on KITTI and nuScenes; its external track uses published planners on nuPlan. We make three contributions.
\begin{itemize}
\itemsep0pt
\item \emph{Formulation.} We define the {single-step} decision value of perception compute and measure allocation against a budget-constrained decision oracle, with all-full inference as a baseline.
\item \emph{Benchmark.} \bench{} caches perception outputs, decision values and costs, so evaluating a new allocator needs only its pre-escalation scores and, for measured budgets, its cost.
\item \emph{Findings.} {Under deployed geometry, escalation harms 34--62\% of the inputs it affects, and decision value depends on the downstream system. On 850 nuScenes scenes the oracle at a 20\% budget removes 31\% of the braking loss where all-full inference does not reliably reduce it. Ranking by perception gain can select worse allocators, and learned allocators capture part of the oracle's gain but keep an advantage after cost only at small overhead.}
\end{itemize}

\begin{figure}[!t]
\centering
\includegraphics[width=\linewidth]{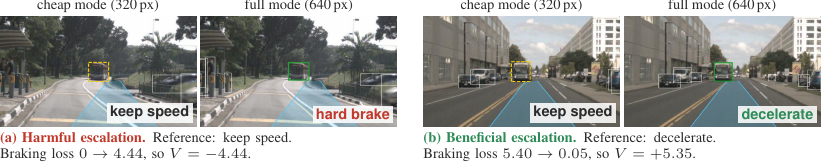}
\caption{\textbf{Better detection need not improve the decision.} Two nuScenes frames under one braking controller (monocular geometry, YOLOv8s). The full mode recovers a bus that the cheap mode misses (green box; dashed yellow: the missed reference object; shaded: the driving corridor); labels give each mode's command.
\textbf{(a)} The bus lies outside the corridor, but its image-based position puts it inside, so the controller brakes hard. \textbf{(b)} The bus lies inside, and the controller decelerates. $V$: the reduction of the braking loss (Eq.~\ref{eq:intro-value}).}
\label{fig:concept}
\end{figure}
\section{Related Work}
\paragraph{Adaptive inference and task-aware computation.}
Adaptive inference varies computation across inputs: dynamic model selection trades accuracy against computation \citep{weiss2013dynamic}, real-time detectors adapt image scale or prioritize safety-critical regions \citep{heo2022rtscale,kang2022dnnsam,liu2022selfcueing}, LiDAR pipelines route frames through shallow or deep processing \citep{kim2026thinkasneeded}, and robot policies condition computation on the task state \citep{spvla2026,actionaware2026,dasip2025}. These methods are evaluated by the accuracy or task success they retain at reduced cost; \bench{} instead values an escalation by its effect on the loss of a separate decision system. {Streaming perception scores detections at query time, so latency makes outputs stale, and it schedules which frames to process \citep{li2020streaming}. Chanakya learns runtime settings such as input resolution from rewards that balance accuracy and latency \citep{ghosh2023chanakya}. Both judge these choices by perception quality. \bench{} values the decision, and its budget is an average rather than a per-frame deadline, so staleness from latency lies outside decision value.}

\paragraph{Selective offloading and routing.}
The closest work routes inputs between a weak and a strong detector. Adaptive Feeding classifies each image as easy or hard and sends it to a fast or an accurate detector \citep{zhou2017adaptive}. ORIC estimates from the weak detector's results a per-image reward, the gain in contextual mAP when the strong detector's output replaces the weak one's, and offloads images under a budget \citep{qiu2024edge}. Budget-adaptive routing predicts from raw pixels whether swapping in the strong detector's output raises dataset-wide AP, which lets the weak pass be skipped, and charges the predictor to the compute budget \citep{geng2026budget}. Both find that reward-guided routing can exceed the strong detector's mAP; within one network, extra depth can likewise turn a correct prediction into an error \citep{kaya2019overthinking}. \bench{} shares this allocation structure: a signed per-input value, estimated before escalation and selected under a budget. What differs is the value: detection quality in detector routing, the loss of a downstream decision system in \bench{}. On nuScenes, four per-input perception-loss gains disagree in sign with decision value on 25--51\% of the frames where both are non-zero, and the same perception change can have opposite values for two downstream systems. Routing between language models \citep{ong2025routellm}, selective prediction \citep{geifman2017selective,galil2023selective}, cost-sensitive offloading \citep{moothedath2026offloading} and learning to defer \citep{mozannar2020defer} likewise define utility on the prediction.

\paragraph{Planning-aware perception evaluation.}
Another line of work evaluates perception by its downstream consequences. PKL scores a detector by the divergence between a learned planner's outputs given reference and given detected objects \citep{philion2020pkl}. Planning-aware metrics evaluate detection and prediction by their effect on downstream planning \citep{ivanovic2022planningaware}. TIP models the planner as an expected-utility maximizer and analyzes how perception noise affects its objective \citep{li2023tip}, and task-aware risk estimation assesses the risk that a detected perception failure poses to the motion plan \citep{antonante2023taskaware}. These methods score a perception output or failure once it is available; \bench{} asks for the marginal downstream value of one additional computation on one input, under a budget, before that computation runs.

\paragraph{Value of information and computation.}
Valuing information or computation by its effect on decisions is classical. Information value theory prices the resolution of uncertainty by the decisions it improves \citep{howard1966information}. Rational metareasoning defines the value of a computation as the expected improvement in decision quality it yields, net of its cost \citep{russell1991principles,lieder2017strategy}. Active perception chooses sensing actions for the task at hand \citep{bajcsy1988active}, including perceptual attention in robot driving \citep{reece1995control}, and selective perception runs the analyses whose expected value of information justifies their cost \citep{oliver2005selective}. Our formulation operationalizes these ideas for cheap-to-full escalations of learned perception on a fixed input: the value is the realized change in downstream loss, the reference is a budget-constrained oracle, and the allocator commits before escalation. A negative value does not contradict the non-negative expected value of information: decision value is realized on a single input, the full output replaces rather than supplements the cheap one, and a fixed decision rule consumes it, not a decision maker that optimally combines both. {We therefore claim no new concept of value. What is new is the measurement of how often realized value is negative across detectors and downstream systems, and a benchmark that charges allocators their own cost against a budgeted oracle.}
\section{Decision Value of Perception Compute}
\label{sec:formulation}
Let $\mathbf{x}_i$ be an input and $\mathbf{s}_i$ its reference (ground-truth) scene state.
The cheap and full perception modes produce
$\mathbf{z}_i^{\mathsf{c}}=f_{\mathsf{c}}(\mathbf{x}_i)$ and
$\mathbf{z}_i^{\mathsf{f}}=f_{\mathsf{f}}(\mathbf{x}_i)$.
A downstream system $q$ consists of a decision rule $\pi_q$ and a downstream loss $\mathcal{L}_q$.
\emph{Escalating} input $i$ means using the full mode instead of the cheap one, and its \emph{decision value} is
\begin{equation}
V_i^q
=
\mathcal{L}_q\!\left(\pi_q(\mathbf{z}_i^{\mathsf{c}}),\mathbf{s}_i\right)
-
\mathcal{L}_q\!\left(\pi_q(\mathbf{z}_i^{\mathsf{f}}),\mathbf{s}_i\right),
\label{eq:intro-value}
\end{equation}
which is positive when the full mode reduces the downstream loss, zero when it has no effect on the loss, and negative when it increases the loss. In Figure~\ref{fig:concept}a, $V_i^q=-4.44$: the full output makes the braking controller brake hard where the scene requires keeping speed. With monocular geometry, only 273 of the 3{,}376 nuScenes frames have $V_i^q\neq0$ under this controller (Table~\ref{tab:ceiling}).
$V_i^q$ is a gross value: the cost of computation enters through the budget below.
We call running every input in the cheap mode \emph{all-cheap} and escalating every input \emph{all-full}; gains are reported as reductions of the all-cheap loss.
A per-input perception loss $\mathcal{E}$ defines a perception gain
$G_i^{\mathcal{E}}
=
\mathcal{E}(\mathbf{z}_i^{\mathsf{c}},\mathbf{s}_i)
-
\mathcal{E}(\mathbf{z}_i^{\mathsf{f}},\mathbf{s}_i)$
in the same way, and a planning-aware metric $\mathcal{M}_r$, oriented as a loss, defines
$H_i^r
=
\mathcal{M}_r(\mathbf{z}_i^{\mathsf{c}},\mathbf{s}_i)
-
\mathcal{M}_r(\mathbf{z}_i^{\mathsf{f}},\mathbf{s}_i)$.
Both are positive when the full mode scores better, but the allocation target is $V^q$, not $G^{\mathcal{E}}$ or $H^r$. {Appendix~\ref{app:glossary} lists the terms used in the paper.}

\paragraph{The oracle reference.}
With equal escalation costs, a selection budget $k$ allows at most $k$ escalations, and the decision oracle selects
\begin{equation}
\mathcal{A}_q^*(k)
=
\arg\max_{\mathcal{A}:\,|\mathcal{A}|\le k}
\sum_{i\in\mathcal{A}} V_i^q .
\label{eq:oracle}
\end{equation}
The budget is a capacity, not a quota: because $V_i^q$ can be negative, the oracle escalates only beneficial inputs and may leave capacity unused.
Whenever all beneficial escalations fit within the budget and at least one escalation is harmful, the oracle gains strictly more than all-full inference.
\paragraph{{Single-step scope.}}
\bench{} scores each input counterfactually against a fixed reference scene (Figure~\ref{fig:task}). Where a decision rule or loss depends on the previous action, both branches use that of all-cheap operation (Appendix~\ref{app:hw}).
$V_i^q$ is thus the value of one escalation for one decision, with scene and history fixed, not its closed-loop value.
\paragraph{What a single score must encode.}
Eq.~(\ref{eq:oracle}) also determines what an allocator must predict.
Let $S_i$ be its score for input $i$ and $\mathcal{A}_S(k)$ the $k$ highest-scoring inputs.

\begin{proposition}[Optimal ranking under equal costs]
Assume equal escalation costs and let \(K_q = |\{i:V_i^q\ge 0\}|\).
A score attains the optimum of Eq.~(\ref{eq:oracle}) for every budget $k\le K_q$ if and only if:
(i) all inputs with $V_i^q>0$ are ranked ahead of those with $V_i^q\le 0$, with positive values ordered non-increasingly; and
(ii) all inputs with $V_i^q=0$ are ranked ahead of those with $V_i^q<0$.
\label{prop:rank}
\end{proposition}

The proof is in Appendix~\ref{app:proofs}, which also shows that no single strict ranking is optimal for two downstream systems whose beneficial escalations are ordered differently{, and extends the result to values expected given pre-escalation information}. Allocation quality therefore depends on the order of the beneficial inputs and on the boundary between useful and harmful ones, not on correlation alone.
\paragraph{Multiple fidelities and measured cost.}
In deployment, perception modes differ in cost, and the allocator itself is not free.
Let $\mathcal{F}$ be the available modes, with cheap mode $\mathsf{c}\in\mathcal{F}$.
For each mode $m\in\mathcal{F}$, define
\begin{equation}
V_i^{q,m}
=
\mathcal{L}_q\!\left(\pi_q(\mathbf{z}_i^{\mathsf{c}}),\mathbf{s}_i\right)
-
\mathcal{L}_q\!\left(\pi_q(\mathbf{z}_i^{m}),\mathbf{s}_i\right),
\label{eq:multifid-value}
\end{equation}
and let $\Delta C_m$ be the extra cost of mode $m$, with $\Delta C_{\mathsf{c}}=0$.
In our cascade, higher-fidelity modes run from scratch, so $\Delta C_m=C_m$.
If the cheap mode and the allocator run on each of $N$ inputs, at cost $C_{\mathsf{c}}+C_S$ per input, an average per-input budget
$b\ge C_{\mathsf{c}}+C_S$ leaves
$B=N\!\left(b-C_{\mathsf{c}}-C_S\right)$
for escalation, and allocation becomes
\begin{equation}
\max_{\{m_i\in\mathcal{F}\}}
\sum_i V_i^{q,m_i}
\qquad
\text{s.t.}
\qquad
\sum_i \Delta C_{m_i}\le B ,
\label{eq:multifid}
\end{equation}
a multiple-choice knapsack that reduces to Eq.~(\ref{eq:oracle}) for two modes with equal escalation cost, with $k=\lfloor B/\Delta C_{\mathsf f}\rfloor$. {The budget is an average over inputs, not a per-frame deadline.}

\section{\bench{}: Decision Evaluation for Escalated Perception}
\label{sec:benchmark}
\bench{} evaluates allocators, not perception models (Figure~\ref{fig:task}). {Its layers are the formulation of Eq.~(\ref{eq:intro-value}), the released cache and the tracks (Appendix~\ref{app:layers}).} A \emph{track}{, one instantiation,} fixes the perception modes, the downstream system, the data split and the cost profile; {what an allocator reads before escalation and how it is trained and spends the budget are left to it. A submission is a score per input, whose $k$ highest-scoring inputs escalate, or a 0/1 decision mask, scored as submitted with any excess reported.} Four principles define the tracks.
\begin{itemize}
\itemsep0pt
\item \emph{Decision alignment.} Extra computation is valued by its effect on a fixed downstream loss, $V_i^q$ in Eq.~(\ref{eq:intro-value}), not by perception quality.
\item \emph{Pre-escalation validity.} An allocator commits before the full mode runs and may use only the input, the cheap output, previous frames, calibration and ego state; such signals are \emph{deployable}. Signals that need the full output or the reference state are \emph{diagnostic}: they never enter an allocator's input at test time, and as scores they show what unavailable information would contribute. On non-test units, decision values, full outputs and reference states may serve as supervision.
\item \emph{Budget-matched reference.} Allocators are compared with the decision oracle at the same budget; all-full inference is a baseline, not an upper bound.
\item \emph{Resource accounting.} Latency and energy budgets charge the allocator on every input.
\end{itemize}

\begin{figure}[!h]
\centering
\includegraphics[width=\linewidth]{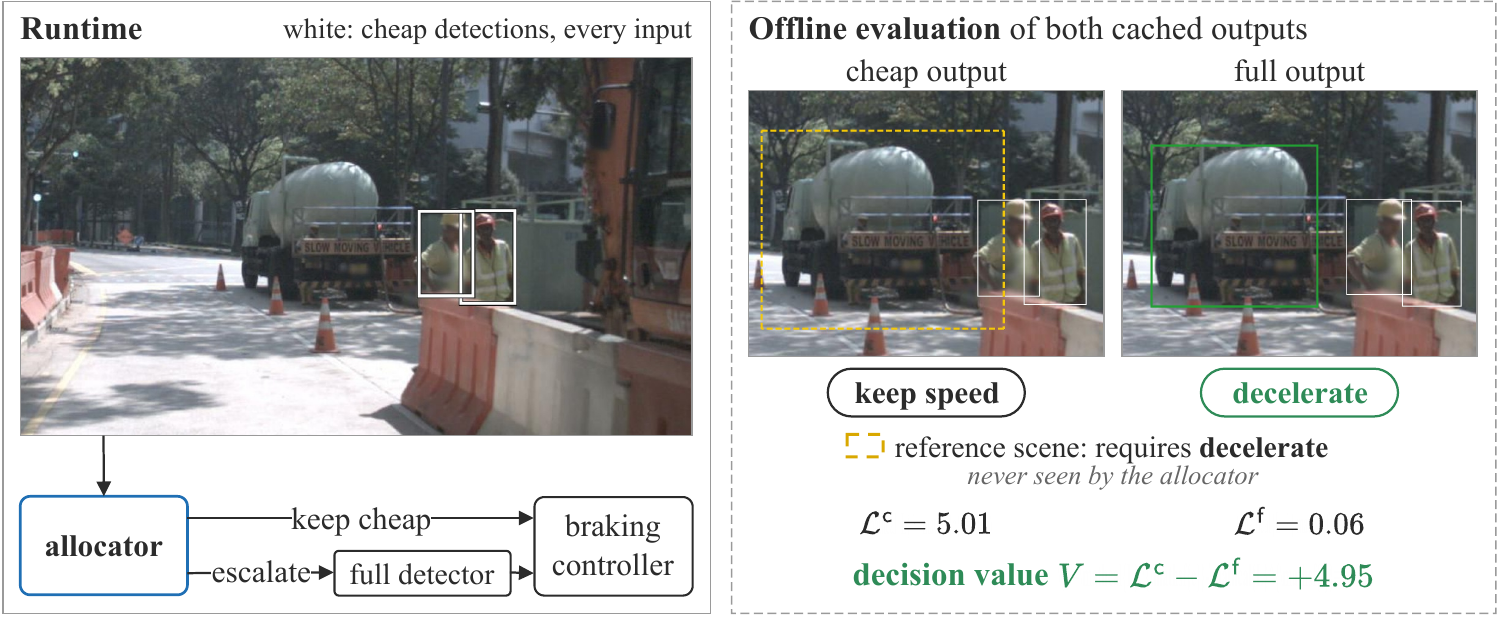}
\caption{\textbf{How \bench{} evaluates allocation}, on one nuScenes input (braking controller). \textbf{Runtime:} the allocator reads only pre-escalation information and keeps the cheap output or escalates. \textbf{Offline:} both cached outputs are scored through the same controller against the reference scene, which the allocator never receives; the full output recovers a truck (green) that the cheap one misses (dashed).}
\label{fig:task}
\end{figure}

\subsection{Tracks and downstream systems}
\label{sec:systems}

In every escalation pair, the full mode is a detector at $640$ pixels and the cheap mode is the same detector at a lower resolution: $320$ for YOLOv8s \citep{ultralytics2023yolov8} on nuScenes and nuPlan, and on KITTI $320$, $384$ or $512$ for YOLOv8s and $320$ or $480$ for RT-DETR-l \citep{zhao2024rtdetr}. {Both detectors use released COCO-trained weights without fine-tuning, and both modes share one confidence threshold, 0.25 (Appendix~\ref{app:calib}).}
A \emph{cell} is one combination of dataset, detector pair, geometry and downstream system. A core cell uses either the deployed monocular geometry, which estimates each object's 3D position from the image, or an oracle-geometry control, in which matched detections take their 3D geometry from the reference while unmatched ones keep the monocular estimate. The control changes only what the downstream system receives; allocators see identical inputs in both.

{\paragraph{From detections to a decision.} A downstream system sees perception only through the object list it receives. For the braking controller $q_{\mathrm{brake}}$, each detected object is placed on the road using the cell's geometry; the objects that overlap a 2.4\,m wide corridor give the deceleration needed to stop 2\,m behind the most constraining one; and that deceleration selects a command: keep speed, decelerate or brake hard. The same rule applied to the reference objects gives the stop-behind requirement $a^{\star}$, and the braking loss penalizes {shortfall against it}, unnecessary braking and changes of command. Two detection lists therefore change the loss only if they change the command: a recovered object can trigger needed braking, and a false positive unneeded braking (Figure~\ref{fig:concept}).\par}

The core track has three downstream systems.
The braking controller $q_{\mathrm{brake}}$ runs on nuScenes and KITTI.
The receding-horizon controller $q_{\mathrm{traj}}$, on KITTI, plans longitudinal and lateral actions around the perceived obstacles and is scored against the reference obstacles.
The published learned planner $q_{\mathrm{plan}}$, on nuScenes, receives the detected objects and predicts the most likely ego position at each future step; its losses are the average and final displacement (ADE, FDE) of this path from the logged future ego trajectory, on the 2{,}655 frames that have one (Appendix~\ref{app:hw}). We include it because PKL is built on it. It was trained on most of these scenes (Appendix~\ref{app:hw}), and both losses measure deviation from the log, not safety.
The external track runs PDM-Closed \citep{dauner2023pdm} and IDM \citep{treiber2000congested} at their released configurations on nuPlan \citep{caesar2021nuplan}. Their safety loss penalizes collisions and insufficient clearance with respect to the logged objects, and their scalar loss adds the deviation from the logged ego trajectory.
Each mode's observation is built from real detections on the nuPlan front camera, and each state is evaluated open-loop with a fresh planner instance per branch (Appendix~\ref{app:external}). {The nuScenes cells use 85 scenes (3{,}376 frames), none at night or in rain, so Section~\ref{sec:followup} repeats braking on all 850 trainval scenes; the KITTI cells use the 21 KITTI tracking training sequences (8{,}008 frames), and the nuPlan track 60 scenarios (1{,}440 states).}

\subsection{Splits, statistics and metrics}
\label{sec:splits}

Each track uses one deterministic split, fixed before any baseline was evaluated: by scene on nuScenes, by sequence on KITTI and by log on nuPlan (a log is one continuous recording), since scenarios from one log overlap in time. {We call these the \emph{frozen} splits, and the tracks and settings evaluated on them frozen.}
The test splits contain 24 scenes, 6 sequences and 9 logs.
Learned allocators are fit on the remaining units and evaluated once on the test split.

Two quantities describe a perception--decision setting rather than an allocator, so they are reported over all units. An input is \emph{affected} if $V_i^q\neq 0$; the \emph{harmed share} is the fraction of affected inputs with $V_i^q<0$, and the \emph{harm ratio} $\rho_q$ compares total harm with total benefit.
Confidence intervals are 95\%, from 1{,}000 bootstrap resamples of scenes, sequences or logs, never frames; comparisons with random are paired within each draw, and intervals are not corrected for multiple comparisons unless stated. {We call a difference \emph{not resolved} when its 95\% interval includes zero, and mark as \emph{post hoc} any analysis specified after earlier results had been seen (Appendix~\ref{app:followup}).}
Selections are evaluated in exact expectation over tie-breaks (Appendix~\ref{app:ties}).

An allocator is scored by its \emph{realized decision gain}, the sum of $V_i^q$ over the inputs it escalates as a share of the all-cheap loss (differences between such shares are in percentage points), and by its \textbf{normalized decision gain} (nDG), the fraction of the oracle's gain it realizes at the same budget:
\begin{equation}
\ndg_q(S,k)
=
\frac{
\sum_{i\in\mathcal{A}_S(k)} V_i^q
}{
\sum_{i\in\mathcal{A}_q^*(k)} V_i^q
}, \qquad  \rho_q
=
\frac{
\sum_{i:V_i^q<0}|V_i^q|
}{
\sum_{i:V_i^q>0}V_i^q
}.
\label{eq:ndg}
\end{equation}
Within the budget, $\ndg\le1$, with equality at the oracle; $\ndg=0$ means no net decision gain, and $\ndg<0$ means that the selected escalations cause net harm.
When all beneficial inputs fit within the budget, all-full inference retains a fraction $1-\rho_q$ of the oracle's value, so $\rho_q=0$ is the case in which escalation never harms.
Both metrics are reported only when their denominators are positive. Under measured budgets, nDG divides by the zero-overhead oracle, so an allocator's own cost reduces its usable capacity.

\subsection{Budgets and baselines}
\label{sec:baselines}

The measured-budget track charges costs as a cascade, implementing Eq.~(\ref{eq:multifid}): every input runs the cheap pipeline and the allocator, and an escalated input also runs its selected higher-fidelity mode{, with costs measured on a Jetson AGX Xavier (Appendix~\ref{app:hw})}. {This accounting belongs to the track, not to Eq.~(\ref{eq:intro-value}), which needs only two cached branches: Appendix~\ref{app:bench} also charges the pixel router as skipping the cheap pass, and Appendix~\ref{app:fu-seq} evaluates allocation before any current-frame perception.}
Per-input budgets
$b = C_{\mathsf{c}}+\alpha C_{\mathsf{f}}$ {(on nuScenes, $12.71+19.35\,\alpha$\,ms)}
are expressed in milliseconds or millijoules, at the levels where a zero-cost allocator could escalate a fraction
$\alpha\in\{0.1,0.2,0.3,0.5\}$
of inputs.
Selection budgets evaluate offline rankings of each whole test split; \emph{causal} allocation, which decides each input as it arrives, is reported at a fixed rate in Appendix~\ref{app:bench}.

The baselines are random routing and ego speed as simple controls;
cheap-detection uncertainty and cheap-side criticality, computed from the monocular geometry of the cheap detections, as hand-designed signals;
ridge and gradient-boosted \emph{gates}, which regress $V_i^q$ on cheap-side features;
detection-list \emph{routers} in the style of ORIC, MLP or gradient-boosted models that predict $V_i^q$ or $\mathbf{1}[V_i^q>0]$ from the 25 most confident cheap detections;
and a 0.15\,GFLOPs pixel router on the downscaled camera frame.
The diagnostic signals are criticality from reference geometry, eight fixed perception-loss gains, including the false-negative gain (the reduction in missed objects) and the risk-weighted gain, and the PKL and TIP gains.

\section{Empirical Findings}
\label{sec:results}

The experiments answer three questions. Does a better perception output reliably give a better decision (Sections~\ref{sec:sign} and~\ref{sec:external})? Can an allocator recognize valuable escalations before escalating (Sections~\ref{sec:signals-results} and~\ref{sec:planning-metric})? Does its advantage survive its own measured cost (Section~\ref{sec:cost})? {Section~\ref{sec:followup} checks the answers on all 850 nuScenes scenes and on NAVSIM.} Unless stated otherwise, allocators are reported on the frozen test splits at a 20\% selection budget.

\subsection{Expensive perception has sign-varying downstream value}
\label{sec:sign}

\begin{table}[t]
\caption{Decision value is sign-varying (all units, point estimates). Affected: inputs with $V_i^q\neq0$, of 3{,}376 nuScenes frames (2{,}655 for $q_{\mathrm{plan}}$; {34{,}149 on all 850 scenes}), 8{,}008 KITTI frames and 1{,}440 nuPlan states. Harmed: share made worse; $\rho_q$: harm ratio. All-full, Oracle@20: \% of the all-cheap loss removed by escalating every input, or {by the oracle at a 20\% budget} ($^\dagger$: the budget binds). Controls: Appendix~\ref{app:sign}.}
\label{tab:ceiling}
\centering
\fontsize{7.6}{9.2}\selectfont
\setlength{\tabcolsep}{3pt}
\begin{tabular}{llrrrrr}
\toprule
Dataset, system, geometry & Modes / loss & Affected & Harmed & $\rho_q$ & All-full & Oracle@20 \\
\midrule
nuScenes, $q_{\mathrm{brake}}$, oracle & YOLOv8s 320$\to$640 & 141 & 35.5\% & 0.42 & 14.01\% & 24.07\% \\
nuScenes, $q_{\mathrm{brake}}$, mono & YOLOv8s 320$\to$640 & 273 & 39.6\% & 0.41 & 16.47\% & 27.77\% \\
nuScenes, $q_{\mathrm{plan}}$, oracle$^\dagger$ & YOLOv8s 320$\to$640 / ADE & 1{,}179 & 50.9\% & 0.70 & \phantom{0}1.31\% & \phantom{0}4.41\% \\
nuScenes, $q_{\mathrm{plan}}$, mono$^\dagger$ & YOLOv8s 320$\to$640 / ADE & 1{,}612 & 49.3\% & 0.71 & \phantom{0}1.59\% & \phantom{0}5.31\% \\
nuScenes, 850 scenes, $q_{\mathrm{brake}}$, mono & YOLOv8s 320$\to$640 & 2{,}893 & 46.9\% & 0.96 & \phantom{0}1.11\% & 31.43\% \\
\midrule
KITTI, $q_{\mathrm{brake}}$, mono & YOLOv8s 320$\to$640 & 2{,}075 & 34.2\% & 0.17 & 51.68\% & 62.11\% \\
KITTI, $q_{\mathrm{traj}}$, mono & YOLOv8s 320$\to$640 & 678 & 48.4\% & 0.28 & 12.78\% & 17.81\% \\
KITTI, $q_{\mathrm{traj}}$, mono & YOLOv8s 384$\to$640 & 458 & 48.7\% & 0.46 & \phantom{0}5.97\% & 11.03\% \\
KITTI, $q_{\mathrm{traj}}$, mono & YOLOv8s 512$\to$640 & 266 & 53.8\% & 1.11 & $-0.59$\% & \phantom{0}5.16\% \\
KITTI, $q_{\mathrm{traj}}$, mono & RT-DETR-l 480$\to$640 & 295 & 43.4\% & 0.58 & \phantom{0}2.94\% & \phantom{0}6.98\% \\
\midrule
nuPlan, PDM-Closed & safety loss & 45 & 44.4\% & 0.33 & \phantom{0}9.81\% & 14.74\% \\
nuPlan, IDM & safety loss & 5 & 40.0\% & 0.06 & \phantom{0}1.42\% & \phantom{0}1.50\% \\
\bottomrule
\end{tabular}
\end{table}

Escalation often makes the decision worse (Table~\ref{tab:ceiling}). With monocular geometry, it harms 34--54\% of the affected inputs in every frozen setting.
On nuScenes braking with oracle geometry, helpful escalations reduce the all-cheap loss by 24.1\% and harmful ones increase it by 10.1\% ($\rho_q=0.42$). {On KITTI at 512$\to$640, total harm exceeds total benefit for the trajectory controller in point estimate ($\rho_q=1.11$), but the interval $[0.45,3.82]$ includes 1, so whether all-full inference raises its loss is not resolved.}

The mechanism is an exchange between recovered misses and added false positives among the objects the controller considers. For braking on nuScenes under oracle geometry, the full mode adds a false positive that the controller reacts to on 76\% of the harmed frames and recovers a needed miss on 6\%, against 29\% and 62\% on the helped frames. Across the braking settings{\ of these tracks}, 73--96\% of the total harm is unnecessary braking rather than {braking} shortfall{; on all 850 scenes, braking shortfall instead carries 60\% of the harm under the standard requirement, and 23\% when the lead vehicle may brake (Appendix~\ref{app:fu-lead})}. Perception improvement does not fix the sign: here, four perception-level gains disagree in sign with decision value on 25--33\% of the braking frames and 50--51\% of the planner frames where both are non-zero. The downstream system matters: of the 79 frames that change both systems' losses, 40 are helped by one and harmed by the other (Appendix~\ref{app:mech}).

{Two references help read these shares. If better detection always helped, the harmed share would be near zero. Replacing a detector by one of nearly equal recall harms 45.5\% and 52.4\% of affected braking frames in two such pairs, and three clear upgrades still harm 40.7--44.2\% (Appendix~\ref{app:fu-other}).}

\subsection{The pattern holds for a published planner under real detector outcomes}
\label{sec:external}

On nuPlan, escalation harms 44\% of the states whose PDM-Closed safety loss changes and 38\% under the scalar loss (Table~\ref{tab:ceiling}; $\rho_q$ 0.33 and 0.35), and an oracle at 20\% capacity reduces the safety loss by 14.7\% against 9.8\% for escalating every state. IDM changes its safety loss on only 5 of 1{,}440 states: the same detections matter to one published planner and barely to another. These events are sparse, so this track corroborates the core track rather than replacing it (Appendix~\ref{app:external}).

\paragraph{Construct validity.} Monocular range error is not the only source of harm: when matched detections take their geometry from the reference, the braking controller still harms 25--48\% of the affected inputs, although the trajectory controller's harm ratio falls below 0.06 on four of five fidelity pairs. Harm also persists when each mode uses its own confidence threshold (Appendix~\ref{app:calib}). Requiring detections to persist for three frames moves harmed shares by at most 6.8 points (Appendix~\ref{app:sign}), and post hoc variations of the loss weights and corridor width keep the harmed share at 33--55\% in every monocular setting (Appendix~\ref{app:hw}). These controls do not show that the magnitudes transfer to other detectors, planners or closed-loop operation.

\subsection{Pre-escalation signal exists for braking, not yet for planning}
\label{sec:signals-results}

\begin{figure}[t]
\centering
\includegraphics{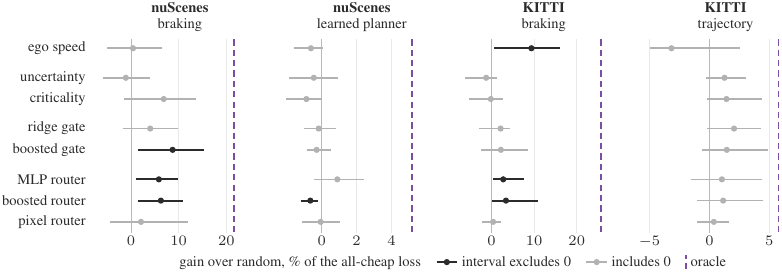}
\caption{\textbf{Deployable allocators against random} on the four core cells with deployed monocular geometry (20\% selection budget, frozen test splits; learned planner: average displacement). Points: realized gain minus random's, in \% of the all-cheap loss, with paired 95\% bootstrap intervals not corrected for multiple comparisons (black: excludes zero); dashed: the oracle's gain minus random's. Each panel has its own axis. Routers are those trained on $V$; every variant and cell: Figure~\ref{fig:alloc}.}
\label{fig:forest}
\end{figure}

Learned allocators carry pre-escalation signal, but it depends on the downstream system (Figure~\ref{fig:forest}): on braking, the gradient-boosted gate and both routers beat random on nuScenes and both routers on KITTI, while none beats random on KITTI's trajectory controller or the learned planner, where one router is worse than random. Where they win, they recover less than half of the oracle's headroom over random. Neither heuristic beats random in any core cell. Ego speed is strongest on KITTI braking but beats random on no nuScenes cell. On nuPlan the wins concentrate in one test log.

Learned allocators win mainly by capturing more benefit: on average over the ten core cells, gates and routers capture 6.0--7.5\% of the all-cheap loss against 3.8\% for random, while {the harm they select, 0.7--1.2\%, is not systematically below random's 1.3\%}.

\subsection{Ranking allocators by perception gain can select worse ones}
\label{sec:planning-metric}

\begin{figure}[t]
\centering
\includegraphics{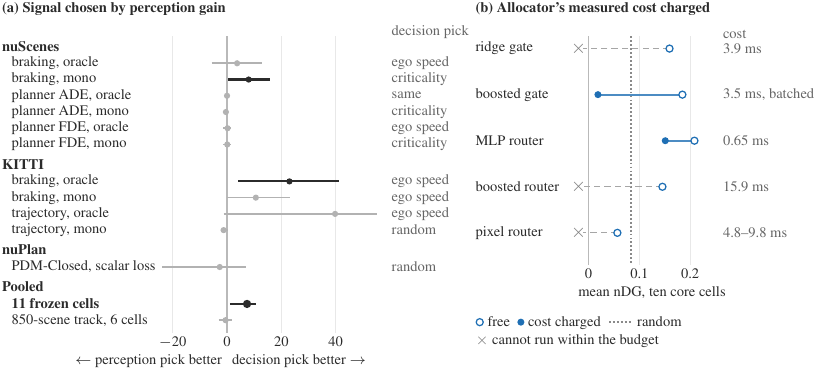}
\caption{\textbf{Selection objective and allocator cost}, at 20\% budgets. \textbf{(a)}~Choosing among four untrained deployable signals on validation: test realized gain of the decision-value pick minus that of the false-negative-gain (perception) pick, in points of the all-cheap loss, with paired 95\% intervals, uncorrected (black: excludes zero). Right: the decision-value pick (same: both pick it); the perception pick is uncertainty in every core cell and ego speed on PDM-Closed. \textbf{(b)}~Charging measured latency: mean nDG over the ten core cells of Figure~\ref{fig:forest}'s learned allocators, free (selection budget) or charged (latency budget), both by offline ranking and relative to the zero-overhead oracle; right: cost per input; all variants: Figure~\ref{fig:alloc}c.}
\label{fig:choice}
\end{figure}

{Suppose a practitioner must choose one of four untrained deployable signals (random, uncertainty, criticality, ego speed) on a validation set. Choosing by perception gain (the false-negative gain) instead of by decision value costs 7.4 points $[1.0,10.8]$ of realized gain on test at a 20\% budget, averaged over the ten core cells and PDM-Closed's scalar loss (Figure~\ref{fig:choice}a). Most of this difference comes from KITTI under oracle geometry, where decision value selects ego speed.} Under the risk-weighted gain the reduction is not resolved{, nor is the difference on the six nuScenes cells of the 850-scene track}. Trained on the published objectives of ORIC or budget-adaptive routing, the routers beat random in 12 and 15 of 200 settings, against 40 on decision value (post hoc, uncorrected; Appendix~\ref{app:bench}). {On the 228 test scenes of Section~\ref{sec:followup}, the gradient-boosted gate trained on decision value realizes 11.9 to 16.1 points more than when trained on the false-negative gain, the risk-weighted gain, ORIC's reward or the AP-swap objective, each interval above zero (ranked offline at a 20\% selection budget; post hoc; Table~\ref{tab:fu-target}). The routers give no such separation.} Perception-evaluation scores, even computed exactly, recover only part of the oracle's value (Appendices~\ref{app:planning},~\ref{app:bench}).

\subsection{Measured cost removes most learned advantages}
\label{sec:cost}

A selection budget treats the allocator as free; charging its measured latency removes most learned advantages (Figure~\ref{fig:choice}b). The feature-based gates spend 18--19\% of a full detector pass on feature extraction, so at a 20\% latency budget the ridge gate cannot run and the batched gradient-boosted gate escalates at most 1.7\% of inputs; neither beats random. The MLP detection-list routers cost 0.65\,ms and keep most of their nDG. {On the frozen core track they are the only learned baselines with an nDG interval above random at this budget, in one of their 20 comparisons. An allocator's \emph{break-even overhead} is the largest per-input cost at which its ranking still reaches random's gain; at this budget, its median over the ten core cells is 1.6--3.0\,ms for the learned allocators. The MLP routers (0.65\,ms) stay within their own in seven and eight cells, whereas the gradient-boosted routers (15.9\,ms) and the gate called once per input (21.4\,ms) cost 6 to 14 times the median (point estimates; Appendix~\ref{app:bench}). At a 20\% energy budget, the gate on criticality features realizes $+12.8\%$ $[+3.6,+28.6]$ on the 228 test scenes of Section~\ref{sec:followup}, ranked offline (Appendix~\ref{app:fu-cost}).}

{
\subsection{Evaluation at scale}
\label{sec:followup}

Beyond the benchmark's 85 nuScenes scenes, we evaluate braking on all 850 trainval scenes and two planners on NAVSIM \citep{dauner2024navsim}. Allocators are fitted on 537 of the new scenes and scored once on 228 others, which share recording logs with the fit set. Appendix~\ref{app:followup} gives the protocol and marks the post hoc analyses.

\begin{table}[!t]
\caption{{\textbf{Evaluation at scale}, 95\% intervals over scenes (logs on NAVSIM). Harmed: \% of affected inputs made worse. All-full, realized: \% of the all-cheap loss removed by escalating every input or the gate's selection. Latency: gate cost charged. Causal cap: threshold fixed in advance, running cap per scene. Used: escalations. Top scene out: without the most influential test scene. Re-splits: mean (min., max.) over 20 log-grouped repeats that reuse scenes (Appendix~\ref{app:fu-cost}). $^\dagger$Post hoc or sensitivity.}}
\label{tab:checks}
\centering
\footnotesize

\setlength{\tabcolsep}{4pt}
\begin{tabular}{@{}lll@{}}
\toprule
\textbf{(a)} Every input escalated & Harmed & All-full \\
\midrule
\multicolumn{3}{@{}l}{\emph{nuScenes braking, monocular, YOLOv8s 320$\to$640; all 850 scenes}} \\
\quad standard loss and requirement & 46.9\% [44.6, 49.1] & $+1.1$\% [$-18.3$, $+17.4$] \\
\quad 14 loss-weight settings$^\dagger$ (range) & 45.1--49.4\% & $-1.4$ to $+3.0$\% \\
\quad lead vehicle may brake at 4\,m/s$^2$\,$^\dagger$ & 47.0\% [44.9, 49.1] & $+7.0$\% [$+2.8$, $+11.1$] \\
\addlinespace[2pt]
\multicolumn{3}{@{}l}{\emph{NAVSIM navtest, loss one minus the PDM score; 12{,}146 tokens}} \\
\quad PDM-Closed on YOLOv8s 320$\to$640 detections & 62.1\% [58.4, 65.5] & $+1.3$\% [$-0.8$, $+3.3$] \\
\quad end-to-end camera planner, images halved per side & 42.1\% [40.5, 43.9] & $-0.4$\% [$-2.5$, $+1.4$] \\
\bottomrule
\end{tabular}

\vspace{6pt}
\begin{tabular}{@{}llrlll@{}}
\toprule
\multicolumn{6}{@{}l}{\textbf{(b)} Gradient-boosted gate on $V_i^q$, nuScenes braking, monocular; 228 test scenes (9{,}181 frames)} \\
\midrule
Budget, features & Rule & Used & Realized & Top scene out$^\dagger$ & Re-splits$^\dagger$ \\
\midrule
\multirow{2}{*}{\makecell[l]{20\% selection,\\all features}} & ranking & 1{,}836 & $+12.7$ [$+2.1$, $+30.0$] & $+5.9$ [$+1.1$, $+11.3$] & 11.8 (2.4, 16.5) \\
 & causal cap$^\dagger$ & 1{,}139 & $+6.0$ [$+0.4$, $+12.6$] & $+3.7$ [$-0.2$, $+7.1$] & \phantom{0}7.4 (4.0, 9.6) \\
\addlinespace[2pt]
\multirow{2}{*}{\makecell[l]{20\% latency,\\criticality features}} & ranking & 867 & $+9.1$ [$+0.0$, $+22.3$] & $+3.6$ [$-0.8$, $+7.1$] & \phantom{0}8.4 (1.5, 12.7) \\
 & causal cap$^\dagger$ & 521 & $+4.7$ [$-1.2$, $+11.7$] & $+2.3$ [$-2.1$, $+5.3$] & \phantom{0}4.1 (1.2, 6.4) \\
\bottomrule
\end{tabular}
\end{table}

\paragraph{Harm persists at scale.} In every setting of Table~\ref{tab:checks}a, the full mode worsens 42--62\% of affected inputs. The net effect of escalating every input is not stable: on the 850 scenes its interval spans zero and a few night scenes decide its sign, while letting the lead vehicle brake makes it positive (Appendices~\ref{app:fu-influence},~\ref{app:fu-lead}). Allocation nevertheless has large headroom: on the 850 scenes the oracle at a 20\% budget removes 31.4\% $[25.7,37.2]$ of the all-cheap loss (Table~\ref{tab:fu-lead}).

\paragraph{Allocator gains reach unseen logs but depend on the protocol.} On 73 benchmark scenes whose recording logs never enter the fit, the gate fitted on 754 new scenes beats random at the same number of escalations by $+7.0$ points $[+4.0,+9.9]$ at the selection budget, and by $+6.3$ $[+3.2,+8.9]$ with its latency charged, on criticality features alone (post hoc; Appendix~\ref{app:fu-influence}). On the 228 test scenes (Table~\ref{tab:checks}b), its gain is positive in every point estimate, but a causal threshold with a running cap per scene roughly halves it, and with latency charged the ranked gain's interval has its lower end at zero. Without the most influential test scene, only the offline selection-budget ranking keeps an interval above zero. The cap leaves budget unused and charging latency changes the features, so these rows do not isolate the causal rule or the cost. Fitted in one city, no gate or router beats random in the other.

\paragraph{Training and information.} In the two primary comparisons (20\% selection budget, corrected for both), a decision-focused surrogate learner chosen on validation does not beat the regression gate: $+1.4$ points $[-3.0,+5.0]$ on braking and $-1.0$ $[-2.0,-0.2]$ on average displacement, whose interval reaches zero when logs are resampled. Nor does a linear scorer trained on the exact budgeted objective beat ridge regression on the same inputs ($-1.9$ points $[-14.1,+5.2]$ on braking, $-0.02$ $[-0.54,+0.31]$ on average displacement; post hoc). At a 5\% budget the decision-focused learner leads on braking, by $+4.6$ points $[+2.3,+12.8]$ (secondary, uncorrected). Adding reference objects, which no allocator can read, raises both learners' gain, with intervals above zero, to $+19.5\%$ and $+18.5\%$ against $+27.9\%$ for the oracle, so missing information explains only part of the gap (Appendix~\ref{app:fu-direct}).

\paragraph{Sequences and closed loop.} In a sequential variant, a policy refreshes the detection list or reuses the previous one, which grows stale. Compared with refreshing once every 11 frames, causal policies with a 20\% refresh budget remove 27--32\% of the loss on the 228 test scenes; an oracle with hindsight removes 76\%. The sequential learner is not resolved against periodic refresh (Appendix~\ref{app:fu-seq}). In nuPlan forks under replayed perception, the closed-loop direction agrees with the single-step sign in 57.1\% $[40.5,73.2]$ (bootstrap over logs) of the 70 forks, on 10 logs, that change PDM-Closed's 8\,s safety loss, so whether single-step value predicts closed-loop outcomes is not resolved.
}

\section{Conclusion}
\label{sec:discussion}

Better perception does not imply a better decision: in every frozen monocular-geometry setting, 34--54\% of the escalations that change the downstream loss make it worse{, and harm is as frequent at scale and on NAVSIM. On 850 nuScenes scenes, the oracle at a 20\% budget removes 31\% of the braking loss where all-full inference does not reliably reduce it}. Perception compute should therefore be allocated by its effect on the decision, which \bench{} measures against a budgeted decision oracle, charging allocators for their own computation. {Learned allocators capture part of the oracle's gain, by finding benefit, not avoiding harm, and under measured budgets keep an advantage only at small overhead. On nuScenes braking, a learned gate's gain reaches unseen recording logs but not another city.} A new allocator should improve realized decision gain under its own cost and a causal budget, reported per downstream system. {\bench{} assumes an allocator that chooses among existing perception modes and leaves the downstream system unchanged; redesigning perception or co-training the planner is outside its scope. Its value is single-step, with scene and action history fixed, and its link to closed-loop outcomes is unresolved.}

\bibliography{references}
\bibliographystyle{iclr2027_conference}

\newpage
\appendix
\renewcommand{\topfraction}{0.9}\renewcommand{\bottomfraction}{0.8}\renewcommand{\textfraction}{0.07}\renewcommand{\floatpagefraction}{0.8}

\begin{center}
    \Large{\textbf{APPENDIX}}
\end{center}

\section{Proofs}
\label{app:proofs}

\subsection{Proof of Proposition~\ref{prop:rank}}
Write $\mathcal{P}_q=\{i:V_i^q>0\}$, $\mathcal{Z}_q=\{i:V_i^q=0\}$ and $\mathcal{N}_q^-=\{i:V_i^q<0\}$, so that $K_q=|\mathcal{P}_q|+|\mathcal{Z}_q|$. At budget $k$ a score selects its $k$ top-ranked inputs. With equal escalation costs every $\mathcal{A}$ with $|\mathcal{A}|\le k$ is feasible, and removing a non-positive element never decreases a sum, so the optimum of Eq.~(\ref{eq:oracle}) is
\begin{equation}
\mathrm{OPT}(k)=\sum_{r=1}^{\min(k,|\mathcal{P}_q|)}V_{(r)}^q ,
\end{equation}
where $V_{(r)}^q$ is the $r$-th largest positive value; for $k\ge|\mathcal{P}_q|$ it equals $\sum_{i\in\mathcal{P}_q}V_i^q$.

\paragraph{Sufficiency.} Under (i) and (ii), for $k\le|\mathcal{P}_q|$ the top $k$ inputs are positive inputs carrying the $k$ largest values, with ties in any order, and for $|\mathcal{P}_q|<k\le K_q$ they are $\mathcal{P}_q$ together with $k-|\mathcal{P}_q|$ zero-valued inputs. Both attain $\mathrm{OPT}(k)$ whatever the order of the zero-valued inputs.

\paragraph{Necessity of (ii).} If some input in $\mathcal{N}_q^-$ is ranked above some input in $\mathcal{Z}_q$, the top $K_q$ inputs cannot all be non-negative, so at budget $K_q$ the selection contains some $j\in\mathcal{N}_q^-$. Its value is then at most $\sum_{i\in\mathcal{P}_q}V_i^q+V_j^q<\mathrm{OPT}(K_q)$.

\paragraph{Necessity of (i).} Suppose some $i\in\mathcal{P}_q$ is ranked below an input $j$ with $V_j^q<V_i^q$, and let $k$ be the position of $j$. If $k\le K_q$, the selection at budget $k$ contains $j$ but not $i$, and exchanging $j$ for $i$ gives a feasible set with strictly larger value, so the score falls below $\mathrm{OPT}(k)$. If $k>K_q$, then $i$ is not among the top $K_q$ inputs, and the value at budget $K_q$ is at most $\sum_{i'\in\mathcal{P}_q\setminus\{i\}}V_{i'}^q<\mathrm{OPT}(K_q)$. \qed

\paragraph{Abstention.} If an allocator may also leave capacity unused, through an abstention set whose members are never selected, the same argument with (ii) replaced by abstention on every input in $\mathcal{N}_q^-$, and every input in $\mathcal{P}_q$ left non-abstained, gives optimality at every budget, including $k>K_q$.

\paragraph{Remark.} Both conditions use equal escalation costs. With per-input costs $\Delta C_{m_i}$ and a budget $B$, Eq.~(\ref{eq:multifid}) is a knapsack whose optimum depends on the pairs $(V_i^{q,m_i},\Delta C_{m_i})$ jointly; no ordering of $V^q$ alone is sufficient, and the abstention rule generalizes to a threshold on $V_i^{q,m_i}/\Delta C_{m_i}$ only in the fractional relaxation.

\paragraph{Ranking by conditional expected value.} {Proposition~\ref{prop:rank} concerns realized values, which an allocator does not observe before escalation. Let $\mathcal{I}$ be the information available to it and $m_i=\E[V_i^q\mid\mathcal{I}]$. Because the gain of an allocation is a sum of per-input values (Section~\ref{sec:formulation}), any $\mathcal{A}$ chosen from $\mathcal{I}$ has conditional expected gain $\sum_{i\in\mathcal{A}}m_i$, and the argument above with $m_i$ in place of $V_i^q$ shows that the expected-gain optimum at budget $k$ selects up to $k$ inputs with the largest positive $m_i$. If $V_i^q$ is conditionally independent of the rest of $\mathcal{I}$ given features $x_i$, then $m_i=\E[V^q\mid x_i]$, the minimizer of the expected squared error over functions of $x_i$. The plug-in rule for this optimum therefore estimates $m_i$ by regression and escalates, among the inputs with a positive estimate, the $k$ largest. A score submission escalates exactly its $k$ highest-scoring inputs, so regression followed by top-$k$ attains this rule's estimated gain when at least $k$ estimates are non-negative, the counterpart of $k\le K_q$ in Proposition~\ref{prop:rank}; at larger $k$ it must also escalate inputs with negative estimates, and the abstention the rule needs is available only as a decision mask (Section~\ref{sec:benchmark}). The statement concerns the population and equal costs. A fit from finite data or a restricted model class need not order inputs as $m_i$ does, which is where losses built from Eq.~(\ref{eq:oracle}) can help, and it does not cover causal selection or unequal costs.}

\subsection{No shared optimum across downstream systems}
\begin{corollary}[No shared optimum under a ranking reversal]
If two inputs $i,j$ have positive value under downstream systems $q_1$ and $q_2$, but
$V_i^{q_1}>V_j^{q_1}$ while
$V_i^{q_2}<V_j^{q_2}$,
then no single strict ranking is optimal at every budget for both systems.
\label{cor:reversal}
\end{corollary}
\emph{Proof.}
Since $i$ and $j$ are positive under both systems, condition (i) of Proposition~\ref{prop:rank} requires an optimal ranking under $q_1$ to place $i$ above $j$ and one under $q_2$ to do the reverse, which no strict ranking does.

\subsection{Regret from a top-\texorpdfstring{$k$}{k} mismatch}
For $k\le K_q$, take an oracle selection $\mathcal{A}_q^*(k)$ of size $k$ whose zero-valued members are drawn from $\mathcal{A}_S(k)$ first, so that the missed-oracle set $\mathcal{U}=\mathcal{A}_q^*(k)\setminus\mathcal{A}_S(k)$ and the wrongly selected set $\mathcal{W}=\mathcal{A}_S(k)\setminus\mathcal{A}_q^*(k)$ have equal size $u$ and never both contain a zero-valued input. If every missed oracle input exceeds every wrongly selected input by at least $\delta_V>0$, the regret satisfies
\begin{equation}
\mathcal{R}_q(S,k)=\sum_{i\in\mathcal{A}_q^*(k)}V_i^q-\sum_{i\in\mathcal{A}_S(k)}V_i^q=\sum_{r=1}^{u}\bigl(V_{\mathcal{U}_r}^q-V_{\mathcal{W}_r}^q\bigr)\ge u\,\delta_V .
\end{equation}
Top-$k$ disagreement therefore matters most when the decision-value margin at the budget boundary is large, which is why aggregate correlation is a poor summary of allocation quality.

\section{Cross-Target Diagnostic Analysis}
\label{app:planning}

This appendix evaluates each score as one fixed function against several downstream targets on the same nuScenes frames, transition and geometry, over all units. It is an analysis of what the diagnostic signals know, not a benchmark result; the official held-out results are in Figure~\ref{fig:alloc} and Tables~\ref{tab:routers} and~\ref{tab:heldout}.

\paragraph{Self-reference gap.} For a planning-aware metric $\mathcal{M}_r$ computed through the learned planner, with gain $H^r$, define $\Gamma_{\mathrm{self}}(r)=1-\ndg_{q_{\mathrm{plan}}}(H^r,k)/\ndg_{q_{\mathrm{plan}}^{\,\mathrm{self}}}(H^r,k)$, which compares scoring the planner against the logged trajectory with scoring it against its own reference-conditioned output; a descriptive contrast that is unbounded above. On the 2{,}655 nuScenes frames with a logged future and a 20\% budget, with both nDG values computed on those frames ($0.893$ against the self-reference for PKL), $\Gamma_{\mathrm{self}}(\mathrm{PKL})=0.60$ $[0.40,0.84]$ and $\Gamma_{\mathrm{self}}(\mathrm{TIP})=0.56$ under oracle geometry, and $0.47$ and $0.46$ under monocular geometry. PKL shares model, weights and forward pass with $q_{\mathrm{plan}}^{\,\mathrm{self}}$, so a frame whose predicted heatmap moves affects both quantities by construction; measured against the logged trajectory the remaining signal is $0.36$, against $0.06$ for random. PKL is designed to measure consequences through its planner; $\Gamma_{\mathrm{self}}$ contrasts its allocation value against that planner's own output with its value against the logged outcome.

\begin{figure}[htbp]
\centering
\begin{tikzpicture}
\pgfplotsset{
  etastyle/.style={
    width=5.0cm, height=3.7cm,
    xlabel={\scriptsize budget (\%)},
    xmin=7, xmax=53, ymin=-0.08, ymax=1.02,
    xtick={10,20,30,50}, ytick={0,0.2,0.4,0.6,0.8,1.0},
    grid=major, grid style={gray!18},
    axis line style={gray!55}, tick style={gray!55},
    tick label style={font=\scriptsize}, title style={font=\scriptsize},
    mark size=1.3pt, line width=0.8pt,
  }
}
\begin{axis}[etastyle, name=p1,
  ylabel={\scriptsize nDG (share of oracle value)},
  title={braking controller $q_{\mathrm{brake}}$}]
\addplot[cRisk, mark=*] coordinates {(10,0.399)(20,0.575)(30,0.572)(50,0.561)};
\addplot[cDE, mark=square*] coordinates {(10,0.103)(20,0.171)(30,0.328)(50,0.563)};
\addplot[cPKL, mark=triangle*] coordinates {(10,0.179)(20,0.242)(30,0.312)(50,0.393)};
\addplot[cTIP, mark=diamond*] coordinates {(10,0.166)(20,0.173)(30,0.201)(50,0.266)};
\addplot[cRand, densely dotted, mark=none] coordinates {(10,0.063)(20,0.122)(30,0.192)(50,0.327)};
\end{axis}
\begin{axis}[etastyle, name=p2, at={(p1.east)}, anchor=west, xshift=5mm,
  ylabel={}, yticklabels={,,}, title={learned planner $q_{\mathrm{plan}}$}]
\addplot[cRisk, mark=*] coordinates {(10,0.090)(20,0.183)(30,0.196)(50,0.212)};
\addplot[cDE, mark=square*] coordinates {(10,0.133)(20,0.228)(30,0.223)(50,0.184)};
\addplot[cPKL, mark=triangle*] coordinates {(10,0.386)(20,0.359)(30,0.389)(50,0.399)};
\addplot[cTIP, mark=diamond*] coordinates {(10,0.375)(20,0.376)(30,0.370)(50,0.373)};
\addplot[cRand, densely dotted, mark=none] coordinates {(10,0.038)(20,0.061)(30,0.101)(50,0.160)};
\end{axis}
\begin{axis}[etastyle, name=p3, at={(p2.east)}, anchor=west, xshift=5mm,
  ylabel={}, yticklabels={,,}, title={same planner, self-reference},
  legend style={font=\scriptsize, at={(-0.82,-0.40)}, anchor=north, legend columns=5,
                draw=gray!40, inner sep=2pt,
                /tikz/every even column/.append style={column sep=5pt}}]
\addplot[cRisk, mark=*] coordinates {(10,0.386)(20,0.455)(30,0.462)(50,0.514)};
\addplot[cDE, mark=square*] coordinates {(10,0.208)(20,0.328)(30,0.398)(50,0.523)};
\addplot[cPKL, mark=triangle*] coordinates {(10,0.820)(20,0.869)(30,0.891)(50,0.894)};
\addplot[cTIP, mark=diamond*] coordinates {(10,0.854)(20,0.850)(30,0.860)(50,0.870)};
\addplot[cRand, densely dotted, mark=none] coordinates {(10,0.068)(20,0.123)(30,0.171)(50,0.277)};
\legend{risk-weighted gain, false-negative gain, PKL gain, TIP gain, random}
\end{axis}
\end{tikzpicture}
\caption{\textbf{Perception-level signals recover part of the decision value, and the best one depends on the downstream system.} nDG on nuScenes (YOLOv8s 320$\to$640, oracle geometry, all units) when escalation is ranked by each signal, computed exactly after escalation. Right: the planner scored against its own reference-conditioned output.}
\label{fig:eta}
\end{figure}

\begin{table}[htbp]
\caption{Cross-target diagnostic analysis: nDG at a 20\% budget on nuScenes, oracle geometry, all units. Only the first three rows are deployable; the gate is fit out of fold by scene. The $q_{\mathrm{plan}}^{\,\mathrm{self}}$ column uses all 3{,}376 frames, the $q_{\mathrm{plan}}$ column the 2{,}655 frames with a logged future; the self-reference gap in the text recomputes $q_{\mathrm{plan}}^{\,\mathrm{self}}$ on those 2{,}655 frames. Bold: best diagnostic score in each column.}
\label{tab:target}
\centering
\small
\begin{tabular}{llrrr}
\toprule
& & Braking & \multicolumn{2}{c}{Published learned planner} \\
\cmidrule(lr){4-5}
Allocation score & class & $q_{\mathrm{brake}}$ & $q_{\mathrm{plan}}$ & $q_{\mathrm{plan}}^{\,\mathrm{self}}$ \\
\midrule
Random (16 seeds) & deployable & 0.122 & 0.061 & 0.123 \\
Cheap-detection uncertainty & deployable & 0.005 & 0.134 & 0.096 \\
Pre-escalation gate, GBM & deployable & 0.309 & 0.079 & --- \\
Criticality (reference geometry) & diagnostic & 0.182 & 0.016 & 0.219 \\
False-negative gain $G^{\mathcal{E}_{\mathrm{FN}}}$ & diagnostic & 0.171 & 0.228 & 0.328 \\
Risk-weighted gain $G^{\mathcal{E}_{\mathrm{risk}}}$ & diagnostic & \textbf{0.575} & 0.183 & 0.455 \\
PKL gain $H^{\mathrm{PKL}}$ & diagnostic & 0.242 & 0.359 & \textbf{0.869} \\
TIP gain $H^{\mathrm{TIP}}$ & diagnostic & 0.173 & \textbf{0.376} & 0.850 \\
Target decision oracle & label & 1.000 & 1.000 & 1.000 \\
\bottomrule
\end{tabular}
\end{table}

On $q_{\mathrm{brake}}$ and the logged-trajectory $q_{\mathrm{plan}}$, no diagnostic approaches the oracle: the best are the risk-weighted gain at $0.575$ $[0.32,0.75]$ on $q_{\mathrm{brake}}$ and TIP at $0.376$ $[0.12,0.54]$ on $q_{\mathrm{plan}}$. The best score also changes with the downstream system (Figure~\ref{fig:eta}): $\mathcal{E}_{\mathrm{risk}}$ leads on $q_{\mathrm{brake}}$, against PKL's $0.242$, TIP leads on the logged-trajectory target and PKL on the self-referenced one. Under deployed monocular geometry the risk-weighted gain falls to $0.227$ on the braking target, since escalation also changes the estimated range of what it detects, which a detection-set score cannot see; PKL and TIP, which read the lifted detections through a planner, reach $0.289$ and $0.271$ against $0.135$ for random.

\subsection{Longitudinal braking controller}
On the full 3{,}376-frame set at a 20\% budget, random routing reaches $0.122$, uncertainty $0.005$, criticality $0.182$, false-negative gain $0.171$, the risk-weighted gain $0.575$, and PKL and TIP $0.242$ $[0.01,0.43]$ and $0.173$ $[-0.14,0.41]$. The oracle reduces the all-cheap loss by 24.07\% against 14.01\% for all-full inference; of the 141 frames with non-zero decision value, 91 are helped and 50 harmed.

\subsection{Learned planner against the logged trajectory}
On 2{,}655 frames, $q_{\mathrm{plan}}$ has non-zero decision value on 1{,}179, of which 600 (50.9\%) are harmed by escalation. The oracle with a 20\% budget reduces the all-cheap loss by 4.41\% against 1.31\% for all-full inference. PKL reaches $\ndg=0.359$ $[0.13,0.54]$ and TIP $0.376$ $[0.12,0.54]$, against $0.228$ for the best perception-based score and $0.061$ for random.

\subsection{Learned planner against its own reference-conditioned output}
The self-referenced control has non-zero decision value on 1{,}413 of 3{,}376 frames, of which 578 (40.9\%) are harmed. The oracle with a 20\% budget reduces the loss by 13.67\% against 7.49\% for all-full inference. PKL reaches $\ndg=0.869$ $[0.82,0.91]$ and TIP $0.850$ $[0.79,0.89]$. This condition serves the self-reference gap above and is not a downstream system of independent interest.

\subsection{Lateral corridor controller}
The lateral corridor controller enters only the robustness sweep of Appendix~\ref{app:robust}; on nuScenes its decision value is non-zero on 23 frames, too few for scene-level intervals.

\subsection{TIP score direction}
\label{app:tip}
PKL is higher when worse, so its gain is $H^{\mathrm{PKL}}=\mathrm{PKL}_{\mathsf{c}}-\mathrm{PKL}_{\mathsf{f}}$. The released TIP implementation computes a non-positive preference-gap quantity, with more negative values indicating worse agreement, so, following the sign convention of the released implementation, its gain is $H^{\mathrm{TIP}}=\mathrm{TIP}_{\mathsf{f}}-\mathrm{TIP}_{\mathsf{c}}$. Under this convention TIP gain correlates positively with perception gain on the wiring-validation subset, as PKL gain does, which confirms the direction.

\section{Sign Variation Across Configurations}
\label{app:sign}

\begin{table}[htbp]
\caption{Sign-varying decision value and harm ratio across detector families, fidelity gaps and geometries. Loss reductions are relative to all-cheap inference; Oracle@20 has capacity up to 20\% of inputs. Every KITTI row is budget-slack, so $\rho_q=1-{}$all-full$/$Oracle@20 there. $\rho_q$ is smallest for the largest fidelity gap under oracle geometry (row 5), where escalation is almost always beneficial.}
\label{tab:sign}
\centering
\small
\setlength{\tabcolsep}{4pt}
\begin{tabular}{llrrrrr}
\toprule
Downstream system & Pair & Geometry & Affected & Harmed & $\rho_q$ & All-full / Oracle@20 \\
\midrule
nuScenes, $q_{\mathrm{brake}}$ & Y8 320$\to$640 & oracle & 141 & 35.5\% & 0.42 & 14.01 / 24.07\% \\
nuScenes, $q_{\mathrm{plan}}$ & Y8 320$\to$640 & oracle & 1{,}179 & 50.9\% & 0.70 & \phantom{0}1.31 / \phantom{0}4.41\% \\
nuScenes, $q_{\mathrm{plan}}^{\,\mathrm{self}}$ & Y8 320$\to$640 & oracle & 1{,}413 & 40.9\% & 0.46 & \phantom{0}7.49 / 13.67\% \\
KITTI, $q_{\mathrm{traj}}$ & Y8 320$\to$640 & mono & \phantom{0,}678 & 48.4\% & 0.28 & 12.78 / 17.81\% \\
KITTI, $q_{\mathrm{traj}}$ & Y8 320$\to$640 & oracle & \phantom{0,}463 & \phantom{0}3.0\% & 0.003 & 68.34 / 68.58\% \\
KITTI, $q_{\mathrm{traj}}$ & Y8 384$\to$640 & mono & \phantom{0,}458 & 48.7\% & 0.46 & \phantom{0}5.97 / 11.03\% \\
KITTI, $q_{\mathrm{traj}}$ & Y8 512$\to$640 & mono & \phantom{0,}266 & 53.8\% & 1.11 & $-0.59$ / \phantom{0}5.16\% \\
KITTI, $q_{\mathrm{traj}}$ & RT 480$\to$640 & mono & \phantom{0,}295 & 43.4\% & 0.58 & \phantom{0}2.94 / \phantom{0}6.98\% \\
\bottomrule
\end{tabular}
\end{table}

\paragraph{Persistence and oracle geometry.} Two controls on KITTI test the sensitivity of the sign variation to transient detections and to monocular lifting. Passing a detection to the controller only if a detection of the same class with image IoU of at least 0.3 appears in each of the two previous frames of the same mode adds 0.2\,s of latency at 10\,Hz, drops 26--32\% of detections, and moves the harmed share by at most 6.8 points across the twelve settings; only the braking controller at YOLOv8s 320$\to$640 sees $\rho_q$ fall by a third or more, from 0.17 to 0.11 under monocular and from 0.08 to 0.05 under oracle geometry. Table~\ref{tab:refgeom} extends oracle geometry to every fidelity pair. The braking controller's harm persists when matched detections carry reference geometry, at 25--48\% of affected inputs. The trajectory controller's harm ratio collapses on four of five pairs while 3--37\% of its affected inputs stay harmed, so range error on real objects drives the size of its harm more than its frequency. Oracle geometry replaces the geometry of matched detections only: 73\% of KITTI and 71\% of nuScenes frames contain an unmatched detection, which keeps its monocular geometry.

\begin{table}[htbp]
\caption{Harm with deployed monocular and oracle geometry on every KITTI fidelity pair, over all units. Harmed: share of affected inputs made worse by the full mode (\%); $\rho_q$: harm ratio.}
\label{tab:refgeom}
\centering
\small
\setlength{\tabcolsep}{3.5pt}
\begin{tabular}{lrrrrrrrr}
\toprule
& \multicolumn{4}{c}{$q_{\mathrm{brake}}$} & \multicolumn{4}{c}{$q_{\mathrm{traj}}$} \\
\cmidrule(lr){2-5}\cmidrule(lr){6-9}
& \multicolumn{2}{c}{Harmed} & \multicolumn{2}{c}{$\rho_q$} & \multicolumn{2}{c}{Harmed} & \multicolumn{2}{c}{$\rho_q$} \\
Pair & mono & oracle & mono & oracle & mono & oracle & mono & oracle \\
\midrule
YOLOv8s 320$\to$640 & 34.2 & 24.6 & 0.168 & 0.080 & 48.4 & 3.0 & 0.282 & 0.003 \\
YOLOv8s 384$\to$640 & 39.8 & 35.3 & 0.473 & 0.295 & 48.7 & 5.1 & 0.459 & 0.002 \\
YOLOv8s 512$\to$640 & 42.9 & 39.7 & 0.615 & 0.467 & 53.8 & 17.0 & 1.114 & 0.051 \\
RT-DETR-l 320$\to$640 & 44.2 & 47.5 & 0.701 & 0.691 & 49.7 & 19.9 & 0.762 & 0.023 \\
RT-DETR-l 480$\to$640 & 42.2 & 47.1 & 0.720 & 0.884 & 43.4 & 36.8 & 0.579 & 0.141 \\
\bottomrule
\end{tabular}
\end{table}

\section{Per-Mode Operating Points}
\label{app:calib}

All other results use one detection threshold, 0.25, for both modes. Here each fidelity has its own operating point, selected on the training units of each track from precision and recall only, never from decision value. \textbf{S1} sets the full-mode threshold so that it emits as many boxes per frame as the cheap mode; \textbf{S2} places each mode at its F1 optimum; \textbf{S3} uses the lowest full-mode threshold whose precision is at least that of the cheap mode. Under S1 on nuScenes, the full mode goes from 4.18 to 2.59 boxes per frame and from precision 0.58 to 0.76, above the cheap mode's 0.65. Table~\ref{tab:calib} gives harmed share and $\rho_q$ over all units. Under each scheme, every nuScenes cell and every moderate-gap KITTI cell keeps a harmed share of at least 34\% and $\rho_q\ge 0.33$. Among the KITTI 320$\to$640 cells, which Table~\ref{tab:calib} omits, the braking and oracle-geometry cells keep $\rho_q$ below 0.18 under every scheme, while the monocular trajectory cell has $\rho_q$ of 0.28--0.35. Across all 25 threshold pairs in \mbox{$\{0.15,\dots,0.55\}^2$} the lowest nuScenes harmed share is 24.4\%, and thresholds tuned on the training and validation units to each mode's own downstream loss give harmed shares of 33--49\% and $\rho_q$ of 0.44--0.84 in the nuScenes cells and the moderate-gap KITTI braking cells.

\begin{table}[htbp]
\caption{Harmed share and harm ratio $\rho_q$ with a shared threshold (S0) and with per-mode operating points (S1--S3), over all units. KITTI rows use monocular geometry.}
\label{tab:calib}
\centering
\small
\setlength{\tabcolsep}{4pt}
\begin{tabular}{lcccc}
\toprule
Cell & S0 (shared) & S1 (equal boxes) & S2 (F1) & S3 (precision) \\
\midrule
nuScenes oracle, $q_{\mathrm{brake}}$ & 35.5\% / 0.42 & 38.5\% / 0.54 & 34.5\% / 0.39 & 34.1\% / 0.41 \\
nuScenes oracle, $q_{\mathrm{plan}}$ & 50.9\% / 0.70 & 50.2\% / 0.78 & 50.1\% / 0.71 & 50.2\% / 0.72 \\
nuScenes mono, $q_{\mathrm{brake}}$ & 39.6\% / 0.41 & 36.2\% / 0.34 & 37.0\% / 0.34 & 37.5\% / 0.37 \\
nuScenes mono, $q_{\mathrm{plan}}$ & 49.3\% / 0.71 & 49.6\% / 0.83 & 49.3\% / 0.73 & 49.2\% / 0.71 \\
\midrule
KITTI Y8 384$\to$640, $q_{\mathrm{traj}}$ & 48.7\% / 0.46 & 48.1\% / 0.48 & 48.2\% / 0.47 & 48.6\% / 0.47 \\
KITTI Y8 384$\to$640, $q_{\mathrm{brake}}$ & 39.8\% / 0.47 & 34.9\% / 0.44 & 37.1\% / 0.45 & 38.8\% / 0.47 \\
KITTI Y8 512$\to$640, $q_{\mathrm{traj}}$ & 53.8\% / 1.11 & 53.6\% / 1.15 & 52.9\% / 1.06 & 53.6\% / 1.11 \\
KITTI Y8 512$\to$640, $q_{\mathrm{brake}}$ & 42.9\% / 0.62 & 41.6\% / 0.62 & 41.2\% / 0.60 & 42.8\% / 0.62 \\
KITTI RT 480$\to$640, $q_{\mathrm{traj}}$ & 43.4\% / 0.58 & 43.4\% / 0.58 & 46.6\% / 0.52 & 43.4\% / 0.58 \\
KITTI RT 480$\to$640, $q_{\mathrm{brake}}$ & 42.2\% / 0.72 & 41.9\% / 0.74 & 41.9\% / 0.84 & 42.2\% / 0.72 \\
\bottomrule
\end{tabular}
\end{table}

\section{Benchmark Reference Results}
\label{app:bench}

\paragraph{Held-out allocation.} Table~\ref{tab:heldout} gives every deployable baseline at a 20\% budget on all fourteen held-out cells, including the final-displacement variant of $q_{\mathrm{plan}}$. Out of fold over all units the same gates look stronger --- $0.510$ on nuScenes braking under monocular geometry, for instance --- and those numbers are released but not scored, since the gates were fit on those units. On the held-out split, the gradient-boosted gate's realized gain on that cell exceeds random's by 0.087 $[0.014,0.154]$ of the all-cheap loss. Across all four budgets, on realized decision gain with every bootstrap draw kept, the ridge and gradient-boosted gates beat random in 4 and 10 rows, the detection-list router variants in 9--11 each, mostly on KITTI, and the pixel router in none (Table~\ref{tab:routers}). Among the untrained signals, ego speed beats random in ten rows, on both KITTI braking cells at every budget from 20\% up, on KITTI's trajectory controller under oracle geometry at 30 and 50\%, and on nuScenes braking under oracle geometry and PDM-Closed's scalar loss at 50\%; cheap-side criticality and uncertainty do so once each, on nuScenes braking and on KITTI's trajectory controller under monocular geometry at 30\%. Figure~\ref{fig:alloc} summarizes the ten core cells.

\begin{figure}[htbp]
\centering
\includegraphics{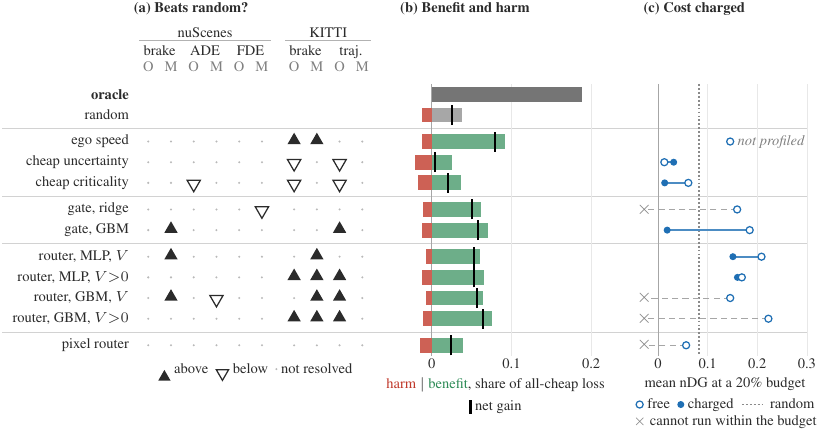}
\caption{\textbf{Scorecard of the deployable allocators on all ten core cells (YOLOv8s 320$\to$640) at a 20\% budget.}
\textbf{(a)} Selection budget: realized decision gain against random per cell, by the paired 95\% interval with every bootstrap draw kept ($\blacktriangle$ above zero, $\triangledown$ below, dot: not resolved); cells share frames, so counts are descriptive. O, M: oracle, monocular geometry; ADE, FDE: the learned planner's average, final displacement.
\textbf{(b)} Mean benefit captured (green), harm selected (red) and net gain (black), in shares of the all-cheap loss.
\textbf{(c)} Mean nDG with the allocator free (open) and with its measured latency charged (filled); $\times$: its own cost exceeds the budget. Per-cell values: Tables~\ref{tab:routers} and~\ref{tab:heldout}.}
\label{fig:alloc}
\end{figure}

The detection-list router is a two-layer MLP with 64 units per layer, or the benchmark's gradient-boosted model, on the 25 most confident cheap detections (the 25 nearest cheap tracks on nuPlan), with $V$ or $\mathbf{1}[V>0]$ as target. The pixel router is a MobileNetV2 \citep{sandler2018mobilenetv2} of 0.15\,GFLOPs on a $128\times128$ frame, trained from scratch on $\mathbf{1}[V>0]$ and served with TensorRT FP16. Both are adaptations of published detector routers rather than reimplementations, and differ from them in ways that matter here. ORIC estimates a per-image detection reward \citep{qiu2024edge}; our detection-list router estimates decision value in its regression variants and $P(V>0)$ in its classification variants, refitted per downstream system; the published objectives themselves are evaluated on these architectures in the paragraph on published routing objectives below. Budget-adaptive routing skips the weak pass when the strong model answers anyway \citep{geng2026budget}; our cascade pays the cheap mode on every input, so the pixel router only decides whether the full mode is added, and its measured cost excludes the image decoding that the cheap pass already pays. Charged instead as a skipping router, where an escalated input runs the full mode in place of the cheap one, it can escalate 83.5\% of KITTI inputs at the 50\% latency budget against 23.9\% under the cascade, and is still rarely better than random: over the 8 latency and 30 module-energy cell and budget combinations where that share is positive, it beats random once in each, on KITTI's braking controller under oracle geometry at the 50\% budget, and loses in none and two, respectively. On nuPlan the detection-list router reads the cheap branch's object list rather than a detector's output, and most of its slots hold objects that escalation cannot change.

\begin{table}[htbp]
\caption{Router variants, nDG at a 20\% budget on the held-out test cells; Figure~\ref{fig:alloc} summarizes the ten core cells. Bold: the paired 95\% interval of realized decision gain against random, over every bootstrap draw, lies above zero; asterisk: below zero. nuPlan rows use real detector outcomes; the pixel router does not apply to nuPlan.}
\label{tab:routers}
\centering
\small
\setlength{\tabcolsep}{4pt}
\begin{tabular}{lrrrrr}
\toprule
& \multicolumn{4}{c}{Detection-list router} & \\
\cmidrule(lr){2-5}
Cell & MLP, $V$ & MLP, $V{>}0$ & GBM, $V$ & GBM, $V{>}0$ & Pixel router \\
\midrule
nuScenes oracle, $q_{\mathrm{brake}}$ & 0.354 & 0.354 & 0.196 & 0.421 & $-0.048$ \\
nuScenes oracle, planner ADE & $-0.079$ & $-0.074$ & $-0.083$ & 0.031 & $-0.163$ \\
nuScenes oracle, planner FDE & 0.037 & $-0.047$ & $-0.093$ & 0.065 & $-0.034$ \\
nuScenes mono, $q_{\mathrm{brake}}$ & \textbf{0.366} & 0.255 & \textbf{0.383} & 0.369 & 0.215 \\
nuScenes mono, planner ADE & 0.250 & 0.102 & $-0.022^{*}$ & 0.090 & 0.082 \\
nuScenes mono, planner FDE & 0.260 & 0.065 & $-0.062$ & 0.131 & 0.066 \\
\midrule
KITTI oracle, $q_{\mathrm{brake}}$ & 0.255 & \textbf{0.313} & 0.293 & \textbf{0.360} & 0.177 \\
KITTI oracle, $q_{\mathrm{traj}}$ & 0.356 & \textbf{0.422} & \textbf{0.516} & \textbf{0.484} & 0.199 \\
KITTI mono, $q_{\mathrm{brake}}$ & \textbf{0.178} & \textbf{0.172} & \textbf{0.201} & \textbf{0.215} & 0.093 \\
KITTI mono, $q_{\mathrm{traj}}$ & 0.107 & 0.127 & 0.125 & 0.059 & $-0.021$ \\
\midrule
nuPlan PDM-Closed, safety & 0.143 & $-0.067$ & 0.072 & 0.069 & --- \\
nuPlan PDM-Closed, scalar & 0.141 & 0.278 & 0.065 & 0.072 & --- \\
nuPlan IDM, safety & \multicolumn{5}{c}{too few affected test states (4)} \\
nuPlan IDM, scalar & 0.031 & 0.938 & 0.000 & 0.059 & --- \\
\bottomrule
\end{tabular}
\end{table}

\begin{table}[htbp]
\caption{Deployable baselines, nDG at a 20\% budget on all held-out test cells; Figure~\ref{fig:alloc} summarizes the ten core cells. Bold: the paired 95\% interval of realized decision gain against random, over every bootstrap draw, lies above zero; asterisk: below zero. nuPlan rows use the real-perception track; $^\dagger$no uncertainty signal is defined on it. Planner rows are $q_{\mathrm{plan}}$ with average (ADE) and final (FDE) displacement.}
\label{tab:heldout}
\centering
\small
\setlength{\tabcolsep}{3pt}
\begin{tabular}{lrrrrrr}
\toprule
Cell & Random & Ego speed & Uncertainty & Criticality & Gate, ridge & Gate, GBM \\
\midrule
nuScenes oracle, $q_{\mathrm{brake}}$ & 0.148 & 0.208 & 0.015 & 0.470 & 0.254 & 0.155 \\
nuScenes oracle, planner ADE & $-0.006$ & $-0.021$ & $-0.117$ & $-0.272^{*}$ & $-0.150$ & $-0.141$ \\
nuScenes oracle, planner FDE & 0.059 & 0.024 & $-0.016$ & $-0.142$ & 0.030 & 0.125 \\
nuScenes mono, $q_{\mathrm{brake}}$ & 0.132 & 0.148 & 0.088 & 0.407 & 0.293 & \textbf{0.482} \\
nuScenes mono, planner ADE & 0.092 & $-0.015$ & 0.013 & $-0.060$ & 0.063 & 0.042 \\
nuScenes mono, planner FDE & 0.080 & 0.007 & $-0.059$ & $-0.048$ & $-0.043^{*}$ & 0.079 \\
\midrule
KITTI oracle, $q_{\mathrm{brake}}$ & 0.139 & \textbf{0.584} & $-0.010^{*}$ & $-0.013^{*}$ & 0.238 & 0.226 \\
KITTI oracle, $q_{\mathrm{traj}}$ & 0.199 & 0.795 & $0.031^{*}$ & $0.015^{*}$ & 0.455 & \textbf{0.538} \\
KITTI mono, $q_{\mathrm{brake}}$ & 0.079 & \textbf{0.415} & 0.035 & 0.074 & 0.156 & 0.158 \\
KITTI mono, $q_{\mathrm{traj}}$ & $-0.091$ & $-0.691$ & 0.148 & 0.178 & 0.298 & 0.184 \\
\midrule
nuPlan PDM-Closed, safety & 0.128 & 0.205 & ---$^\dagger$ & 0.070 & 0.992 & 0.992 \\
nuPlan PDM-Closed, scalar & 0.125 & 0.201 & ---$^\dagger$ & 0.066 & \textbf{0.988} & 0.989 \\
nuPlan IDM, safety & \multicolumn{6}{c}{too few affected test states (4)} \\
nuPlan IDM, scalar & 0.185 & $-0.076$ & ---$^\dagger$ & 0.008 & 0.056 & 0.000 \\
\bottomrule
\end{tabular}
\end{table}

On nuPlan, all pre-escalation inputs are rebuilt from the cheap branch with real detections, including its false positives. The test split contains 27 and 48 affected states for PDM-Closed under the safety and scalar losses, and 4 and 10 for IDM. Most of the oracle's gain on PDM-Closed's test split lies in one log and one scenario, in which the ego moves faster (8.5 against 3.9\,m/s) among many objects; the gates rank that scenario's states first, which is what their near-perfect nDG reflects, and bootstrap draws that omit that log carry little of that gain.

\begin{table}[htbp]
\caption{Diagnostic signals, nDG at a 20\% budget on the held-out test cells. Each column is one fixed signal, not a per-cell selection. Bold: the paired 95\% interval of realized decision gain against random, over every bootstrap draw, lies above zero; asterisk: below zero. These signals require the full output or the reference state. PKL and TIP are defined only on nuScenes. On the real-perception nuPlan track the false-negative gain counts missed tracks; neither it nor the risk-weighted gain sees the false positives that drive most of PDM-Closed's decision value.}
\label{tab:diag-heldout}
\centering
\small
\setlength{\tabcolsep}{4pt}
\begin{tabular}{lrrrrr}
\toprule
Cell & Criticality (reference) & $G^{\mathcal{E}_{\mathrm{FN}}}$ & $G^{\mathcal{E}_{\mathrm{risk}}}$ & PKL & TIP \\
\midrule
nuScenes oracle, $q_{\mathrm{brake}}$ & 0.287 & 0.258 & \textbf{0.687} & 0.461 & 0.365 \\
nuScenes oracle, planner ADE & $-0.195$ & 0.123 & $-0.085$ & $-0.132$ & $-0.181$ \\
nuScenes oracle, planner FDE & $-0.005$ & 0.121 & $-0.036$ & 0.124 & $-0.035$ \\
nuScenes mono, $q_{\mathrm{brake}}$ & 0.323 & 0.210 & 0.330 & 0.215 & 0.110 \\
nuScenes mono, planner ADE & $-0.051^{*}$ & 0.350 & 0.404 & 0.126 & 0.117 \\
nuScenes mono, planner FDE & $-0.037$ & 0.264 & 0.246 & 0.077 & 0.168 \\
\midrule
KITTI oracle, $q_{\mathrm{brake}}$ & 0.072 & \textbf{0.269} & 0.178 & --- & --- \\
KITTI oracle, $q_{\mathrm{traj}}$ & $0.015^{*}$ & 0.534 & 0.057 & --- & --- \\
KITTI mono, $q_{\mathrm{brake}}$ & 0.084 & 0.098 & 0.056 & --- & --- \\
KITTI mono, $q_{\mathrm{traj}}$ & \textbf{0.404} & \textbf{0.473} & \textbf{0.400} & --- & --- \\
\midrule
nuPlan PDM-Closed, safety & 0.070 & 0.105 & 0.140 & --- & --- \\
nuPlan PDM-Closed, scalar & 0.068 & 0.096 & 0.136 & --- & --- \\
nuPlan IDM, safety & \multicolumn{5}{c}{too few affected test states (4)} \\
nuPlan IDM, scalar & 0.008 & 0.001 & 0.062 & --- & --- \\
\bottomrule
\end{tabular}
\end{table}

\begin{figure}[htbp]
\centering
\begin{tikzpicture}
\pgfplotsset{dstyle/.style={width=4.75cm, height=4.8cm,
  xmin=-0.17, xmax=0.66, ymin=0.4, ymax=7.6, ytick={1,2,3,4,5,6,7},
  xtick={0,0.2,0.4,0.6}, grid=major, grid style={gray!15}, axis line style={gray!55}, tick style={gray!55},
  tick label style={font=\scriptsize}, title style={font=\scriptsize}, xlabel={\scriptsize nDG at a 20\% budget}, clip=false}}
\begin{axis}[dstyle, name=d1, title={KITTI oracle, $q_{\mathrm{traj}}$},
  yticklabels={cheap uncertainty, cheap criticality, pixel router, gate (ridge), gate (GBM), router (MLP{,} $V$), router (MLP{,} $V{>}0$)}]
\draw[cRand, densely dotted, line width=0.9pt] (axis cs:0.199,0.4) -- (axis cs:0.199,7.6);
\dumb{7}{0.422}{0.369}{cPKL}
\dumb{6}{0.356}{0.302}{cOracle}
\dumb{5}{0.538}{0.017}{cPurple}
\dumbinf{4}{0.455}{cRisk}
\dumbinf{3}{0.199}{cBrown}
\dumb{2}{0.015}{0.001}{cDE}
\dumb{1}{0.031}{0.016}{cTIP}
\end{axis}
\begin{axis}[dstyle, name=d2, at={(d1.east)}, anchor=west, xshift=3mm, yticklabels={}, title={KITTI oracle, $q_{\mathrm{brake}}$}]
\draw[cRand, densely dotted, line width=0.9pt] (axis cs:0.139,0.4) -- (axis cs:0.139,7.6);
\dumb{7}{0.313}{0.279}{cPKL}
\dumb{6}{0.255}{0.190}{cOracle}
\dumb{5}{0.226}{0.016}{cPurple}
\dumbinf{4}{0.238}{cRisk}
\dumbinf{3}{0.177}{cBrown}
\dumb{2}{-0.013}{-0.004}{cDE}
\dumb{1}{-0.010}{-0.013}{cTIP}
\end{axis}
\begin{axis}[dstyle, name=d3, at={(d2.east)}, anchor=west, xshift=3mm, yticklabels={}, title={nuScenes mono, $q_{\mathrm{brake}}$}]
\draw[cRand, densely dotted, line width=0.9pt] (axis cs:0.132,0.4) -- (axis cs:0.132,7.6);
\dumb{7}{0.255}{0.297}{cPKL}
\dumb{6}{0.366}{0.264}{cOracle}
\dumb{5}{0.482}{0.157}{cPurple}
\dumbinf{4}{0.293}{cRisk}
\dumbinf{3}{0.215}{cBrown}
\dumb{2}{0.407}{0.184}{cDE}
\dumb{1}{0.088}{0.123}{cTIP}
\end{axis}
\node[font=\scriptsize, anchor=north] at ($(d2.south)+(0,-10.5mm)$) {%
  \tikz{\draw[black, line width=0.9pt] (0,0) circle (2pt);}~selection budget (allocator free) \quad
  \tikz{\fill[black] (0,0) circle (2pt);}~measured latency budget (allocator pays its cost) \quad
  \tikz{\draw[cRand, densely dotted, line width=0.9pt] (0,0) -- (0.5,0);}~random};
\end{tikzpicture}
\caption{\textbf{Charging allocator cost reorders the allocators.} nDG on the held-out split at a 20\% budget, when the allocator is free (open) and when its measured latency on the device is charged (filled; the gradient-boosted gate with batched inference). The feature-based gates lose their advantage because feature extraction consumes the budget (inf.: the allocator cannot run within it at all), while both MLP variants of the detection-list router, at about 3.5\% of a full pass, keep most of it; its gradient-boosted variants cannot run within any measured budget. Figure~\ref{fig:budget-curves} covers every budget level.}
\label{fig:budget}
\end{figure}

\begin{figure}[htbp]
\centering
\begin{tikzpicture}
\pgfplotsset{
  budstyle/.style={
    width=4.95cm, height=4.0cm,
    xmin=7, xmax=53, ymin=-0.15, ymax=1.0,
    xtick={10,20,30,50}, ytick={0,0.2,0.4,0.6,0.8,1.0},
    xlabel={\scriptsize latency budget (\%)},
    grid=major, grid style={gray!18},
    axis line style={gray!55}, tick style={gray!55},
    tick label style={font=\scriptsize}, title style={font=\scriptsize},
    mark size=1.2pt, line width=0.8pt,
  }
}
\begin{axis}[budstyle, name=b1, ylabel={\scriptsize nDG}, title={KITTI oracle, $q_{\mathrm{traj}}$}]
\addplot[cRand, densely dotted, mark=none] coordinates {(10,0.100)(20,0.199)(30,0.299)(50,0.499)};
\addplot[cTIP, mark=diamond*] coordinates {(10,-0.001)(20,0.016)(30,0.049)(50,0.083)};
\addplot[cDE, mark=square*] coordinates {(10,0.000)(20,0.001)(30,0.014)(50,0.015)};
\addplot[cPKL, mark=*, line width=1.4pt] coordinates {(10,0.213)(20,0.369)(30,0.601)(50,0.836)};
\addplot[cOracle, mark=o, line width=1.0pt] coordinates {(10,0.146)(20,0.302)(30,0.469)(50,0.722)};
\addplot[cRisk, dashed, mark=triangle*] coordinates {(10,nan)(20,nan)(30,0.206)(50,0.505)};
\addplot[cPurple, mark=triangle*] coordinates {(10,nan)(20,0.017)(30,0.424)(50,0.755)};
\addplot[cBrown, mark=x] coordinates {(10,nan)(20,nan)(30,-0.000)(50,0.301)};
\end{axis}
\begin{axis}[budstyle, name=b2, at={(b1.east)}, anchor=west, xshift=4mm, yticklabels={,,}, title={KITTI oracle, $q_{\mathrm{brake}}$}]
\addplot[cRand, densely dotted, mark=none] coordinates {(10,0.070)(20,0.139)(30,0.208)(50,0.348)};
\addplot[cTIP, mark=diamond*] coordinates {(10,-0.016)(20,-0.013)(30,0.043)(50,0.035)};
\addplot[cDE, mark=square*] coordinates {(10,0.005)(20,-0.004)(30,-0.013)(50,0.061)};
\addplot[cPKL, mark=*, line width=1.4pt] coordinates {(10,0.107)(20,0.279)(30,0.384)(50,0.598)};
\addplot[cOracle, mark=o, line width=1.0pt] coordinates {(10,0.096)(20,0.190)(30,0.366)(50,0.539)};
\addplot[cRisk, dashed, mark=triangle*] coordinates {(10,nan)(20,nan)(30,0.053)(50,0.316)};
\addplot[cPurple, mark=triangle*] coordinates {(10,nan)(20,0.016)(30,0.083)(50,0.494)};
\addplot[cBrown, mark=x] coordinates {(10,nan)(20,nan)(30,0.047)(50,0.203)};
\end{axis}
\begin{axis}[budstyle, name=b3, at={(b2.east)}, anchor=west, xshift=4mm, yticklabels={,,}, title={nuScenes mono, $q_{\mathrm{brake}}$},
  legend style={font=\scriptsize, at={(-0.62,-0.32)}, anchor=north, legend columns=4,
                draw=gray!40, inner sep=2pt, /tikz/every even column/.append style={column sep=5pt}}]
\addplot[cRand, densely dotted, mark=none] coordinates {(10,0.066)(20,0.132)(30,0.198)(50,0.330)};
\addplot[cTIP, mark=diamond*] coordinates {(10,0.023)(20,0.123)(30,0.271)(50,0.294)};
\addplot[cDE, mark=square*] coordinates {(10,0.052)(20,0.184)(30,0.410)(50,0.472)};
\addplot[cPKL, mark=*, line width=1.4pt] coordinates {(10,0.179)(20,0.297)(30,0.228)(50,0.219)};
\addplot[cOracle, mark=o, line width=1.0pt] coordinates {(10,0.127)(20,0.264)(30,0.333)(50,0.620)};
\addplot[cRisk, dashed, mark=triangle*] coordinates {(10,nan)(20,nan)(30,0.250)(50,0.337)};
\addplot[cPurple, mark=triangle*] coordinates {(10,nan)(20,0.157)(30,0.230)(50,0.493)};
\addlegendimage{cBrown, mark=x} %
\legend{random, cheap-detection uncertainty, cheap-side criticality, {router (MLP, $V{>}0$)}, {router (MLP, $V$)}, {gate (ridge)}, {gate (GBM, batched)}, pixel router}
\end{axis}
\end{tikzpicture}
\caption{\textbf{nDG across measured latency budgets.} Held-out split, the three core cells of Figure~\ref{fig:budget}, every budget level. Missing points are infeasible: the ridge gate cannot run within the 10 and 20\% budgets, the batched gradient-boosted gate within 10\%, and the pixel router within 10 and 20\% on KITTI and any budget on nuScenes. Both MLP variants of the detection-list router, at about 3.5\% of a full pass, keep their ranking skill.}
\label{fig:budget-curves}
\end{figure}

\paragraph{Typical uncertainty per cell.} Table~\ref{tab:power} gives the median half-width of the paired interval on realized decision value at a 20\% budget, as a share of the all-cheap loss. It describes the typical uncertainty of a comparison with random in each cell, not a minimum detectable effect, and whether a particular allocator beats random is decided by its own interval. Typical half-widths are below 2\% in the planner cells, 4--7\% in the braking cells, and 14--18\% in the three cells that carry the largest reported wins, KITTI's trajectory controller under oracle geometry and the two PDM-Closed losses. The gradient-boosted gate's advantage on KITTI's trajectory controller is 17.6\% of the all-cheap loss, and its own interval excludes zero; those of PDM-Closed's safety gates do not (below).

\begin{table}[htbp]
\caption{Typical uncertainty of paired comparisons with random at a 20\% budget: the median half-width, over all scored signals, of the paired 95\% interval on realized decision value, as a share of the all-cheap loss, with every bootstrap draw kept.}
\label{tab:power}
\centering
\small
\begin{tabular}{lrlr}
\toprule
Cell & Half-width & Cell & Half-width \\
\midrule
nuScenes oracle, planner ADE & 0.7\% & KITTI mono, $q_{\mathrm{brake}}$ & 3.7\% \\
nuPlan IDM, scalar & 1.0\% & nuScenes mono, $q_{\mathrm{brake}}$ & 5.8\% \\
nuScenes oracle, planner FDE & 1.1\% & nuScenes oracle, $q_{\mathrm{brake}}$ & 6.0\% \\
nuScenes mono, planner ADE & 1.1\% & KITTI oracle, $q_{\mathrm{brake}}$ & 6.5\% \\
nuPlan IDM, safety & 1.1\% & nuPlan PDM-Closed, scalar & 14.0\% \\
nuScenes mono, planner FDE & 1.7\% & nuPlan PDM-Closed, safety & 16.0\% \\
KITTI mono, $q_{\mathrm{traj}}$ & 2.3\% & KITTI oracle, $q_{\mathrm{traj}}$ & 17.6\% \\
\bottomrule
\end{tabular}
\end{table}

\paragraph{Streaming allocation.} The selection budget of Figure~\ref{fig:alloc} is filled by ranking a whole test split, which no deployed allocator can do. Over the twelve pooled cells and six learned signals at a 20\% rate, with every bootstrap draw kept, a score threshold calibrated on the validation units and frozen before the test stream differs from that ranking in nDG by $0.00$ $[-0.05,+0.09]$, while 49 of the 72 rows miss the rate by more than five points, and KITTI escalates a median of 29--33\% of inputs at a 20\% target.

Enforcing the rate causally, with at most $\lfloor 1+\alpha t\rfloor$ escalations after $t$ inputs of a unit, lowers nDG by $0.09$ $[-0.17,+0.00]$, because escalations arrive in bursts that the cap truncates; the interval reaches zero, so the size of this cost is not resolved. Rate controllers tuned on validation units do not recover it under the pre-specified selection rule: an adaptive threshold that tracks the rate online cuts the rows missing it from 54 under the cap to 14 and still lowers nDG by $0.07$ $[-0.15,+0.02]$, and the best token bucket has unbounded capacity and so reduces to the cap. On random scores the cap costs about 0.01 and the adaptive threshold nothing, so what loss there is belongs to the learned rankings under a causal budget. What the selection track reports is therefore the ranking quality of a signal; turning a ranking into a causal allocation at a fixed rate remains an open part of the problem.

\paragraph{Ranking allocators by perception gain.} \bench{} scores a signal by the decision value it realizes; a perception-first evaluation would score the same signals by the perception gain they realize instead. Holding the protocol and the tie expectation fixed and changing only that objective, we compare the sets of optimal signals, with ties within $10^{-9}$, on the 52 cell--budget pairs whose decision value is defined on test. On the target-free deployable pool --- random, uncertainty, cheap-side criticality and ego speed, none of them fitted to either objective --- the two objectives share no optimum in 43 of the 52 pairs when the gain is the reduction in missed objects and in 44 when it is the risk-weighted gain, with a median Kendall $\tau$ of $-0.33$ between the two rankings. Adding the learned deployable signals gives 50 and 50; the pool with the perception diagnostics gives 45 and 42, where the perception objective's winner is partly definitional, since the signal equal to the gain being scored attains its optimum. These counts are descriptive.

To estimate what choosing by perception gain costs, each objective selects its winner from the target-free pool on the validation units, and the two winners are compared once on test by realized decision gain, with the signal pools, test statistic and decision rule specified in advance. Over the ten core cells and PDM-Closed's scalar loss (the other nuPlan cells have at most seven affected validation states), selecting by the missed-object gain lowers realized test gain by $0.074$ $[0.010,0.108]$ of the all-cheap loss at a 20\% budget, and by $0.039$--$0.091$ at the other budgets, each interval excluding zero. Under the risk-weighted gain the reduction at 20\%, $0.064$ $[-0.004,0.099]$, is not resolved. Most of the effect lies in KITTI under oracle geometry, where the decision objective selects ego speed and the perception objectives select uncertainty or cheap-side criticality; in four nuScenes cell--objective pairs both select the same signal. Figure~\ref{fig:choice}a shows each cell; on the six nuScenes cells of the 850-scene track the difference, choice by decision value minus choice by the false-negative gain, is $-0.006$ $[-0.031,0.017]$.

\paragraph{Training on perception gain instead.} The objective can also be swapped at training time. We refit every learned allocator on a perception-gain label, with the same inputs, hyperparameters, seeds and training units, and score it against decision value; the comparison and its interpretation were specified in advance, and the same fitting procedure is used for the decision-value baselines.

The mean paired difference at a 20\% budget over the twelve pooled cells is $+0.16$ $[-0.04,+0.27]$ for the ridge gate and $+0.13$ $[-0.03,+0.23]$ for the gradient-boosted gate, and every interval includes zero, for all six architectures and all three perception-gain labels. Five of the six lean towards the decision-value target and the gradient-boosted regression router leans the other way. The effect is cell-specific rather than uniform: the decision-value target wins significantly in 10 of the 72 pooled cell and architecture pairs and the perception-gain target in 8: four are the gradient-boosted routers on PDM-Closed, where the two labels can be compared on only 11 and 28 training states, which is consistent with sparse decision-value supervision, and four are on the nuScenes planner. On the ten core cells the target moves the gates by $+0.052$ and $+0.049$. We therefore do not claim that the objective rather than the architecture produces the allocation advantage{\ on this split, which Appendix~\ref{app:fu-influence} revisits for the gate at larger scale}: Tables~\ref{tab:routers} and~\ref{tab:heldout} report what these architectures achieve when they are trained on decision value. This is the controlled counterpart of routers that predict perception quality, such as RTScale \citep{heo2022rtscale}, and, with the risk-weighted label, of a task-aware risk target \citep{antonante2023taskaware}: inputs, architecture and split are fixed and only the reward changes.

\paragraph{Published routing objectives.} The perception-gain labels above are perception losses, whereas the published detector routers optimize detection quality. We therefore also refit the four detection-list router variants and the pixel router on the two published objectives, taken from the papers and the authors' released implementations: ORIC's offloading reward, the gain in contextual mAP@0.5 (101-point interpolation, 1{,}000 sampled context frames) from replacing the weak detector's output on a frame with the strong detector's \citep{qiu2024edge}, and the objective of budget-adaptive routing, the change in dataset-wide AP@0.5 (one precision--recall curve, all-point interpolation) when one frame's weak output is swapped for the strong one while every other frame keeps its weak output \citep{geng2026budget}. Each enters in its published CDF-normalized form, with the normalization fitted on the fit units only, for the regression variants and in its binary form for the classification variants. Architecture, features, hyperparameters, seeds, training units and scorer are unchanged and only the label differs, so this evaluates the published objectives on our architectures rather than the published methods; AP is computed on the cached detections, which keep confidences down to 0.10, ORIC's reward at that floor and the budget-adaptive objective at its published threshold of 0.30; and the one component of a published method that is reimplemented, the budget-adaptive arbiter of \citet{geng2026budget}, picks the gradient-boosted estimator at every quota and so coincides with it. The nuPlan cells are excluded, since their observations carry no scored detections from which AP can be computed.

The choice of objective matters. Across the ten core cells and four quotas, a router trained on a published objective beats random where its decision-value counterpart does not in 24 rows, and the reverse holds in 84. The direction depends on the architecture: on the published objectives the gradient-boosted regression router is significantly ahead of its decision-value version in 10 and 11 of 40 rows and the pixel router in 2 and 2, while the MLP regression router falls behind in 4 rows under ORIC's reward. Decision-value training yields more wins against random overall, 40 against 12 and 15 across the five architectures, but no objective serves every architecture best.

\paragraph{Target transform.} The published objectives enter the regression routers in CDF-normalized form, whereas $V$ enters raw, so a difference between them may come from the objective or from the transform. To separate the two, we also refit both regression routers on the rank of $V$ and on a signed empirical CDF of $V$ ($F^{+}(V)$ for $V>0$, zero at zero, $F^{-}(V)-1$ for $V<0$, fitted on the fit units), with everything else unchanged (Figure~\ref{fig:transform}). This comparison was decided after the published-objective results were seen, and its intervals are not corrected for multiplicity.

For the MLP router, the deficit under ORIC's reward is consistent with a target-transform effect. Over the ten core cells at a 20\% budget, the reward trails raw $V$ by $0.139$ $[0.021,0.244]$ in nDG and rank$(V)$ trails it by $0.170$ $[0.048,0.244]$, while the reward and rank$(V)$ differ by only $+0.032$ $[-0.058,0.095]$. For the gradient-boosted router, no pooled comparison separates the AP-swap objective from a transform of $V$: it leads rank$(V)$ by $0.120$ $[-0.005,0.219]$ and the signed CDF of $V$ by $0.053$ $[-0.058,0.128]$, its own lead over raw $V$, $0.106$ $[-0.012,0.214]$, is not resolved, and the signed CDF recovers about half of that lead, as well as part of the published objectives' gain on the planner cells.

\begin{figure}[htbp]
\centering
\includegraphics{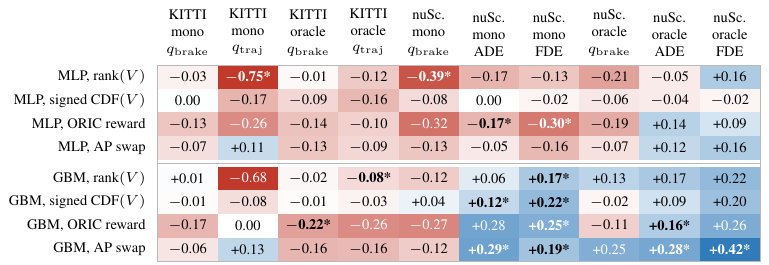}
\caption{\textbf{Target transform against training objective.} Difference in nDG from the same regression router trained on raw decision value, on the ten core cells at a 20\% budget; bold with an asterisk: paired 95\% interval excludes zero. Rank and signed CDF keep the decision-value objective and change only its transform; ORIC's reward and the AP-swap objective change the objective in its published CDF form. Post hoc, uncorrected for multiplicity.}
\label{fig:transform}
\end{figure}

\paragraph{Training seeds.} The bootstrap holds each fitted allocator fixed, so it contains no training-seed variation. Refitting with five further seeds leaves the gradient-boosted gate and routers bit-identical, since their fits use no randomness at these sample sizes, so each of their wins holds for every seed. The MLP routers vary: over six seeds their nDG at a 20\% budget spans up to $0.22$ within a cell for the regression router, $0.15$--$0.37$ on nuScenes braking under monocular geometry, and up to $0.28$ for the classification router. Of the five MLP wins on the core track in Table~\ref{tab:routers}, counted by the filtered nDG interval for each seed, the classification router's on KITTI's trajectory controller under oracle geometry holds for all six seeds and each of the others for two or three.

\paragraph{Transfer between downstream systems.} Decision value differs between downstream systems, but an allocator fitted for one of them need not. Scoring each cached allocator against another downstream system's decision value on that system's test split, the median transfer regret, the nDG of the transferred allocator minus that of the allocator fitted for the evaluated downstream system, is $-0.009$ over the 528 off-diagonal entries, and with every draw kept 14 of 78 diagonal entries beat random at a 20\% budget against 9 of 132 off-diagonal ones. These are descriptive counts over dependent entries, not a paired test of fitting against transferring. Cached router scores exist only for the frames of the downstream system they were fitted for, so 64 planner-to-braking entries cover 78.7\% of the braking frames; on the 464 off-diagonal entries with full coverage the median regret is $-0.004$ and 9 of 116 beat random at 20\%, so the counts do not come from missing scores.

On nuPlan the direction of transfer does: an allocator fitted for IDM realizes none of PDM-Closed's value, while the gradient-boosted gate fitted for PDM-Closed's safety loss reaches $0.94$ on IDM, where the same gate fitted for IDM reaches $0.00$. Dependence on the downstream system is a property of the decision values, and the allocators recover part of it.

\paragraph{Unnormalized gain.} nDG divides by the oracle's gain at the same budget, which is smallest in exactly the cells whose value sits in one unit. The benchmark therefore judges significance on realized decision gain with every bootstrap draw kept. Of the 18 deployable rows other than ego speed whose filtered nDG interval lies above random at a 20\% budget, 14 also beat it on realized gain: the KITTI and nuScenes rows, and PDM-Closed's scalar loss under the ridge gate, whose advantage over random is 30.5\% of the all-cheap loss. The four that do not are the two PDM-Closed safety gates, the scalar-loss gradient-boosted gate, and IDM's detection-list router; for the three gates the lower bound is exactly zero, because draws without the dominant log carry no difference to random. The filter on bootstrap draws described below is one-sided in this sense: across the four budgets it adds 4 to 11 wins and removes none. Removing that log leaves the gates' ratio at $0.99$ and their realized gain at 12.0 of 162.9.

\paragraph{All-full inference.} A cascade pays $C_{\mathsf{c}}$ on every input, so a per-input budget also admits running the full mode alone, without the cheap pass, once $b\ge C_{\mathsf{f}}$, that is, once the escalation allowance $b-C_{\mathsf{c}}$ reaches $C_{\mathsf{f}}-C_{\mathsf{c}}$: from $\alpha=28.6\%$ on KITTI, 34.3\% on nuScenes and 39.0\% on nuPlan under the measured latencies. At the 10 and 20\% budgets of Figure~\ref{fig:budget-curves} it is infeasible. Where it is feasible, what it realizes is the gross benefit less the harm it absorbs: on KITTI's trajectory controller under oracle geometry, where harm is rare, it reaches $\ndg=0.998$ against $0.601$ and $0.836$ for the detection-list router at the 30 and 50\% budgets. With overhead charged and causal enforcement, it leads the best feasible allocator in eight of ten core cells at a 50\% latency budget.

Charging each allocator its measured overhead, enforcing the budget causally and comparing loss reductions at the 50\% latency budget, all-full inference is ahead of the best feasible allocator in eight of the ten core cells; the two exceptions, the learned planner under oracle geometry and KITTI's trajectory controller under monocular geometry, are the cells in which all-full inference raises the loss on the test split in point estimate. It is therefore the baseline to beat once it fits; when every beneficial input fits within the budget, the harm ratio is what an allocator can still recover from it.

\paragraph{Quota and capacity.} A ranking submission escalates exactly its $\mathrm{round}(\alpha N)$ highest-scoring inputs, while the oracle of Eq.~(\ref{eq:oracle}) may leave capacity unused. On every test cell at the reported budgets at least that many inputs have non-negative value, so the two oracles coincide there. A score submission cannot abstain{; a 0/1 decision mask can, since it is scored as submitted, and Appendix~\ref{app:fu-cost} reports whether each causal rule stays within the budget}.

\paragraph{Bootstrap draws.} A draw is dropped when the oracle's gain in it falls below a quarter of its full-sample value, since nDG is not comparable across draws with near-zero denominators. At a 20\% budget this removes at most 8 of 1{,}000 draws on nine of the ten core cells and 89 on KITTI's trajectory controller under oracle geometry, but 330 and 354 on the two PDM-Closed cells, whose oracle gain rests on few states; their intervals are conditional on that filter. The bootstrap resamples test units with each fitted allocator held fixed, so it does not include the variation from refitting with other training seeds. Counts of cells in which a signal beats random are descriptive: the cells share frames, fidelities and allocators and are not independent replications.

\paragraph{Measured budgets.} Figure~\ref{fig:budget-curves} gives nDG for three core cells at every budget level. The allocator costs in Table~\ref{tab:overhead} come from one profiling run of every allocator on the device; energy is charged under one module convention for detectors and allocators alike, the GPU, CPU and SoC rails over idle, and every energy claim in this paragraph also holds when only the GPU rail or all rails are charged. On nuPlan, with its own detector costs, the batched gradient-boosted gate also beats random at the 20\% latency budget on PDM-Closed's scalar loss, escalating 5\% of states, and the classification-target detection-list router beats random on IDM's scalar loss at every latency and energy budget. A 0.15\,GFLOPs pixel router costs 26--50\% of a full pass and beats random nowhere.

On the KITTI fidelity ladder at a 20\% budget, the best 640-only allocation reaches 0.73 of the multi-fidelity optimum on braking and 0.57 on the trajectory controller, so choosing among fidelities is worth about a third more than choosing whether to escalate on braking and three quarters more on the trajectory controller. Latency and energy disagree about the intermediate modes --- 384\,px costs 0.78 of 640\,px in milliseconds but 0.57 in millijoules --- so at a 10\% budget the energy-optimal braking plan sends more inputs to 384\,px than the latency-optimal one and exceeds the latency budget by 21\%. The deployable multi-fidelity gate pays three model evaluations per input, cannot run within budgets below 30\%, and reaches $\ndg=0.068$ on braking and $0.172$ on the trajectory controller at 50\%.

\paragraph{Benefit captured and harm incurred.} Figure~\ref{fig:alloc}b shows these terms. The gap between an allocator and the oracle at the same budget splits exactly into the positive value the allocator misses and the harm it selects, less the positive value the budget forces even the oracle to leave out. At a 20\% budget the last term is negligible on the core cells, so the gap is mostly missed benefit. Averaged over the ten core cells, as shares of the all-cheap loss, the oracle realizes $0.189$; ego speed captures a benefit of $0.092$ at a harm of $0.012$, the gates and detection-list routers $0.060$--$0.075$ at $0.007$--$0.012$, the pixel router $0.039$ at $0.015$, and random $0.038$ at $0.013$. Per cell, the ten non-random deployable signals incur significantly less harm than random in 11 of 100 signal--cell pairs and more in 12, and capture more benefit in 11 and less in 9.

\paragraph{Overhead a ranking can afford.} Scoring each cached ranking at every capacity with the scorer's own budget rule gives the overheads at which it still realizes at least random's gain. At a 20\% latency budget on the core cells, the median largest such overhead is 1.6--3.0\,ms for the learned allocators. The MLP routers' measured 0.65\,ms lies inside that set in eight and seven of the ten cells, the batched gradient-boosted gate's 3.5\,ms in three, and the gradient-boosted routers' 15.9\,ms, the single-call gate's 21.4\,ms and the ridge gate's 3.9\,ms in none. Uncertainty and cheap-side criticality cost little but beat random in few cells even at no cost. Because $V$ is signed, the set of affordable overheads need not be an interval; at the measured overheads every result equals the budget tables.

\paragraph{Stability of the nuScenes wins.} The nuScenes wins also hold under the filtered nDG interval, and each keeps an nDG of at least $0.33$ when any single test scene is removed; the gradient-boosted fits do not depend on the training seed, whereas the MLP router's win holds for two of six seeds.

\paragraph{Where the gap to the oracle goes.} For the learned allocators, with train-only fits and validation-calibrated thresholds, we split the gap between the zero-overhead oracle and a causal allocator at the same latency budget into three terms: the oracle value lost to the allocator's overhead, the gap between that oracle and top-$k$ by the allocator's scores at the remaining capacity, and the difference between top-$k$ and the causal cap policy above. With one inference call per input, the gradient-boosted models cannot run at either the 20 or the 50\% budget. For the MLP routers at 20\%, over the ten core and two PDM-Closed cells, the gap averages $0.18$--$0.20$ of the all-cheap loss, of which $0.175$--$0.177$ is ranking, $0.007$--$0.021$ the causal cap, and $0.0001$ on average (at most $0.0012$ in any cell) overhead; at 50\% the ranking term is $0.094$--$0.125$ of a $0.138$--$0.156$ gap across the MLP routers and the ridge gate. The causal term is negative in 16 of the 62 feasible rows, where the cap skips harmful inputs that top-$k$ keeps, so the terms are an accounting of this protocol rather than losses or separate causes. {\label{sec:gap}This accounting assigns most of the gap to the ranking, the benefit it misses and the harm it selects, rather than to overhead or causal enforcement.}

\begin{table}[htbp]
\caption{Share of inputs each deployable allocator can escalate under a measured per-input latency budget after paying for itself, on nuScenes / KITTI (median over cells). inf.: the allocator's own overhead exceeds the budget, $C_S>b-C_{\mathsf{c}}$, so it cannot run at all. Diagnostic signals require the full output or the reference state and are not deployable under measured budgets.}
\label{tab:overhead}
\centering
\small
\begin{tabular}{lrrr}
\toprule
Allocator & Overhead per input & Escalated at 20\% budget & at 50\% budget \\
\midrule
Random & 0\,ms & 20.0 / 20.0\% & 50.0 / 50.0\% \\
Cheap-detection uncertainty & 0.21--0.29\,ms & 18.5 / 18.8\% & 48.5 / 48.8\% \\
Cheap-side criticality & 1.30--1.47\,ms & 12.4 / 12.9\% & 42.4 / 42.9\% \\
Detection-list router, MLP & 0.65\,ms & 16.6 / 16.5\% & 46.6 / 46.5\% \\
Gate, ridge & 3.93\,ms & inf. / inf. & 29.7 / 28.7\% \\
Gate, GBM, batched inference & 3.54\,ms & 1.7 / 0.8\% & 31.7 / 30.8\% \\
Gate, GBM, one call per input & 21.4\,ms & inf. / inf. & inf. / inf. \\
Pixel router & 9.8 / 4.8\,ms & inf. / inf. & inf. / 23.9\% \\
\bottomrule
\end{tabular}
\end{table}

\section{External Planner Track}
\label{app:external}

\paragraph{Setup.} PDM-Closed and IDM were checked against their released configurations field by field, and an unmodified run of the nuPlan closed-loop simulator succeeded for both on the same scenarios; this run is a configuration check, and the evaluation below is statewise and open-loop. For each of 1{,}440 logged states (60 scenarios $\times$ 24 states), a fresh planner instance is built for each of three branches --- reference tracks, tracks the cheap detector would report, and tracks the full detector would report --- since PDM-Closed carries internal state. Each proposed trajectory is scored against the logged tracks over 0.5--4.0\,s by five losses: collision, clearance shortfall $c$ below 1\,m (clipped at 2\,m), mean deviation from the logged ego trajectory, the safety loss $10\cdot\mathbf{1}[\mathrm{collision}]+1.5\,c^2$, and the scalar loss, which adds $0.1$ times the mean deviation in meters to the safety loss.

\paragraph{Real detector outcomes.} We fetch the 12{,}921 front-camera images inside the benchmark's scenario windows and run the YOLOv8s engines at 320 and 640\,px; the image nearest each state is at most 41\,ms from its lidar sweep. Logged tracks are projected into the image with the camera calibration and matched to each mode's detections at threshold 0.25 by Hungarian assignment over class-compatible pairs with IoU of at least 0.3. A vehicle, pedestrian or bicycle track inside the image is kept only if that mode detects it; tracks outside the image and static classes pass through unchanged; and unmatched detections enter as zero-velocity false positives, lifted onto the road plane with a class-median size. Per state, about 18.5 tracks fall inside the image; the cheap mode keeps 5.4 and adds 0.83 false positives, and the full mode keeps 7.7 and adds 1.68. Dropping the unmatched detections, so that each mode only removes logged tracks, leaves PDM-Closed's safety loss changing on 6 states instead of 45. Recall is 0.29 at 320\,px and 0.42 at 640\,px, and 0.65 at both within 10\,m. A branch in which every in-image track is kept and no false positive is added reproduces the reference observation on every state, and the reference branch reproduces the stored reference values exactly.

\paragraph{Sensitivities.} For PDM-Closed under the safety loss, removing false positives leaves 6 affected states (harm 33\%); a full-mode threshold count-matched on the training logs (0.48) gives 35 affected states (harm 26\%, $\rho_q=0.25$); and IoU 0.5 matching gives 51 (harm 47\%, $\rho_q=0.37$). A synthetic alternative to this construction, a detection profile fitted on 46{,}469 KITTI object outcomes, predicts a recall gap of 0.34 between the two modes against the 0.12 measured here, and correspondingly denser decision value for both planners (Table~\ref{tab:external}).

\begin{table}[htbp]
\caption{All five losses for both published planners under real detector outcomes, over all 1{,}440 states, and the two aggregate losses under the KITTI detection profile. The losses are geometric and are not nuPlan's metric suite. The IDM collision row is a single affected state.}
\label{tab:external}
\centering
\small
\begin{tabular}{llrrrr}
\toprule
Planner & Loss & Affected & Harmed & $\rho_q$ & All-full / Oracle@20 \\
\midrule
\multicolumn{6}{l}{\emph{Real detector outcomes}} \\
PDM-Closed & collision & 24 & 25.0\% & 0.33 & 11.54 / 17.31\% \\
PDM-Closed & clearance shortfall & 45 & 44.4\% & 0.34 & \phantom{0}5.22 / \phantom{0}7.89\% \\
PDM-Closed & logged-trajectory deviation & 104 & 37.5\% & 0.55 & \phantom{0}1.90 / \phantom{0}4.25\% \\
PDM-Closed & safety aggregate & 45 & 44.4\% & 0.33 & \phantom{0}9.81 / 14.74\% \\
PDM-Closed & scalar aggregate & 104 & 37.5\% & 0.35 & \phantom{0}8.26 / 12.68\% \\
IDM & collision & 1 & 0.0\% & 0.00 & \phantom{0}1.61 / \phantom{0}1.61\% \\
IDM & clearance shortfall & 5 & 40.0\% & 0.31 & \phantom{0}0.79 / \phantom{0}1.13\% \\
IDM & logged-trajectory deviation & 30 & 36.7\% & 0.76 & \phantom{0}0.14 / \phantom{0}0.56\% \\
IDM & safety aggregate & 5 & 40.0\% & 0.06 & \phantom{0}1.42 / \phantom{0}1.50\% \\
IDM & scalar aggregate & 30 & 33.3\% & 0.14 & \phantom{0}1.09 / \phantom{0}1.26\% \\
\midrule
\multicolumn{6}{l}{\emph{KITTI detection profile}} \\
PDM-Closed & safety aggregate & 70 & 37.1\% & 0.30 & 14.71 / 21.12\% \\
PDM-Closed & scalar aggregate & 175 & 30.3\% & 0.30 & 12.58 / 17.90\% \\
IDM & safety aggregate & 47 & 27.7\% & 0.22 & 16.82 / 21.58\% \\
IDM & scalar aggregate & 105 & 33.3\% & 0.23 & 14.00 / 18.08\% \\
\bottomrule
\end{tabular}
\end{table}

The cells are not independent: the safety aggregate contains collision and clearance, and the scalar aggregate contains safety. We read them as two planners under a family of related losses rather than ten replications. Under the detection profile, where both planners respond on enough states, their rankings on commonly affected states have Spearman correlation $0.62$--$0.81$ and Goodman--Kruskal $\gamma$ of $+0.49$ to $+0.69$; under real detector outcomes IDM responds on too few states for this comparison.

\section{Tie-Breaking and Lift Conventions}
\label{app:ties}

\paragraph{Ties in top-$k$ selection.}
Selecting the top $k$ of a discrete-valued signal leaves large tie groups, and how they are broken can dominate the result. At a 20\% budget over 3{,}376 nuScenes frames, only 326 frames lie strictly above the cut value of the false-negative gain while 366 share it, so 52\% of the selected set is decided by the tie-break rather than by the signal. The effect is larger in the score-free comparison, where the braking controller has zero decision value on most frames: under monocular geometry on the 2{,}655 frames with a logged future, 74\% of its top-20\% set and 83\% of its top-30\% set are tied at zero. Breaking ties by row order makes two independent systems select the same early frames and manufactures agreement between them. The cross-target analyses of Appendix~\ref{app:planning} therefore break ties with a random key, using a different seed per system, averaged over eight draws, and the benchmark evaluates each selection in exact expectation over random tie-breaks, which agrees with 20{,}000 sampled orders to within Monte Carlo error. Continuous signals --- the risk-weighted gain, uncertainty, PKL and TIP --- have tie groups of size one and are unaffected. Agreement between independent implementations does not test a tie-breaking convention they share: two implementations of the metric agree to three decimal places whenever both use a stable sort, so the convention is specified here rather than inferred from that agreement.

\paragraph{A tie-invariant version of the transfer test.}
Random tie-breaking removes the row-order artifact but does not show that \emph{no} optimal allocation for one target serves the other. Let $\Omega_{q_1}(k)$ be the set of all allocations of size at most $k$ that attain the optimum of Eq.~(\ref{eq:oracle}) for $q_1$, and define the most favorable member for $q_2$,
\begin{equation}
\mathcal{A}^{*}_{q_1\to q_2}(k)=\arg\max_{\mathcal{A}\in\Omega_{q_1}(k)}\ \sum_{i\in\mathcal{A}}V_i^{q_2},
\end{equation}
which gives the transfer every advantage. The two directions behave differently. The planner's optimum is nearly determined, and even its most favorable completion reaches $\ndg_{q_{\mathrm{brake}}}=0.067$ at a 20\% budget under oracle geometry and $0.227$ under monocular. The braking controller's optimum is heavily under-determined --- 111 responsive frames among the 2{,}655 with a logged future against a 531-frame budget under oracle geometry --- and its most favorable completion reaches $\ndg_{q_{\mathrm{plan}}}=0.930$ and $0.896$. We report both, and rest the cross-system comparison of Section~\ref{sec:sign} on per-frame sign disagreement, on the pairwise statistic below and on the binding direction, rather than on a symmetric statement.

\paragraph{Robustness of the pairwise result.}
The share of strictly ranked frame pairs that the two systems order differently is stable across geometries: $0.510$, with Goodman--Kruskal $\gamma=-0.021$, under oracle geometry and $0.502$, with $\gamma=-0.003$, under monocular geometry.

\paragraph{Box centers in the monocular lift.} The lift reports the distance to the box center, as the nuScenes submission format and the published planning-aware metrics expect. The convention matters for those metrics, which consume the 3D submission: scoring instead from the near face of each box would place every object about half an object length closer, 2.3\,m for a car and 5.6\,m for a bus.

\section{Harm Mechanism}
\label{app:mech}

Table~\ref{tab:mech} gives the per-frame detection statistics behind the mechanism of Section~\ref{sec:sign}. Escalation recovers missed objects and adds false positives at once. Under $q_{\mathrm{brake}}$, harmed frames gain $+1.980$ false positives against $+1.176$ on helped frames while recovering only $0.600$ misses against $1.396$; on 56.0\% of harmed frames escalation recovers no miss at all, introducing only spurious detections for a collision-averse controller to react to. The same ordering holds for the learned planner. Detection-level summaries such as mAP aggregate these effects over the dataset, and per-input planning-aware scores such as PKL and TIP value them through their own planner objectives; neither measures whether the exchange improves the target downstream decision on that frame.

\begin{table}[htbp]
\caption{Mean change in detection errors from cheap to full mode ($\Delta = $ full $-$ cheap), by the sign of decision value. Harmed frames add more false positives and recover fewer misses.}
\label{tab:mech}
\centering
\small
\begin{tabular}{lrrrrrr}
\toprule
& \multicolumn{3}{c}{nuScenes, $q_{\mathrm{brake}}$} & \multicolumn{3}{c}{nuScenes, $q_{\mathrm{plan}}^{\,\mathrm{self}}$} \\
\cmidrule(lr){2-4}\cmidrule(lr){5-7}
& helped & harmed & unaffected & helped & harmed & unaffected \\
\midrule
$\Delta$ false positives & $+1.176$ & $+1.980$ & $+0.849$ & $+1.093$ & $+1.251$ & $+0.671$ \\
$\Delta$ false negatives & $-1.396$ & $-0.600$ & $-0.790$ & $-1.321$ & $-1.014$ & $-0.522$ \\
$\Delta$ detections & $+2.505$ & $+2.480$ & $+1.581$ & $+2.315$ & $+2.190$ & $+1.154$ \\
\midrule
frames with an added FP & 61.5\% & 78.0\% & --- & 60.0\% & 64.2\% & --- \\
frames recovering no miss & 25.3\% & 56.0\% & --- & 31.0\% & 44.6\% & --- \\
\bottomrule
\end{tabular}
\end{table}

Figure~\ref{fig:bev} shows where these changes fall relative to the controller's corridor, and Figure~\ref{fig:gallery} shows examples of both mechanisms in camera images.

\begin{figure}[htbp]
\centering
\includegraphics{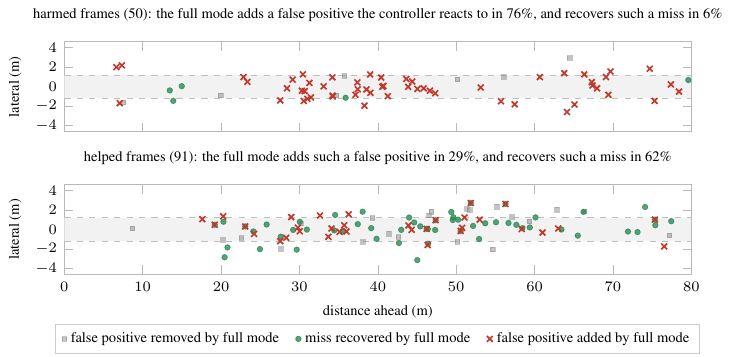}
\caption{\textbf{Where escalation changes what the braking controller sees.} Bird's-eye view of every object the controller considers --- inside its corridor (shaded) and within 80\,m --- that differs between the cheap and full modes, on the nuScenes frames with non-zero decision value under oracle geometry whose corridor objects differ between the modes. Harmed frames are dominated by false positives that the full mode adds; helped frames by misses that it recovers.}
\label{fig:bev}
\end{figure}

Table~\ref{tab:signagree} tests directly whether a perception-level reward, of the kind used by selective offloading, determines the sign of decision value. It does not. Spearman correlations between perception gain and decision value on affected frames are at most $0.11$, among affected frames with positive perception gain, 16--24\% are harmed under braking and about half under the planner, and for the planner the sign disagreement is close to 50\% for all four perception gains. On the 79 frames with non-zero decision value under both systems, 40 have opposite signs, so the same perception transition helps one downstream system and harms the other. Figure~\ref{fig:sign} shows the per-frame joint distribution for one perception gain.

\begin{figure}[htbp]
\centering
\includegraphics{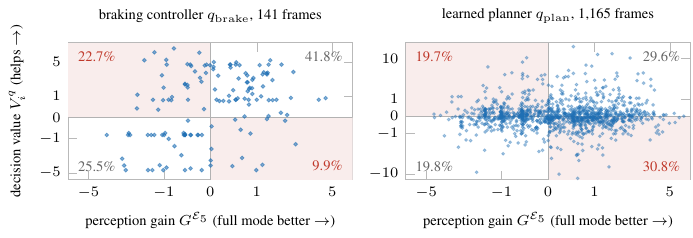}
\caption{\textbf{Perception gain does not determine decision value.} Each point is a nuScenes frame with non-zero perception gain and decision value (YOLOv8s 320$\to$640, oracle geometry), on $\operatorname{asinh}$ axes. Shaded quadrants mark sign disagreement: 32.6\% of frames for the braking controller and 50.6\% for the learned planner.}
\label{fig:sign}
\end{figure}

\begin{table}[htbp]
\caption{Perception-level gains do not fix the sign of decision value. nuScenes, YOLOv8s 320$\to$640, oracle geometry, shared threshold, frames with non-zero decision value. Corr.: Spearman correlation between gain and decision value. $P(\mathrm{harm}{\mid}G{>}0)$: share of frames with positive perception gain that escalation harms. Disagree: share of frames with non-zero gain and non-zero decision value whose signs differ. Per-mode thresholds are in Appendix~\ref{app:calib}.}
\label{tab:signagree}
\centering
\small
\setlength{\tabcolsep}{4pt}
\begin{tabular}{lrrrrrr}
\toprule
& \multicolumn{3}{c}{$q_{\mathrm{brake}}$ (141 frames)} & \multicolumn{3}{c}{$q_{\mathrm{plan}}$ (1{,}179 frames)} \\
\cmidrule(lr){2-4}\cmidrule(lr){5-7}
Perception gain & Corr. & $P(\mathrm{harm}{\mid}G{>}0)$ & Disagree & Corr. & $P(\mathrm{harm}{\mid}G{>}0)$ & Disagree \\
\midrule
Exact (false-negative count) & $+0.08$ & 24.4\% & 25.0\% & $+0.02$ & 50.6\% & 50.4\% \\
False negatives + false positives & $+0.07$ & 15.7\% & 31.9\% & $+0.02$ & 50.3\% & 49.8\% \\
Combined error ($\mathcal{E}_5$) & $+0.08$ & 19.2\% & 32.6\% & $+0.01$ & 51.0\% & 50.6\% \\
Risk-weighted ($\mathcal{E}_{\mathrm{risk}}$) & $+0.11$ & 23.8\% & 26.8\% & $+0.01$ & 50.5\% & 51.0\% \\
\bottomrule
\end{tabular}
\end{table}

\begin{figure}[!t]
\centering
\setlength{\tabcolsep}{1.5pt}
\renewcommand{\arraystretch}{0.6}
\begin{tabular}{@{}cc@{\hspace{7pt}}cc@{}}
\scriptsize harmed, cheap 320\,px & \scriptsize harmed, full 640\,px & \scriptsize helped, cheap 320\,px & \scriptsize helped, full 640\,px \\
\includegraphics[width=0.236\linewidth]{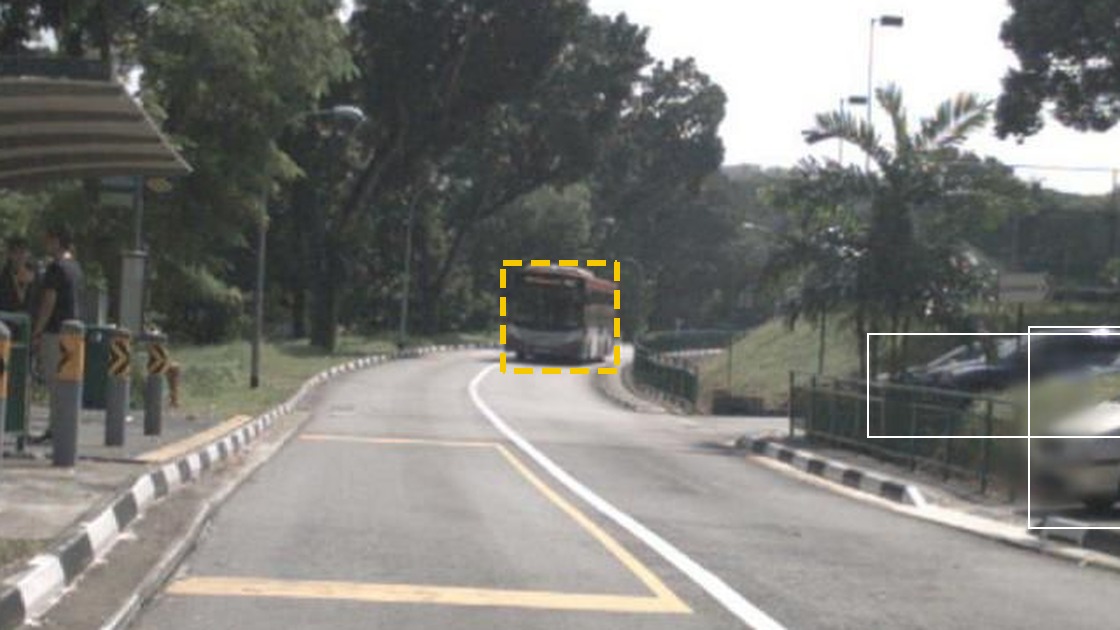} &
\includegraphics[width=0.236\linewidth]{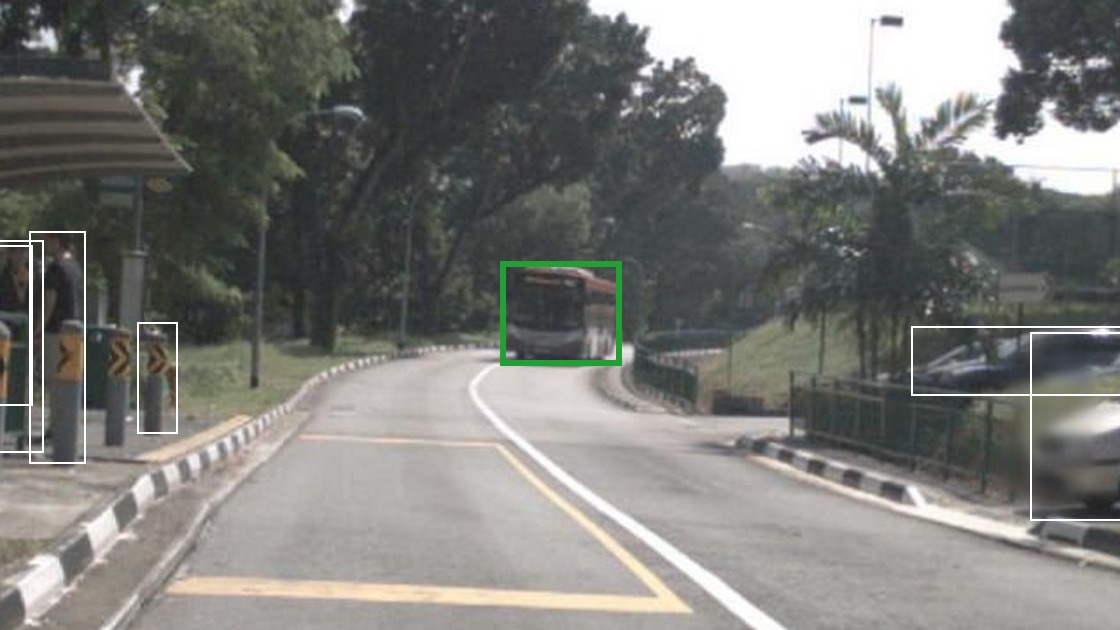} &
\includegraphics[width=0.236\linewidth]{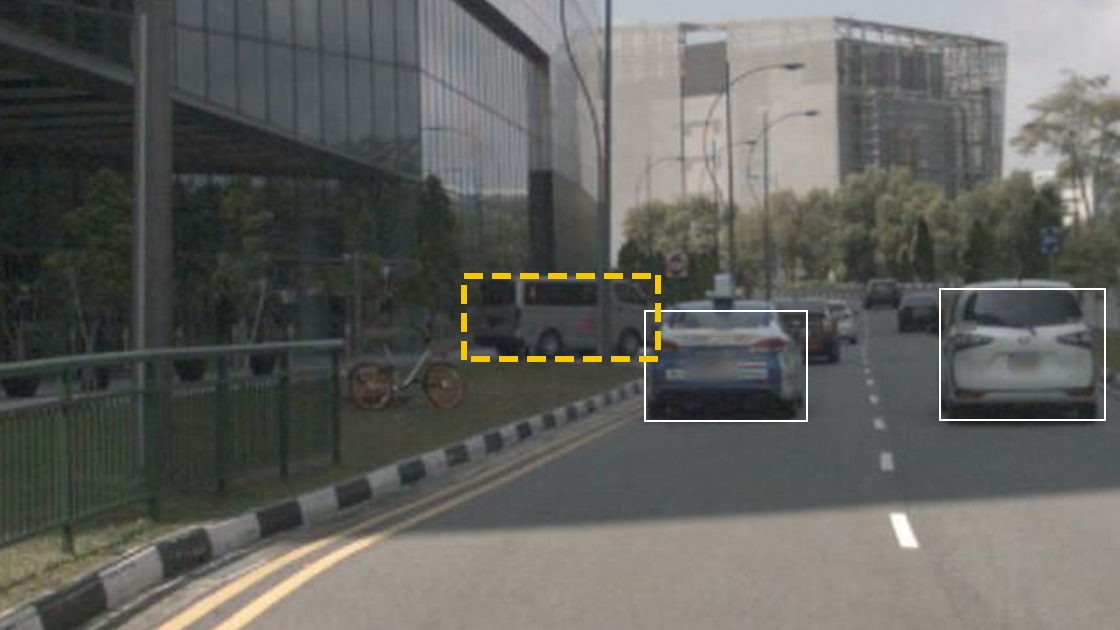} &
\includegraphics[width=0.236\linewidth]{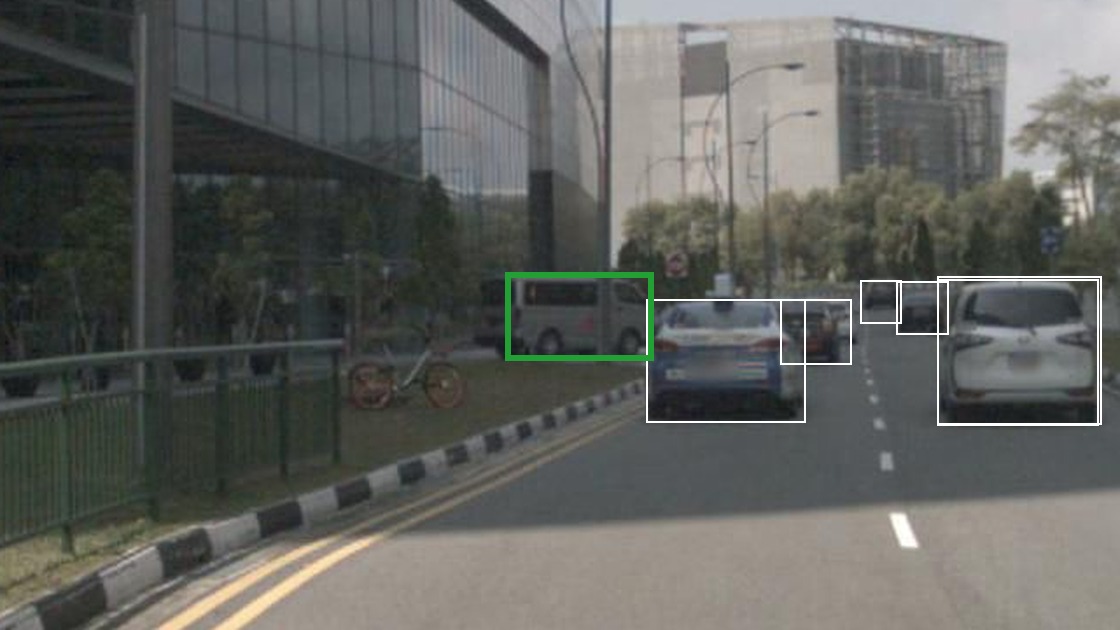} \\
\multicolumn{2}{c}{\scriptsize\color{cHarm} $V_i^q=-4.44$: keep $\to$ hard brake} &
\multicolumn{2}{c}{\scriptsize\color{cHelp} $V_i^q=+6.19$: keep $\to$ hard brake} \\[3pt]
\includegraphics[width=0.236\linewidth]{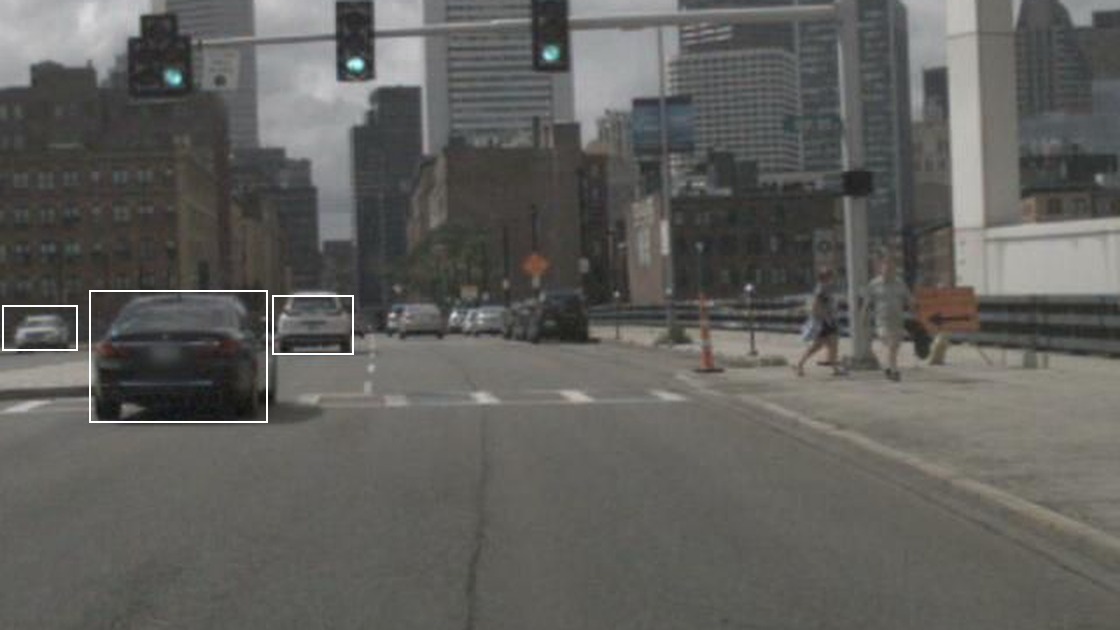} &
\includegraphics[width=0.236\linewidth]{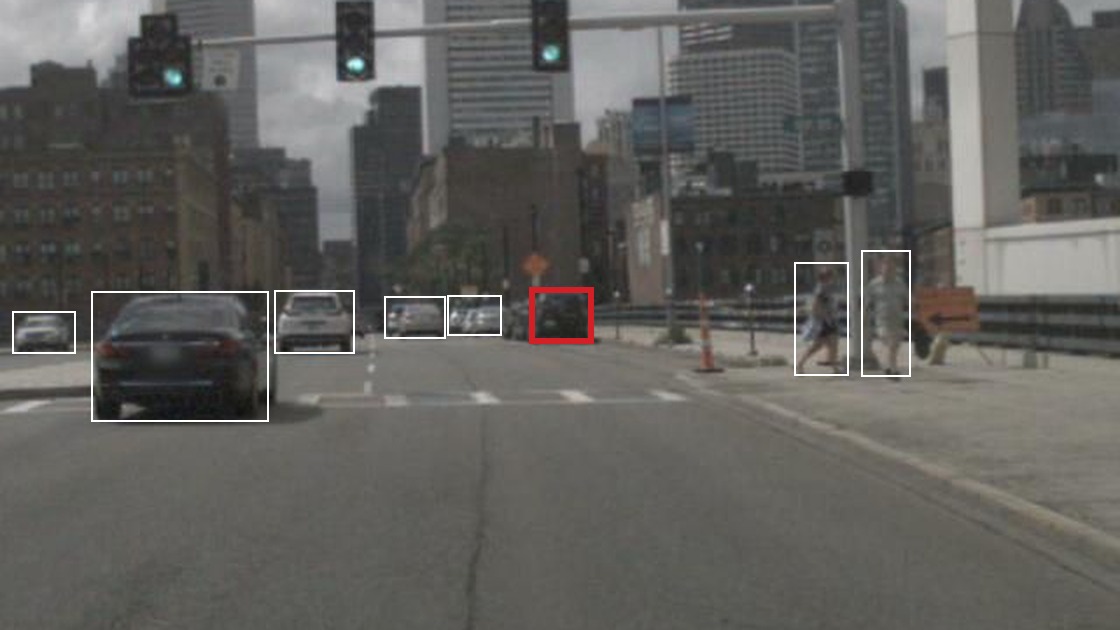} &
\includegraphics[width=0.236\linewidth]{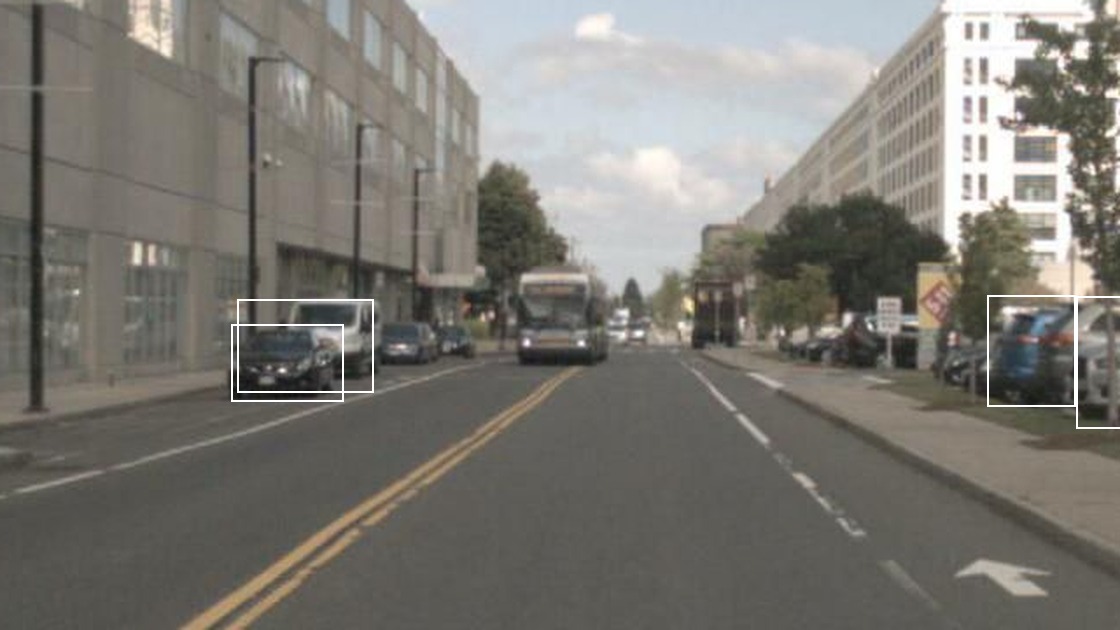} &
\includegraphics[width=0.236\linewidth]{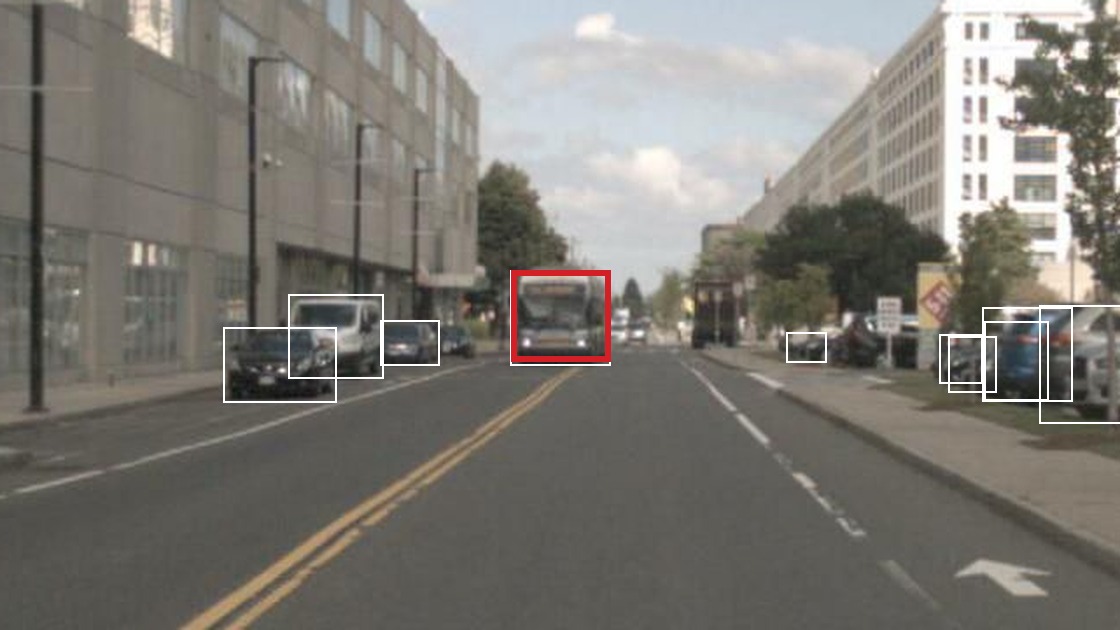} \\
\multicolumn{2}{c}{\scriptsize\color{cHarm} $V_i^q=-4.44$: keep $\to$ hard brake} &
\multicolumn{2}{c}{\scriptsize\color{cHelp} $V_i^q=+5.35$: keep $\to$ decelerate} \\[3pt]
\includegraphics[width=0.236\linewidth]{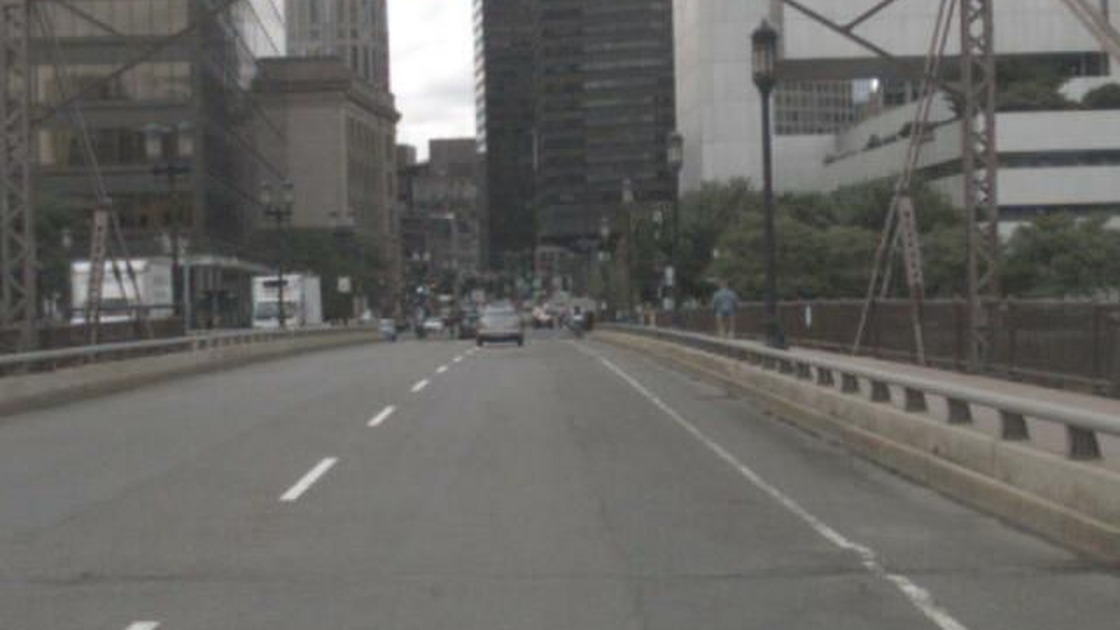} &
\includegraphics[width=0.236\linewidth]{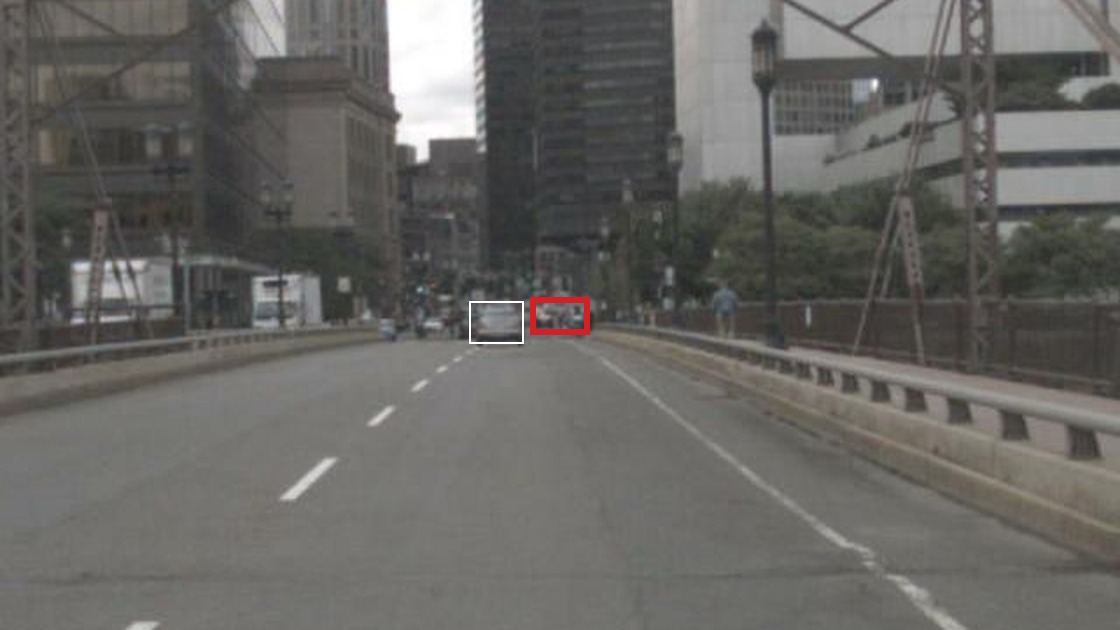} &
\includegraphics[width=0.236\linewidth]{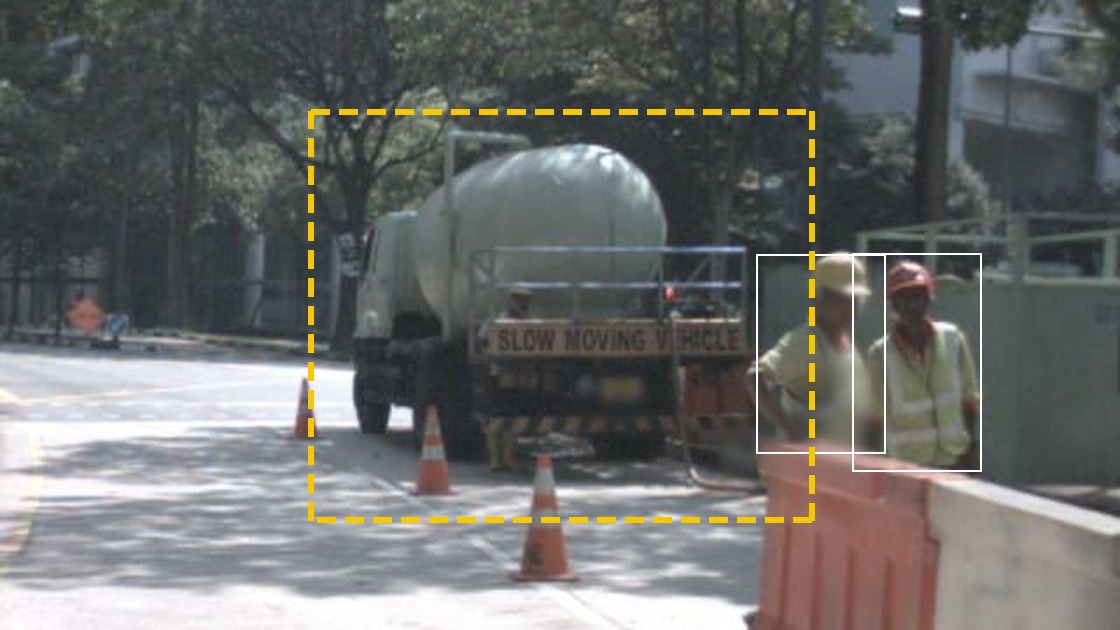} &
\includegraphics[width=0.236\linewidth]{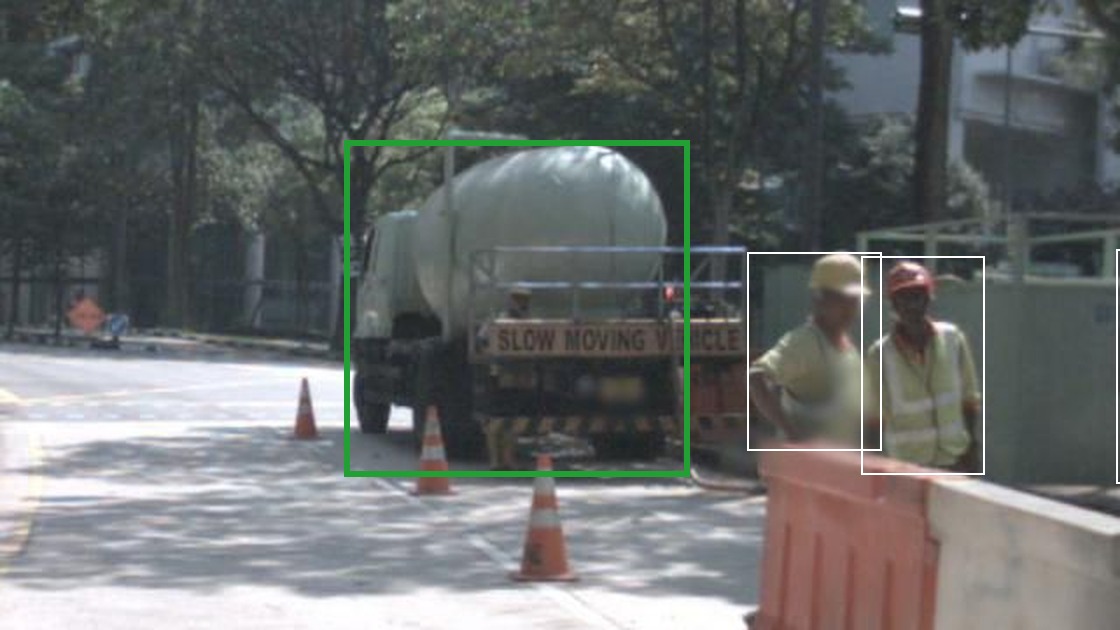} \\
\multicolumn{2}{c}{\scriptsize\color{cHarm} $V_i^q=-4.44$: keep $\to$ hard brake} &
\multicolumn{2}{c}{\scriptsize\color{cHelp} $V_i^q=+4.95$: keep $\to$ decelerate} \\[3pt]
\includegraphics[width=0.236\linewidth]{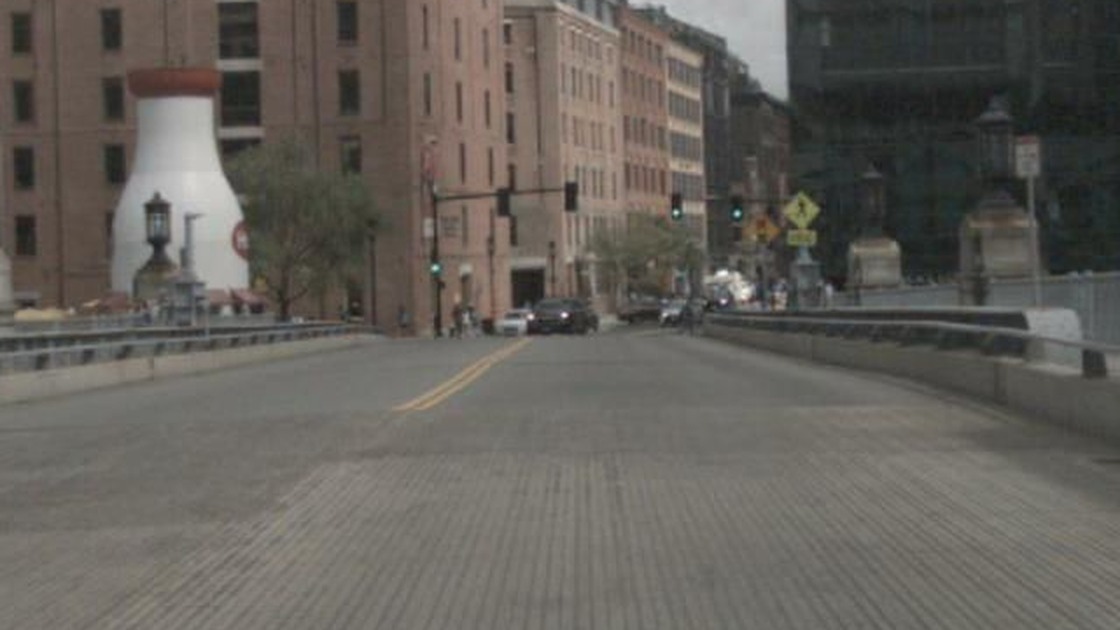} &
\includegraphics[width=0.236\linewidth]{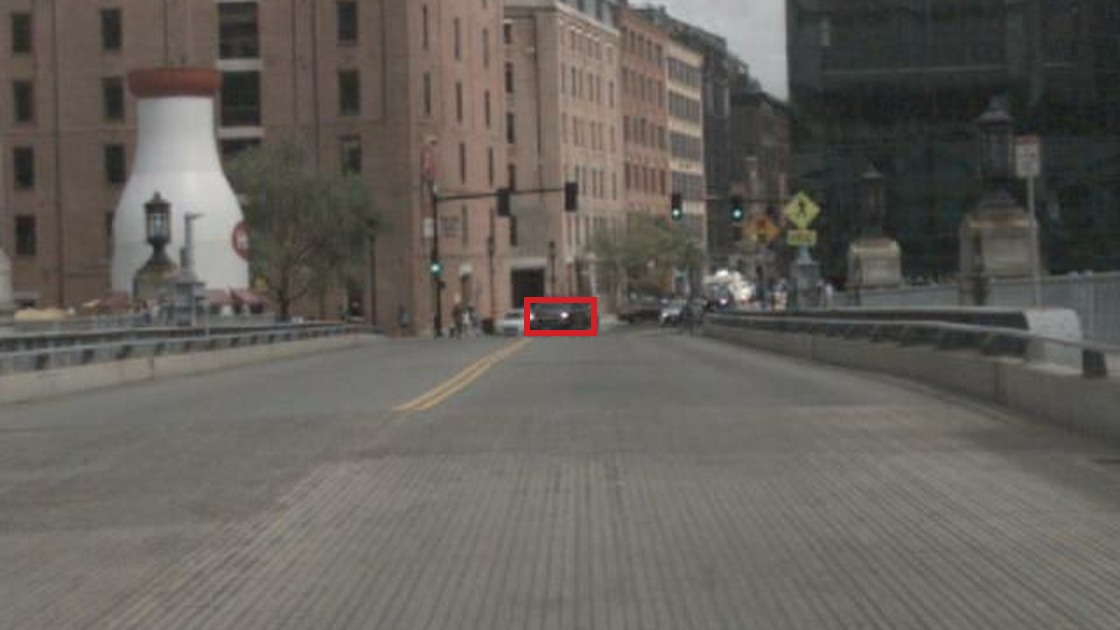} &
\includegraphics[width=0.236\linewidth]{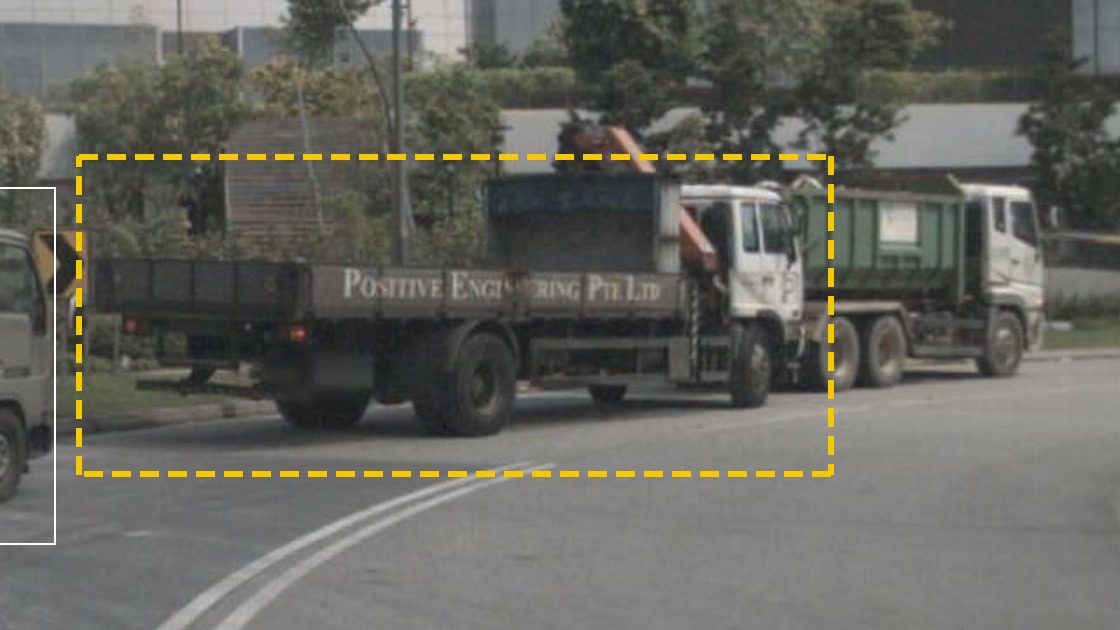} &
\includegraphics[width=0.236\linewidth]{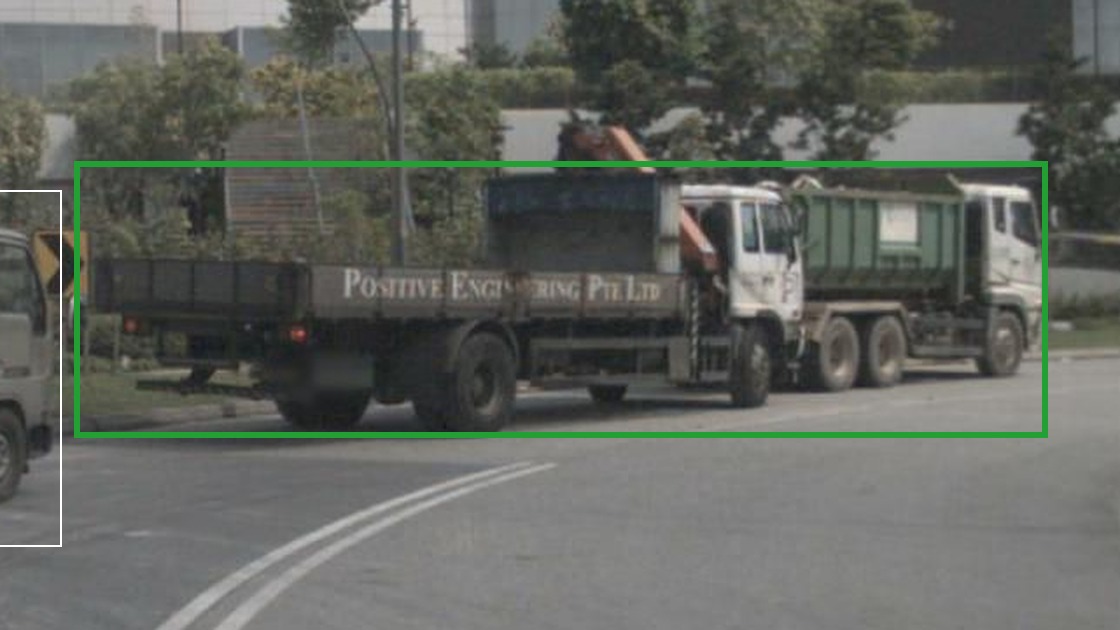} \\
\multicolumn{2}{c}{\scriptsize\color{cHarm} $V_i^q=-4.44$: keep $\to$ hard brake} &
\multicolumn{2}{c}{\scriptsize\color{cHelp} $V_i^q=+4.52$: keep $\to$ decelerate} \\
\end{tabular}
\caption{\textbf{Qualitative examples under the braking controller} (nuScenes, monocular geometry){; each pair shows the cheap and the full output, cropped around the object that decides the braking action. Left pairs are harmed frames and right pairs helped frames. Red: a detection that the full mode adds with no reference object. Green: an object that the full mode detects and the cheap mode misses (dashed yellow in the cheap image). Thin white boxes are the other detections.} The examples are the frames among the six most harmed and six most helped, at most one per scene, whose decisive change is an added detection or a recovered object; in the other four, escalation changes the estimated range of an object that both modes detect. The first harmed frame is harmed by a recovered object: the full mode detects a vehicle that the cheap mode misses, and the monocular lift places it inside the controller's corridor although its reference box lies outside it. In the second helped frame, the full mode's box around the bus overlaps the reference box by less than the matching threshold and is therefore drawn as an added detection.}
\label{fig:gallery}
\end{figure}

\section{Robustness Across Perception and Downstream Settings}
\label{app:robust}

Ego speed, a causal ego-motion signal with no perception content, separates from random on no held-out nuScenes cell at a 20\% budget (Table~\ref{tab:heldout}), so ego speed alone does not explain the perception-level results on nuScenes. The signal is causal on every track: a backward difference over the previous 0.5\,s on nuScenes, where a scene's first frame has no predecessor and takes zero, the OXTS forward velocity of the current frame on KITTI, and the current ego state's velocity on nuPlan. Ego speed is far stronger in a KITTI longitudinal setting, where braking demand scales as $v^2/(2g)$, so both the loss and the room to reduce it grow with speed; trivial baselines must therefore be reported per setting rather than dismissed. Across the 16 configurations below, false-negative gain never exceeds $\ndg=0.337$ and is negative in six.

\begin{table}[htbp]
\caption{nDG of the false-negative gain (the reduction in missed objects) and of a multi-metric diagnostic at a 20\% budget, across 16 configurations. The diagnostic is an out-of-fold gradient-boosted model over perception-error features computed from the reference state, so it is not an upper bound over perception signals: in four configurations it falls below the false-negative gain beside it. It exceeds $0.5$ only in the two aggressive-gap KITTI longitudinal cells and in nuScenes longitudinal under oracle geometry.}
\label{tab:robust}
\centering
\scriptsize
\begin{tabular}{llrr}
\toprule
Configuration & Task & nDG of $G^{\mathcal{E}_{\mathrm{FN}}}$ & Multi-metric diagnostic \\
\midrule
KITTI/Y8 320$\to$640 mono & longitudinal & 0.162 & 0.672 \\
KITTI/Y8 320$\to$640 mono & lateral & $-0.230$ & $-0.118$ \\
KITTI/Y8 320$\to$640 oracle & longitudinal & 0.200 & 0.810 \\
KITTI/Y8 320$\to$640 oracle & lateral & $-1.041$ & $-0.231$ \\
KITTI/Y8 384$\to$640 mono & longitudinal & 0.152 & 0.388 \\
KITTI/Y8 384$\to$640 mono & lateral & $-0.297$ & $-0.273$ \\
KITTI/Y8 512$\to$640 mono & longitudinal & 0.171 & 0.213 \\
KITTI/Y8 512$\to$640 mono & lateral & $-0.283$ & $-0.764$ \\
KITTI/RT 320$\to$640 mono & longitudinal & 0.074 & 0.153 \\
KITTI/RT 320$\to$640 mono & lateral & $-0.448$ & $-0.206$ \\
KITTI/RT 480$\to$640 mono & longitudinal & 0.104 & $-0.005$ \\
KITTI/RT 480$\to$640 mono & lateral & $-0.091$ & 0.138 \\
nuScenes/Y8 320$\to$640 mono & longitudinal & 0.053 & 0.350 \\
nuScenes/Y8 320$\to$640 mono & lateral & 0.337 & 0.245 \\
nuScenes/Y8 320$\to$640 oracle & longitudinal & 0.171 & 0.556 \\
nuScenes/Y8 320$\to$640 oracle & lateral & 0.266 & $-0.261$ \\
\bottomrule
\end{tabular}
\end{table}

\section{Downstream Systems and Hardware Details}
\label{app:hw}

\paragraph{Braking controller.} $q_{\mathrm{brake}}$ computes the deceleration $a^{\mathrm{req}}$ needed to stop 2\,m behind the most constraining obstacle that overlaps a corridor of half-width 1.2\,m within 80\,m, allowing 0.4\,s of reaction time and taking the larger of the requirements at ego speed and at closing speed, capped at 9\,m/s$^2$. It keeps speed for $a^{\mathrm{req}}<1.0$, decelerates at 2.5\,m/s$^2$ for $1.0\le a^{\mathrm{req}}<3.5$ and brakes at 6.0\,m/s$^2$ above. With commanded deceleration $a$, previous command $a_{\mathrm{prev}}$ and the requirement $a^{\star}$ computed from the reference scene, the loss is
\begin{equation}
\mathcal{L}_{\mathrm{brake}}=(a^{\star}-a)_+^2+0.12\,(a-a^{\star})_+^2+0.02\,|a-a_{\mathrm{prev}}|+6\cdot\mathbf{1}\bigl[(a^{\star}-a)_+\ge 8\bigr],
\end{equation}
whose terms are the {shortfall against the requirement}, over-braking, a change of command, and {a penalty for a shortfall of at least 8\,m/s$^2$. The requirement is a stop-behind rule: an object in the corridor counts as an obstacle to stop behind even if it moves, and its closing speed can only raise the requirement. Appendix~\ref{app:fu-lead} tests this choice}.

\paragraph{Receding-horizon controller.} $q_{\mathrm{traj}}$ evaluates seven longitudinal accelerations $a$ from 0 to $-8$\,m/s$^2$ crossed with five lateral offsets $d$ from $-3$ to $3$\,m over a horizon $T=3$\,s with a 0.25\,s rollout step, and selects the candidate of least cost under perceived geometry, with obstacles treated as stationary. The selected candidate is re-simulated against the reference obstacles, and its loss is
\begin{equation}
\mathcal{L}_{\mathrm{traj}}=10\cdot\mathbf{1}[\mathrm{collision}]+1.5\,c^2+\Delta_s^2+0.01\,a^2T+0.5\,d^2+0.05\bigl((a-a_{\mathrm{prev}})^2/4+(d-d_{\mathrm{prev}})^2\bigr),
\end{equation}
where $c$ is the shortfall of the worst lateral clearance below 1\,m, clipped at 2\,m, and $\Delta_s$ the progress shortfall relative to traveling at $\max(v,2\,\mathrm{m/s})$ over the horizon, as a fraction of that distance.

\paragraph{Previous action.} Both controllers depend on the previous action through $a_{\mathrm{prev}}$ and $d_{\mathrm{prev}}$, which the receding-horizon controller also uses to plan. For decision values, both branches of input $i$ take the action of all-cheap operation at the previous frame of the sequence, and none at its first frame, so $V_i^q$ is the value of escalating input $i$ alone and does not depend on which other inputs are escalated. Under this convention the change terms carry at most 2.9\% of the total harm and 0.8\% of the total benefit in every braking and trajectory setting; harm to the braking controller is 73--96\% over-braking, and harm to the receding-horizon controller is almost entirely collision and clearance.

\paragraph{Loss and corridor sensitivity.} Table~\ref{tab:losssens} varies one parameter at a time: the braking controller's over-braking weight (0.06, 0.12, 0.24), {large-shortfall} penalty (3, 6, 12) and change-of-command weight (0, 0.02, 0.04), whose losses recombine exactly because its actions do not depend on them; its corridor half-width (1.0, 1.2, 1.5\,m), which changes its actions; and the receding-horizon controller's collision (5, 10, 20), clearance (0.75, 1.5, 3) and progress (0.5, 1, 2) weights, changed in planning and loss together. The variations were chosen after the main results. The harmed share stays at 33--55\% in every monocular-geometry setting, while $\rho_q$ moves more and rises with the over-braking weight, for instance from 0.28 to 0.62 on nuScenes braking under oracle geometry.

\begin{table}[htbp]
\caption{Harmed share (\%) and harm ratio $\rho_q$ across loss and corridor variants, as the range over variants that include the default, over all units. Braking weights: over-braking, {large-shortfall} and change-of-command weights; corridor: half-width 1.0--1.5\,m; trajectory weights: collision, clearance and progress weights in planning and loss.}
\label{tab:losssens}
\centering
\small
\setlength{\tabcolsep}{3.5pt}
\begin{tabular}{lrrrrrr}
\toprule
& \multicolumn{2}{c}{Braking weights} & \multicolumn{2}{c}{Braking corridor} & \multicolumn{2}{c}{Trajectory weights} \\
\cmidrule(lr){2-3}\cmidrule(lr){4-5}\cmidrule(lr){6-7}
Setting & Harmed & $\rho_q$ & Harmed & $\rho_q$ & Harmed & $\rho_q$ \\
\midrule
nuScenes, oracle & 35.5--36.2 & 0.28--0.62 & 35.5--36.5 & 0.33--0.46 & \multicolumn{2}{c}{---} \\
nuScenes, mono & 38.1--42.5 & 0.28--0.55 & 39.2--41.6 & 0.37--0.41 & \multicolumn{2}{c}{---} \\
KITTI Y8 320$\to$640, mono & 32.6--35.3 & 0.11--0.28 & 32.7--34.6 & 0.15--0.22 & 46.0--48.8 & 0.27--0.32 \\
KITTI Y8 384$\to$640, mono & 38.8--40.1 & 0.35--0.65 & 38.9--39.8 & 0.41--0.51 & 47.3--51.0 & 0.45--0.51 \\
KITTI Y8 512$\to$640, mono & 42.9--43.4 & 0.51--0.72 & 42.6--42.9 & 0.59--0.71 & 51.6--54.5 & 1.10--1.15 \\
KITTI RT 320$\to$640, mono & 43.4--44.3 & 0.61--0.79 & 44.0--44.8 & 0.67--0.81 & 47.1--49.7 & 0.74--0.79 \\
KITTI RT 480$\to$640, mono & 40.7--43.2 & 0.68--0.79 & 42.2--43.1 & 0.72--0.86 & 42.9--45.2 & 0.57--0.60 \\
KITTI Y8 320$\to$640, oracle & 23.2--25.0 & 0.04--0.15 & 22.2--26.3 & 0.07--0.10 & 2.7--3.7 & 0.00 \\
\bottomrule
\end{tabular}
\end{table}

\paragraph{Learned planner.} We use the released planner on which PKL is built, at its published weights. It emits one BEV occupancy heatmap per future timestep over 16 steps, and we read its decision as the argmax path. $q_{\mathrm{plan}}$ is the mean Euclidean displacement, in meters, between that path and the logged future ego trajectory from the dataset's ego poses. We do not claim the logged trajectory is uniquely optimal. Frames whose four-second horizon extends past the end of a scene have no logged future and are excluded, leaving 2{,}655 of 3{,}376. $q_{\mathrm{plan}}^{\,\mathrm{self}}$ replaces the logged trajectory with the planner's own reference-box-conditioned path. The learned planner trails a constant-velocity predictor on $q_{\mathrm{plan}}$ and is included for its link to PKL; the other downstream systems clear that viability check. It was trained on nuScenes' official training split, which contains 62 of the benchmark's 85 scenes and 17 of its 24 test scenes. Its decision values are those of this fixed planner, so allocators are fitted and scored as for any other system, but its outputs on these scenes are not out of sample for the planner.

\begin{table}[htbp]
\caption{Jetson AGX Xavier FP16 TensorRT profiling of the KITTI detector modes on an idle board, one image per call. End-to-end time includes image decoding and preprocessing. Energy is per frame under the module convention, the GPU, CPU and SoC rails over idle.}
\centering
\small
\begin{tabular}{lrrr}
\toprule
Mode & GPU ms & End-to-end ms & Energy mJ \\
\midrule
YOLOv8s 320 & 5.09 & 13.18 & 41.1 \\
YOLOv8s 384 & 6.09 & 14.38 & 47.8 \\
YOLOv8s 512 & 7.89 & 16.72 & 62.0 \\
YOLOv8s 640 & 9.28 & 18.47 & 84.4 \\
RT-DETR-l 320 & 16.82 & 25.31 & 397 \\
RT-DETR-l 480 & 21.35 & 31.21 & 535 \\
RT-DETR-l 640 & 30.97 & 43.43 & 816 \\
\bottomrule
\end{tabular}
\end{table}

\paragraph{Pre-escalation gates.} Both gates regress $V_i^q$ on 65 cheap-side features in four groups: confidence statistics, uncertainty (entropy, margin and suppression ambiguity), scene complexity (counts, areas, overlaps, image and motion statistics), and criticality from the monocular geometry of the cheap detections; on nuPlan they use 18 analogous features of the cheap tracks and the ego state. The ridge gate standardizes its inputs and selects its penalty from 13 values between $10^{-2}$ and $10^{4}$ by leave-one-out cross-validation; the gradient-boosted gate uses histogram gradient boosting with depth 3, 200 iterations, learning rate 0.08 and at least 40 samples per leaf. Both are fit on the training and validation units of each track and scored once on its test units; the streaming thresholds of Appendix~\ref{app:bench} are refit on the training units and calibrated on the validation units.

\paragraph{Monocular geometry.} The deployed lift estimates each detection's range by fusing a ground-plane estimate from the bottom of its box with a class height prior, and takes lateral extent from the box edges; the exact formulas and constants are given in the released code. Geometry is reported in the ego frame. The camera-to-ego rotation enters the ground cue, whose horizon it corrects in both the cue and its fusion weight, and the translation moves the fused point and the bottom box corners to the ego origin; the calibration places the camera 1.08--1.14\,m ahead of the ego origin and 0.31--0.33\,m to its right on KITTI, and 1.70--1.72\,m ahead on nuScenes. Time to collision, which is read from the change in box height, is unaffected. Applying the full rigid transform to the fused point instead gives a larger median range error in every range band and fidelity checked. The lift overestimates range by 1.2--1.3\,m at 0--15\,m and by 1.8--2.4\,m at 15--30\,m, because the bottom edge of a loose box lands on the wheel contact rather than on the near face. The two fidelities share this bias, differing by at most 0.12\,m in every band up to 30\,m and by 1.0--2.6\,m beyond. The lift is not recalibrated against this bias, since that would tune it on the reference it is scored against. Oracle geometry gives each detection the reference box of highest image IoU above 0.5, not a one-to-one match, and inherits its range, lateral extent and time to collision; unmatched detections keep their monocular geometry.

\paragraph{Measured costs.} End to end on the board, the cheap and full pipelines cost 12.71 and 19.35\,ms on nuScenes and 13.18 and 18.47\,ms on KITTI; the gate's feature vector costs 3.52\,ms and ridge inference 0.41\,ms, the MLP detection-list routers 0.65\,ms in total, and the pixel router 1.55\,ms of TensorRT inference after 3.3--8.2\,ms of resizing and upload. On nuPlan front-camera images, the 320 and 640\,px modes cost 14.4 and 23.6\,ms, timed over preprocessing, inference and postprocessing without image decoding, and 69.5 and 143.8\,mJ per frame under the module convention; the nuScenes modes cost 54.6 and 132.7\,mJ. The gate features and models run on the CPU of the board; nuPlan has no separately timed feature pass, so its allocators are charged the nuScenes feature time.

{
\section{Follow-up Analyses}
\label{app:followup}

These analyses extend the benchmark without changing it: the frozen tracks of Sections~\ref{sec:benchmark}--\ref{sec:cost} keep their values. Each analysis was specified, with its primary endpoint and reading rule, in a protocol or an addendum committed before any of its values was computed. The analyses of Sections~\ref{app:fu-influence} and~\ref{app:fu-lead} are post hoc in one respect: a read-only audit of stored tables had already shown the facts they test, so they stress known weak points and are not independent confirmation. Five more were written after reading those results and are sensitivities: allocators scored on benchmark scenes from logs outside the fit (Section~\ref{app:fu-influence}), the loss and the requirement varied together (Section~\ref{app:fu-lead}), learners under log-grouped folds and a scorer trained on the exact objective (Section~\ref{app:fu-direct}), and causal rules at one fit set (Section~\ref{app:fu-cost}). Intervals are 95\%, from 1{,}000 bootstrap draws that resample scenes, with every draw kept{; where stated, whole recording logs are resampled instead}. Cheap and full are YOLOv8s at 192$\times$320 and 384$\times$640, with the shared action history, and braking is under monocular geometry unless stated.

\paragraph{Prior contact.} Before the protocol, the 228 new nuScenes test scenes of Section~\ref{app:fu-xl} had been used in exploratory analyses with an independent implementation (PyTorch on x86) of detector inference, the monocular lift and the braking controller: escalating every frame, detector upgrades and fine-tuned detectors for braking, and frame-level routers, including re-implementations of published routing estimators, at fixed escalation rates. Net-harmful all-full 320$\to$640 escalation for braking on these scenes, and harm with fine-tuned detectors, were therefore expected rather than discovered. The same scenes were also used as an evaluation set in other exploratory work by the authors. None of the benchmark's allocators had been scored on them, and the benchmark's pipeline had not been run on them. The split is by scene, so logs can have scenes in both the fit and the test set: all 228 test scenes have their log in the 537-scene fit, and 127 of them have an adjacent scene of that log there. Section~\ref{app:fu-influence} reports what this changes.

\subsection{nuScenes on all 850 trainval scenes}
\label{app:fu-xl}

The 765 trainval scenes outside the benchmark's 85 are split by the benchmark's rule into 385 training, 152 validation and 228 test scenes (9{,}181 test frames). The frozen 85 are nuScenes' first keyframe archive: 59 of them in Singapore One-North, none at night or in rain. The new scenes are mostly in Boston (441), with 99 night and 165 rain scenes. On the 850-scene cache the detector engines reproduce the frozen detections exactly, and the pipeline reproduces every released 85-scene value (0 of 3{,}852 comparisons differ). Allocators are fitted on the 537 new training and validation scenes and scored once on the 228 test scenes. The learned planner was trained on nuScenes' official training split, which contains 192 of the 228 test scenes, so its results here do not test the planner on scenes it has not seen (Appendix~\ref{app:hw}). Table~\ref{tab:fu-xl} gives the setting level. Of the frames in all 850 scenes, 1{,}535 of 34{,}149 (4.5\%) have $V_i^q>0$ for braking, which is all the oracle uses.

\begin{table}[htbp]
\caption{{Setting level on the 850-scene track, YOLOv8s 320$\to$640, monocular geometry, by population.}}
\label{tab:fu-xl}
\centering
\scriptsize
\setlength{\tabcolsep}{3pt}

\begin{tabular}{llrrrrr}
\toprule
System & Population & Affected & Harmed [95\%] & $\rho_q$ [95\%] & All-full [95\%] & Oracle@20 \\
\midrule
$q_{\mathrm{brake}}$ & 228 test & 792 & 52.1\% [47.9, 56.5] & 1.65 [0.42, 4.31] & $-18.2$\% [$-63.6$, 23.1] & 27.9\% \\
 & 765 new & 2{,}620 & 47.7\% [45.3, 50.2] & 1.01 [0.52, 1.76] & $-0.3$\% [$-20.5$, 17.0] & 31.8\% \\
 & frozen 85 & 273 & 39.6\% [31.5, 46.9] & 0.41 [0.27, 0.58] & 16.5\% [8.7, 25.5] & 27.8\% \\
 & all 850 & 2{,}893 & 46.9\% [44.6, 49.1] & 0.96 [0.52, 1.67] & 1.1\% [$-18.3$, 17.4] & 31.4\% \\
$q_{\mathrm{plan}}$ (ADE) & 228 test & 4{,}440 & 50.7\% [48.8, 52.7] & 1.20 [0.98, 1.49] & $-1.4$\% [$-3.3$, 0.2] & 6.7\% \\
 & all 850 & 16{,}730 & 49.5\% [48.7, 50.5] & 0.96 [0.84, 1.13] & 0.3\% [$-0.9$, 1.2] & 6.9\% \\
\bottomrule
\end{tabular}
\end{table}

\subsection{Influence of few scenes and of the fit set}
\label{app:fu-influence}

\paragraph{Scenes.} The eight scenes with the largest $\sum|V_i^q|$ are all at night, and seven of them come from two logs. All-full on the 850 scenes is $+1.1\%$ $[-18.3,+17.4]$, $+9.1\%$ $[-2.4,+19.5]$ without the most negative scene and $-2.2\%$ $[-22.1,+13.3]$ without the most positive; without the five at either end it is $+15.4\%$ and $-7.9\%$, and without both sets $+7.7\%$ $[+4.6,+10.7]$. On the 228 test scenes, without both sets it is $+0.2\%$ $[-4.8,+4.8]$. The harmed share moves only within 46.3--47.3\% and 50.5--53.2\%.

\paragraph{Allocators.} Table~\ref{tab:fu-target} gives the gradient-boosted gate trained on decision value and on four perception-level targets, with nothing else changed. Random's realized gain is negative on these scenes, so gains over random flatter every row, and we report realized gain against all-cheap. The gate's top scene carries 59\% of its gain and its top five 81\%. Registered reading: each $D_t$ keeps the side of zero after removing its most influential scene and in at least 18 of the 20 repeats, except the one on ORIC's reward, whose interval then reaches zero. That rule tests the side of zero, not size. The routers give no such result. Neither the MLP nor the gradient-boosted router on decision value is resolved against random at the same count ($+3.9\%$ $[-1.2,+11.1]$, $+2.3\%$ $[-0.5,+7.6]$). Of their ten $D_t$ values, nine are not resolved, and one reaches zero once its most influential scene is removed.

\begin{table}[htbp]
\caption{{The gate trained on decision value and on four perception-level targets: realized gain on the 228 test scenes at a 20\% selection budget, and $D_t$, the realized gain of the decision-trained gate minus that of the target-trained gate, under each check. Out of fold: 20 repeats of 5-fold log-grouped fitting over the 765 scenes, mean (minimum, maximum); the first row gives realized gain. Shares of the all-cheap loss, in \%.}}
\label{tab:fu-target}
\centering
\scriptsize
\setlength{\tabcolsep}{2.4pt}

\resizebox{\linewidth}{!}{\begin{tabular}{lrllllr}
\toprule
Training target & Realized & $D_t$ [95\%] & Top scene removed & Capped $V_i^q$ & Moving lead & Out of fold \\
\midrule
Decision value & $+12.7$ & --- & --- & --- & --- & $+13.7$ (5.7, 16.3) \\
False-negative gain & $+0.8$ & $+11.9$ [$+3.0$, $+28.0$] & $+6.8$ [$+2.3$, $+11.0$] & $+7.4$ [$+4.8$, $+10.3$] & $+9.8$ [$+4.3$, $+15.7$] & $+12.2$ (4.2, 14.4) \\
Risk-weighted gain & $-3.3$ & $+16.1$ [$+2.0$, $+40.2$] & $+6.0$ [$+0.6$, $+11.7$] & $+6.3$ [$+3.2$, $+9.4$] & $+10.0$ [$+6.0$, $+13.5$] & $+10.5$ (5.5, 13.3) \\
ORIC's reward & $+0.5$ & $+12.2$ [$+0.9$, $+27.6$] & $+7.1$ [$-0.7$, $+17.3$] & $+3.9$ [$+1.5$, $+6.6$] & $+6.6$ [$+3.3$, $+10.6$] & $+12.9$ (7.1, 16.9) \\
AP-swap objective & $+0.3$ & $+12.4$ [$+4.4$, $+26.9$] & $+9.1$ [$+3.7$, $+15.8$] & $+6.2$ [$+4.0$, $+9.2$] & $+10.7$ [$+6.6$, $+16.5$] & $+11.9$ (4.5, 14.5) \\
\bottomrule
\end{tabular}}
\end{table}

\paragraph{Fit set.} Table~\ref{tab:fu-fit} refits the gate with the test scenes unchanged. Three scenes of the fit set share logs with the two most influential test scenes and carry most of the difference between the 385-scene and the 537-scene fit. On the 216 test scenes outside those two logs the gain does not need them. Log grouping in the repeated out-of-fold runs keeps such neighbours out of the fit, and there the gate realizes $+13.7\%$ on average, or $+10.4\%$ (9.3, 11.6) without the two logs. These repeats re-split the same scenes, so they show stability under re-splitting and are not independent replications. The allocators do not carry over between cities: fitted on the Boston scenes and scored on the Singapore test scenes, or the reverse, no gate or router beats random at a 20\% selection budget under monocular geometry, whereas within each city one gate does: the gradient-boosted gate within Boston ($+9.3$ points $[+2.1,+18.1]$ over random) and the ridge gate within Singapore ($+19.6$ $[+0.3,+48.8]$); gates fitted on the frozen 85 scenes do not beat random on the 228 test scenes either. Fitted on KITTI braking, the ridge gate beats random on the 228 test scenes ($+9.3$ points $[+0.2,+22.9]$).

\begin{table}[htbp]
\caption{{Fit-set sensitivity of the gate on decision value: realized gain, \% of the all-cheap loss, 20\% selection budget.}}
\label{tab:fu-fit}
\centering
\scriptsize

\begin{tabular}{lrll}
\toprule
Fit set & Scenes & 228 test scenes [95\%] & 216 scenes without the two logs [95\%] \\
\midrule
All training and validation scenes & 537 & $+12.7$ [$+2.1$, $+30.0$] & $+6.2$ [$-0.1$, $+13.2$] \\
Without the three same-log scenes & 534 & $+4.4$ [$-23.9$, $+28.4$] & $+9.5$ [$+5.4$, $+15.2$] \\
Without every scene of the two logs & 503 & $+4.0$ [$-19.2$, $+24.9$] & $+8.3$ [$+4.9$, $+11.5$] \\
Training scenes only & 385 & $+3.6$ [$-21.0$, $+26.6$] & $+5.7$ [$+1.9$, $+9.4$] \\
\bottomrule
\end{tabular}
\end{table}

\paragraph{Logs outside the fit.} Because every test scene above has its log in the fit, a post hoc check reverses the direction: allocators fitted on the new scenes are scored on the benchmark's own scenes, after removing the two logs that have scenes on both sides (12 of the 85 scenes and 11 of the 765). The fit set is the other 754 new scenes, on 57 logs, and the test set the other 73 benchmark scenes, on nine logs, with 2{,}902 braking frames; no log is in both, and their decision values equal the released ones. The gate keeps the benchmark's hyperparameters, and the learners' settings are chosen by five-fold log-grouped cross-validation inside the fit set. The registered primary, the gradient-boosted gate at a 20\% selection budget, beats random at the same count by $+7.0\%$ $[+4.0,+9.9]$ of the all-cheap loss, {resampling logs} $[+4.7,+10.2]$, and realizes $+10.4\%$ $[+6.6,+13.9]$; random's realized gain is positive here, so the gain over random is the smaller of the two. Removing any one of the nine logs leaves $+5.8$ to $+8.6\%$. Under the cap of Section~\ref{app:fu-cost} the gate beats random by $+5.7\%$ $[+3.4,+8.1]$ at the same 356 escalations, and on criticality features alone at the 20\% latency budget by $+6.3\%$ $[+3.2,+8.9]$ at the same 274, or by $+4.5\%$ against random at the budget's nominal rate, while realizing $+7.9\%$ $[+4.6,+11.0]$ against all-cheap. In secondary rows, uncorrected, every fitted allocator, including the learners and the exact-objective scorer of Section~\ref{app:fu-direct}, beats random on this cell at 20\%; under oracle geometry the gate does not ($+1.1\%$ $[-2.8,+5.2]$), and on average displacement under monocular geometry no allocator does at any budget, with 62 of the 73 scenes in the planner's training split. These are the benchmark's own scenes, with which its features and baselines were developed, so the check separates logs but is not a fresh dataset, and it does not cross cities.

\subsection{Loss weights and the braking requirement}
\label{app:fu-lead}

\paragraph{Loss weights.} The braking loss recombines exactly (Appendix~\ref{app:hw}){, so every variant in this section is recomputed from the cached detections, without rerunning a detector}. Over the 850 scenes, over-braking weights from 0 to 0.48 and large-shortfall penalties from 0 to 12 (14 settings) give harmed shares of 45.1--49.4\%, $\rho_q$ of 0.89--1.03, with every interval of $\rho_q$ containing 1, all-full of $-1.4$ to $+3.0\%$ and oracle@20 of 27.2--41.4\%. On the 228 test scenes the harmed share is 48.5--53.4\%. At over-braking weight 0 the change-of-command term stays in the loss, so a frame whose other difference is over-braking is still counted, with a sign set by that term alone, and these are not shortfall-only counts: 36.8\% of the 1{,}301 harmed frames over the 850 scenes differ in that term alone. A loss with only the shortfall and large-shortfall terms gives 42.2\% $[35.5,48.5]$ on the test scenes and 41.6\% $[38.5,44.6]$ on all 850, under the standard requirement.

\paragraph{The requirement.} The standard requirement counts every object in the corridor as an obstacle to stop behind (Appendix~\ref{app:hw}). This choice matters at scale. Under it, the bands $a^\star\in[3.5,7)$ and $[7,9]$ hold 560 and 282 frames at ego speed of at least 3\,m/s, of which 76\% and 79\% have a moving constraining object. Of those with 2\,s of future, 86.9\% show a logged mean acceleration above $-0.5$\,m/s$^2$ and 5.2\% a logged deceleration of at least 1\,m/s$^2$, and none of the 54 frames with the large-shortfall term in a branch shows one of 2\,m/s$^2$. We therefore test two variants as sensitivities; the frozen values are not replaced. With gap $g$ and object speed $v_o$ along the ego heading, the moving-lead requirement (primary) is $v^2/\bigl(2(g+v_o^2/(2\cdot4))\bigr)$, in which the lead may brake at 4\,m/s$^2$, and the constant-speed requirement is $\max(0,v^2-v_o^2)/(2g)$. Both reduce to the standard rule for a stationary object. Table~\ref{tab:fu-lead} applies each to the reference and the controller. The moving-lead rule shrinks those bands to 138 and 59 frames, almost all with stationary objects. It does not make $a^\star$ a model of the logged driver: 12.0\% and 7.8\% of the remaining frames show a logged deceleration of at least 1\,m/s$^2$. The registered reading holds on the 228 and the 850 (harmed share at least 30\% with $\rho_q>0$), and every allocator claim keeps its interval above zero, at scene and at log level.

\paragraph{Loss and requirement together.} With only the shortfall terms under the moving-lead requirement, the harmed share is 42.4\% $[35.6,49.4]$ on the test scenes and 42.9\% $[39.8,45.9]$ on all 850, and the registered reading holds on both. Harm is about as frequent and much smaller: against the standard requirement under the same loss, total harm falls from 1{,}668 to 411 and from 3{,}417 to 981, and $\rho_q$ from 1.73 to 0.84 $[0.35,1.74]$ and from 0.89 to 0.47 $[0.30,0.72]$. Under the default loss, under-braking carries 60\% of the harmed mass on all 850 under the standard requirement and 23\% under the moving-lead one, while the frozen 85 barely move.

\paragraph{KITTI.} KITTI's tracks give a range rate by the same local fit; 4 of 46{,}469 objects have no estimate and count as stationary. Over the 21 sequences under monocular geometry, the moving-lead requirement moves the harmed share from 34.2\% to 37.6\% $[26.3,48.2]$ and $\rho_q$ from 0.17 to 0.48. All-full falls from $+51.7\%$ to $+21.4\%$ $[+2.0,+45.4]$ and oracle@20 from 62.1\% to 41.4\%. On the six test sequences the share moves from 44.7\% to 45.1\% $[29.0,55.1]$, a coarse interval. The registered reading holds on the point on both populations; the lower ends of the intervals lie below 30\%.

\begin{table}[htbp]
\caption{{The braking requirement, with the same rule in reference and controller. Over-braking: share of the harmed mass on frames where the full mode brakes harder.}}
\label{tab:fu-lead}
\centering
\scriptsize
\setlength{\tabcolsep}{3pt}

\resizebox{\linewidth}{!}{\begin{tabular}{llrrrrrr}
\toprule
Requirement & Population & Affected & Harmed [95\%] & $\rho_q$ [95\%] & All-full [95\%] & Oracle@20 [95\%] & Over-braking \\
\midrule
Standard & 850 & 2{,}893 & 46.9\% [44.6, 49.1] & 0.96 [0.52, 1.67] & $+1.1$ [$-18.3$, $+17.4$] & 31.4 [25.7, 37.2] & 40.0\% \\
 & 228 & 792 & 52.1\% [47.9, 56.5] & 1.65 [0.42, 4.31] & $-18.2$ [$-63.6$, $+23.1$] & 27.9 [18.3, 40.7] & 29.3\% \\
 & frozen 85 & 273 & 39.6\% [31.5, 46.9] & 0.41 [0.27, 0.58] & $+16.5$ [$+8.7$, $+25.5$] & 27.8 [18.0, 39.3] & 89.8\% \\
Moving lead & 850 & 3{,}105 & 47.0\% [44.9, 49.1] & 0.75 [0.62, 0.89] & $+7.0$ [$+2.8$, $+11.1$] & 28.2 [25.2, 31.3] & 77.0\% \\
 & 228 & 846 & 51.1\% [47.0, 55.5] & 1.07 [0.78, 1.48] & $-1.8$ [$-11.1$, $+6.3$] & 26.8 [21.1, 33.2] & 68.1\% \\
 & frozen 85 & 288 & 41.0\% [33.5, 47.9] & 0.45 [0.31, 0.63] & $+15.1$ [$+7.5$, $+24.7$] & 27.5 [17.4, 39.7] & 87.6\% \\
Constant speed & 850 & 3{,}316 & 47.6\% [45.4, 49.7] & 0.74 [0.65, 0.85] & $+7.6$ [$+4.0$, $+11.2$] & 29.9 [26.6, 33.3] & 86.1\% \\
 & 228 & 927 & 49.7\% [45.8, 53.9] & 0.89 [0.72, 1.11] & $+3.3$ [$-2.7$, $+9.7$] & 30.2 [24.3, 36.8] & 80.6\% \\
 & frozen 85 & 310 & 45.2\% [38.2, 52.2] & 0.53 [0.37, 0.72] & $+13.6$ [$+5.8$, $+23.9$] & 28.9 [17.8, 42.8] & 86.2\% \\
\bottomrule
\end{tabular}}
\end{table}

\subsection{Decision-focused baselines, and the information ladder}
\label{app:fu-direct}

\paragraph{Learners.} The cost-sensitive learner is a gradient-boosted classifier of whether $V_i^q$ exceeds $\lambda$, weighted by $|V_i^q-\lambda|$ plus a small constant. The relaxed top-$k$ learner maximizes $\sum_i\sigma((s_i-\theta)/T)\,V_i^q$, a differentiable relaxation of the budgeted gain, with $\theta$ the batch quantile of the budget. The pairwise learner minimizes a logistic loss on sampled pairs, weighted by $|V_i^q-V_j^q|$. Both MLPs use the routers' architecture, with five seeds. All settings, the early stopping and the choice among the three are made per cell and budget on the validation scenes with models fitted on the training scenes; the chosen setting is refit on the 537 scenes and scored once on the test scenes. {On braking the choice is the cost-sensitive learner at 5\% and the pairwise learner at 20\%; on average displacement at 20\% it is the relaxed top-$k$ learner.} The two primaries (braking and average displacement at 20\%, Holm over the two) are not resolved ($p=0.580$) and in favour of regression ($p=0.022$); the second, $-1.0$ $[-2.0,-0.2]$, reaches zero {when logs are resampled} ($[-2.3,+0.0]$) and in two influence rows. Table~\ref{tab:fu-direct} gives other budgets. With $V_i^q$ capped, the first stays unresolved ($+0.72$ $[-0.63,+1.86]$). Under the 20\% latency budget with the one-group inputs of Section~\ref{app:fu-cost}, the chosen pairwise learner realizes $+9.6\%$ against $+9.0\%$, a difference of $+0.6$ $[-2.0,+4.3]$. All 12 frozen cells at four budgets passed the public scorer with no change to it; the learner's point nDG exceeds the gate's in none of the 12 at 20\%, a point comparison on cells with few test units.

\paragraph{Log-grouped folds.} All test scenes have a log in the fit, so we refit the chosen settings in the 20 log-grouped repeats of Section~\ref{app:fu-influence}, without choosing them again. They were chosen on scenes inside these folds, which favours the learners. The difference to the gate over the 765 scenes is $+2.2$ points (range $-0.8$ to $+4.9$, 17 of 20 above zero) on braking at 20\%, $+1.9$ ($-2.0$ to $+4.9$, 18 of 20) on braking at 5\%, and $-0.6$ ($-1.1$ to $-0.4$, none) on average displacement at 20\%. These read as no consistent difference, the learner ahead and regression ahead; the second sits exactly on the rule's edge. Without the two logs of that section they are $-1.1$ (3 of 20), $+3.0$ (all 20) and $-0.5$ (none). These fits ran on a second machine with the same Python, scikit-learn and numpy and a different torch, which reproduces the stored shares within 0.00024 of the all-cheap loss. Under causal thresholds at the 385-scene fit the difference on braking is not resolved: $+6.5$ $[-1.8,+21.5]$ without a cap and $+3.7$ $[-1.9,+11.5]$ with one.

\paragraph{Exact objective.} The three learners train on surrogates. A post hoc addition trains a linear scorer of the 65 gate features directly on the budgeted objective, the sum of $V_i^q$ over its $k$ highest-scoring fit frames, by the cross-entropy method, a derivative-free optimizer; we claim no optimal solution, local or global, of this piecewise-constant objective. In the optimizer check fixed before any test value, its training objective was at least that of ridge regression in all 42 rows over cells, budgets and targets. Its four settings are chosen on validation as above, and its primaries compare it with ridge regression on the same inputs at 20\%, Holm over two cells: $-1.9$ points $[-14.1,+5.2]$ on braking and $-0.02$ $[-0.54,+0.31]$ on average displacement, both not resolved. On braking it also trails the gradient-boosted gate in point estimate ($-5.9$ $[-20.4,+2.7]$). Its point difference to ridge on braking is negative at every budget from 5 to 50\%, but rests on one scene: without the most influential test scene it leads by $+4.4$ $[+2.1,+6.7]$, and it leads in all 20 log-grouped repeats at 20\% (mean $+3.7$), with settings chosen on scenes inside those folds. Under oracle geometry at 20\% it realizes $+0.0\%$ against $+14.6\%$ for ridge, a difference whose interval, $[-54.7,+4.0]$, spans zero. On the log-disjoint check of Section~\ref{app:fu-influence} it beats random by $+5.8\%$ $[+0.7,+11.3]$, against $+4.7\%$ $[+0.7,+9.9]$ for ridge. A policy-gradient MLP trained on the same objective failed that check: at its validation-chosen setting and epoch, its training objective stayed below the regression MLP's in 30 of its 42 rows, also after the one change to the optimizer that the check allows, so it was not scored. Training a linear scorer on the exact objective therefore does not improve on regression with the same model in the primary comparisons; this does not show that regression is optimal for other model classes or optimizers.

\begin{table}[htbp]
\caption{{Validation-chosen decision-focused learner minus the regression gate, points of the all-cheap loss, by selection budget. Bold: the interval excludes zero. These rows are secondary and uncorrected, except the 20\% column of the first two.}}
\label{tab:fu-direct}
\centering
\scriptsize
\setlength{\tabcolsep}{2.2pt}

\resizebox{\linewidth}{!}{\begin{tabular}{lllllll}
\toprule
Cell & 5\% & 10\% & 20\% & 30\% & 40\% & 50\% \\
\midrule
Braking & $\mathbf{+4.6}$ [+2.3, +12.8] & +0.3 [$-$2.6, +4.5] & +1.4 [$-$3.0, +5.0] & $-$0.7 [$-$4.7, +2.2] & $-$0.3 [$-$4.6, +2.3] & +0.2 [$-$1.8, +3.2] \\
Avg.\ displacement & $-$0.4 [$-$1.0, +0.1] & $\mathbf{-1.1}$ [$-$1.9, $-$0.3] & $\mathbf{-1.0}$ [$-$2.0, $-$0.2] & $-$0.6 [$-$1.6, +0.1] & $\mathbf{-0.8}$ [$-$1.7, $-$0.0] & $-$0.2 [$-$1.0, +0.6] \\
Braking, moving lead & $\mathbf{+2.4}$ [+0.0, +4.5] & +0.8 [$-$2.7, +5.4] & $-$2.6 [$-$7.7, +2.8] & $-$1.9 [$-$6.7, +2.2] & $-$1.7 [$-$5.5, +2.3] & $-$1.1 [$-$5.2, +2.2] \\
\bottomrule
\end{tabular}}
\end{table}

\paragraph{Information ladder.} Table~\ref{tab:fu-ladder} adds inputs to both learners without reselecting them. The upper two rungs need the full output or the reference and cannot run before the decision; they are diagnostic. No rung contains an action or a requirement. The registered reading is mixed for both learners: the gap is ``mainly information'' only if the difference to the deployable inputs at the third rung is resolved and exceeds half of the gap to the oracle, and ``mainly architecture or noise'' only if the fourth rung stays below half of the oracle. Neither holds.

\begin{table}[htbp]
\caption{{Information ladder, 228 test scenes, 20\% budget, realized gain and difference to the deployable inputs, \% of the all-cheap loss. Oracle at this budget: $+27.9\%$.}}
\label{tab:fu-ladder}
\centering
\scriptsize

\begin{tabular}{lrrlrl}
\toprule
Inputs & Columns & \multicolumn{2}{c}{Regression gate} & \multicolumn{2}{c}{Pairwise learner} \\
\midrule
Deployable (cheap-side features) & 65 & $+12.7$ & --- & $+14.1$ & --- \\
+ three previous frames, ego dynamics & 262 & $+10.4$ & $-2.3$ [$-7.5$, $+1.5$] & $+14.2$ & $+0.1$ [$-2.1$, $+2.5$] \\
+ same features of the full output, differences & 392 & $+15.7$ & $+2.9$ [$-1.9$, $+7.2$] & $+16.8$ & $+2.7$ [$+0.7$, $+5.3$] \\
+ reference objects in the corridor & 396 & $+19.5$ & $+6.8$ [$+1.8$, $+12.3$] & $+18.5$ & $+4.4$ [$+2.1$, $+7.3$] \\
\bottomrule
\end{tabular}
\end{table}

\subsection{Gains after allocator cost}
\label{app:fu-cost}

Each of the 15 non-empty subsets of the gate's four feature groups was profiled on the device as executed, together with the cheap and full detectors in the same session (10.78 and 15.43\,ms, not comparable with the profile of Appendix~\ref{app:hw}), so that a 20\% latency budget, $C_{\mathsf c}+0.2\,C_{\mathsf f}$ per input with the allocator's own cost charged (Section~\ref{sec:baselines}), leaves 3.09\,ms for the allocator. The subset was chosen on the validation scenes, with models fitted on the training scenes, as the one with the largest realized gain at this budget with its cost charged, then refit on the 537 scenes and scored once on the test scenes. The choice for the gradient-boosted gate is criticality alone (1.63\,ms), which escalates 867 frames, or 9.4\%. Its registered primary compares it, ranked offline, with random at the same 20\% latency budget, where random, with no overhead, escalates 1{,}836 frames: $+12.7\%$ $[+3.8,+24.7]$. This comparison flatters the gate, because random's realized gain is negative here ($-3.6\%$), so random loses more by escalating more. We therefore also report the gate's realized gain against all-cheap, $+9.1\%$ $[+0.0,+22.3]$, {resampling logs} $[-0.4,+23.3]$ (Table~\ref{tab:survive}), and its gain over random at the same 867 escalations, $+10.8\%$ $[+3.8,+22.3]$. The full feature set (4.04\,ms) cannot run at this budget. Out of fold, with folds grouped by log, the realized gain is $+9.7\%$ on average (3.0, 11.9), all 20 repeats above zero. Without reselection, the same choice beats random at the same latency budget on the benchmark's nuScenes split by 4.8\% $[0.8,9.5]$, but not on KITTI braking ($+0.5\%$ $[-2.1,3.5]$).

\paragraph{Routers.} At the 20\% latency budget on the core cells, one of the 20 rows of the two MLP routers has an nDG interval above random: the classification router on KITTI's trajectory controller under oracle geometry, $+0.21$ $[+0.00,+0.61]$. On the 228 test scenes neither is resolved against random (Section~\ref{app:fu-influence}), and gains on the planner cells are negligible: the gate realizes $+0.2\%$ on average displacement.

\paragraph{Causal rules.} Table~\ref{tab:fu-causal} scores the same fitted models by ranking the test split and by two causal rules. The threshold comes from log-grouped out-of-fold scores of the 537 scenes, so no test frame enters. It keeps the ranking's gain with no resolved cost, but it can exceed the budget, on 4 of the 14 signal and fit rows, and for the gate it breaks the running cap on 97 scenes. The cap stays within budget and costs gain: $-6.8$ points $[-17.6,-0.2]$ for the gate, with 1{,}139 of 1{,}836 escalations ({resampling logs} $[-18.3,+0.6]$), and $-4.4$ $[-11.4,-0.4]$ for one feature group, with 521 of 867 escalations. Under the cap that group still gains $+5.8\%$ $[+2.3,+10.8]$ over random at the same 521 escalations. In the repeats of the last column, which hold out a fifth of each fit set's logs to set the threshold, the ranking realizes $+11.8\%$ (2.4, 16.5) for the gate and $+8.4\%$ (1.5, 12.7) for one group, below the $+13.7\%$ and $+9.7\%$ of the full-fit repeats. Registered reading under the cap: the $D_t$ on the AP-swap objective is robust. The other five claims, including both gains, are tail-dependent: their intervals reach zero once the most influential scene is removed, and all 20 repeats stay above zero. At the 385-scene fit no realized interval excludes zero: the gate gives $+3.6\%$, $+3.6\%$ and $+4.5\%$ $[-4.8,+12.9]$ under the three forms, so a difference between fit sets must not be read as a cost of the causal rule. Through the mask form over the frozen cells and four budgets, the cap stays within budget on all 336 rows and the threshold alone on 174.

\begin{table}[htbp]
\caption{{Ranking and causal rules at one fit set, 20\% budget unless stated: 228 test scenes, 537-scene fit; last column, 20 log-grouped repeats over the 765 scenes, mean (minimum, maximum). Threshold: fixed before the test stream. Cap: that threshold with at most $\lfloor 1+\alpha t\rfloor$ escalations after $t$ frames of a scene, $\alpha$ the rate the budget allows. Top scene removed: under the cap. Gains in \% of the all-cheap loss; $D_t$ as in Table~\ref{tab:fu-target}, in points.}}
\label{tab:fu-causal}
\centering
\scriptsize
\setlength{\tabcolsep}{2.4pt}

\resizebox{\linewidth}{!}{\begin{tabular}{llllll}
\toprule
Row & Ranking & Threshold & Cap & Top scene removed & Out of fold, cap \\
\midrule
Gate on $V_i^q$ & $+12.7$ [$+2.1$, $+30.0$] & $+13.4$ [$+2.6$, $+30.2$] & $+6.0$ [$+0.4$, $+12.6$] & $+3.7$ [$-0.2$, $+7.1$] & $+7.4$ (4.0, 9.6) \\
One feature group, 20\% latency & $+9.1$ [$+0.0$, $+22.3$] & $+9.1$ [$-0.0$, $+22.3$] & $+4.7$ [$-1.2$, $+11.7$] & $+2.3$ [$-2.1$, $+5.3$] & $+4.1$ (1.2, 6.4) \\
One feature group, 20\% energy & $+12.8$ [$+3.6$, $+28.6$] & $+12.8$ [$+3.7$, $+28.6$] & $+8.6$ [$+3.0$, $+18.0$] & --- & --- \\
$D_t$, false-negative gain & $+11.9$ [$+3.0$, $+28.0$] & $+13.9$ [$+3.4$, $+30.3$] & $+5.5$ [$-1.3$, $+12.4$] & $+3.2$ [$-2.2$, $+7.6$] & $+5.4$ (1.8, 8.4) \\
$D_t$, risk-weighted gain & $+16.1$ [$+2.0$, $+40.2$] & $+14.8$ [$+2.3$, $+35.5$] & $+6.1$ [$-2.9$, $+19.1$] & $+1.2$ [$-3.9$, $+5.7$] & $+5.8$ (2.1, 8.3) \\
$D_t$, ORIC's reward & $+12.2$ [$+0.9$, $+27.6$] & $+11.5$ [$+1.5$, $+26.0$] & $+4.2$ [$-1.0$, $+10.2$] & $+2.5$ [$-2.3$, $+7.1$] & $+7.2$ (2.0, 10.1) \\
$D_t$, AP-swap objective & $+12.4$ [$+4.4$, $+26.9$] & $+14.7$ [$+5.0$, $+28.2$] & $+6.8$ [$+2.3$, $+11.8$] & $+6.0$ [$+2.0$, $+10.1$] & $+5.8$ (2.6, 8.6) \\
\bottomrule
\end{tabular}}
\end{table}

\subsection{Sequences, reuse and closed loop}
\label{app:fu-seq}

\paragraph{Replay and sequence oracle.} The braking action at a frame depends only on that frame's branch, and the loss on the current and previous action, so a selection can be replayed exactly with its own history. On braking under either geometry, on the 850 and the 228 scenes and on KITTI, the replayed gain differs from the additive sum by at most 0.40 points of the all-cheap loss: the gate on the 228 scenes moves from 12.73\% to 12.88\%, and the oracle at 20\% on the 850 from 31.43\% to 31.82\%. A dynamic program over frame and previous action, with a budget over units, gives the exact optimum: it exceeds the additive top-$k$ oracle by 0.21--0.40 points, and replaying the additive selection comes within 0.01 points of it. The change-of-command term has weight 0.02, so this bounds a small term by construction. It covers the controller's action memory only, not ego-state feedback.

\paragraph{Reuse.} Here the cheap branch runs no detector: it reuses the previous frame's detections, carried to the current frame with the ego motion, and the full branch detects on the current frame (frame interval 0.5\,s on nuScenes, 0.1\,s on KITTI). Without the carry this is the zero-order hold of streaming perception \citep{li2020streaming}. In the primary cell (228 test scenes, 640, monocular), re-detecting harms 50.6\% $[47.1,54.3]$ of affected frames, with $\rho_q=0.94$ $[0.76,1.11]$ and all-fresh $+2.0\%$ $[-3.3,+8.1]$. In all 24 monocular cells the harmed share is 41.1--52.0\%. This measures one reuse step; the next paragraph lets staleness accumulate.

\paragraph{Refresh scheduling.} Here a policy decides before any current-frame perception{, and the evaluation uses only the cached detections}. At each frame it runs the 640 detector or reuses the last detected list, moved with the ego motion; a list may reach age 10, and then a refresh is forced and charged. Because a list not refreshed grows older, an action changes later states. A dynamic program over frame and age with one shared multiplier gives the exact hindsight oracle, with no gap to its dual bound. Policies are scored by replay with their true staleness (Table~\ref{tab:fu-sched}). The registered primary, the sequential learner minus periodic refresh on KITTI at 20\%, is not resolved, and six sequences give a coarse interval. On nuScenes the learner is not resolved against periodic refresh at any budget. On KITTI the gate with age leads the learner by 5.9 points $[2.9,8.9]$, uncorrected. The learner imitates the oracle's states without aggregation over its own runs, so this is a floor for that one learner and no bound on sequential policies. Periodic refresh slightly exceeds the budget (708 of 707 and 1{,}900 of 1{,}836 refreshes).

\begin{table}[htbp]
\caption{{Refresh scheduling at a 20\% budget, monocular, test units, by replay. Gains in \% of the loss of the sparsest schedule, one refresh every 11 frames. Last column: sequential learner minus periodic refresh, in points. The oracle is in hindsight.}}
\label{tab:fu-sched}
\centering
\scriptsize
\setlength{\tabcolsep}{2.4pt}

\resizebox{\linewidth}{!}{\begin{tabular}{lrlllll}
\toprule
Dataset & Oracle & Periodic & Age $\times$ speed & Gate with age & Sequential learner & Difference \\
\midrule
KITTI, 6 sequences & $+58.7$ & $+4.5$ [$-4.7$, $+14.3$] & $+1.8$ [$-1.7$, $+6.1$] & $+11.5$ [$+2.0$, $+21.6$] & $+5.6$ [$-4.9$, $+15.5$] & $+1.1$ [$-2.9$, $+4.6$] \\
nuScenes, 228 scenes & $+76.3$ & $+27.2$ [$+8.2$, $+49.8$] & $+32.3$ [$+10.2$, $+59.3$] & $+31.2$ [$+10.0$, $+57.7$] & $+31.9$ [$+11.2$, $+58.9$] & $+4.7$ [$-2.7$, $+15.1$] \\
\bottomrule
\end{tabular}}
\end{table}

\paragraph{Closed loop on nuPlan under replayed perception.} The 60 scenarios of the nuPlan track are simulated in closed loop with PDM-Closed or IDM, with detector outcomes replayed by logged time: at each step the observation is built as in Appendix~\ref{app:external} from the logged frame at that time, so this measures closed-loop dynamics under replayed perception rather than closed-loop sensing. From each logged state with non-zero single-step value (338 for PDM-Closed, 66 for IDM), a fork uses the full mode at the first step and the cheap mode afterwards, and its safety loss summed over the next 8\,s is compared with that of an all-cheap fork. Of the 193 PDM-Closed forks with a complete 8\,s window, on 10 logs, 63.7\% leave that loss unchanged, and among the 70 that change it the direction agrees with the sign of the single-step value in 57.1\% $[40.5,73.2]$ (bootstrap over logs). The interval contains one half, so the single-step sign neither predicts nor contradicts the closed-loop direction here. For IDM the agreement is 6 of 13.

\paragraph{Closed loop in CARLA.} In Bench2Drive \citep{jia2024bench2drive} on CARLA \citep{dosovitskiy2017carla}, over 219 routes with paired route means, PDM-Lite \citep{beisswenger2024pdmlite} with cheap perception loses 4.2 driving-score points $[1.7,6.8]$ against privileged perception, and full resolution recovers $+0.9$ $[-0.6,+2.5]$. In forked rollouts at 172 frames on 37 routes, one escalated frame changes the summed braking loss over the next 8\,s by a median of 83, against 63 between two runs of the identical cheap policy, and only 6 of the 172 changes exceed the 95th percentile of those run-to-run differences, so the effect of a single escalation is not distinguishable from run-to-run variation here. Routing at a 20\% budget on 169 test routes, with a gate fitted on 50 development routes, gives $+0.36$ $[-0.63,+1.47]$ driving score against random, not resolved (minimum detectable difference 1.2--2.6), and the gate escalated 29.3\% of frames against random's 20.1\%, so the budgets were not matched.

\subsection{Other registered experiments}
\label{app:fu-other}

The released outputs report these in full. \paragraph{{Other detectors and planners.}} Replacing the detector harms 37--52\% of affected braking frames in ten pairs on the test scenes{, over monocular and oracle geometry}; with both detectors fine-tuned on nuScenes, the 320$\to$640 pair harms 53.6\%. {Under monocular geometry alone, the harmed share is 40.7--52.4\% in all ten pairs, and it moves little with detection quality. The three upgrades from YOLOv8s (recall 0.50 to 0.54--0.56) harm 40.7--44.2\%, each interval below one half, while the two replacements that change recall by less than 0.01 (YOLOv8m to YOLO11m, YOLOv8l to YOLOv8x) harm 52.4\% $[48.3,56.3]$ and 45.5\% $[41.4,49.7]$. Magnitudes move more: the three upgrades have harm ratios of 0.17--0.24 and all-full gains of $+30.7$ to $+36.6\%$, each interval above zero, against harm ratios of 0.74 and 1.83 and all-full gains of $+4.9\%$ and $-9.3\%$ for the two near-equal replacements. Better detection thus shifts the balance toward benefit mainly through magnitude, and still harms about two in five affected decisions.} With published LiDAR detectors, CenterPoint with a pillar backbone and VoxelNeXt \citep{yin2021centerpoint,chen2023voxelnext}, on the official validation split, 35.6\% $[27.6,44.4]$ are harmed, with all-full $+4.3\%$ $[+0.5,+9.5]$. Restoring the annotations that the reference's visibility filter omits leaves 51.0\% $[46.3,55.8]$ on the test scenes. On NAVSIM's navtest split \citep{dauner2024navsim}, PDM-Closed fed by YOLOv8s 320$\to$640 on the front camera reacts to the detector on 1{,}004 of 12{,}146 tokens over 136 logs, of which 62.1\% $[58.4,65.5]$ are harmed under the overall score, with all-full $+1.3\%$ $[-0.8,+3.3]$; 975 of those change ego progress, and 110 change the at-fault collision score, 42.7\% $[30.1,54.8]$ of them harmed. No allocator separates reliably from random there (2 of 54 cells, uncorrected), and IDM remains sparse. An end-to-end camera planner (Latent TransFuser \citep{chitta2023transfuser}, whose cheap mode halves each camera image's resolution per side) on the same tokens harms 42.1\% $[40.5,43.9]$ of its 7{,}596 affected tokens under the overall score, with all-full $-0.4\%$ $[-2.5,+1.4]$, and 55.1\% $[42.9,67.3]$ of the 89 that change the at-fault collision score.

\paragraph{{Experiments not used as evidence.}} Three experiments are not used as evidence: the sampling noise of Diffusion Planner \citep{zheng2025diffusion} is as frequent as its reaction to the detector, a depth network used without focal rescaling reads ranges 41\% long, and Argoverse stopped at its registered recall gate, so no decision value was computed.

\subsection{Results at scale under each check}
\label{app:fu-summary}

Table~\ref{tab:survive} collects the braking results of this appendix under each check; Table~\ref{tab:checks} in Section~\ref{sec:followup} reports a selection, with the NAVSIM planners.

\begin{table}[htbp]
\caption{{Braking results at scale under each check (monocular geometry; 95\% intervals). Harmed shares in \% of affected frames, gains in \% of the all-cheap loss, $D_t$ in points. One scene removed: the most negative or most positive scene (first row), else the row's most influential scene. Shortfall only: without the over-braking and change-of-command terms. Capped: $V_i^q$ clipped to $\pm c$, $c$ the 95th percentile of $|V_i^q|$ over affected fit frames. Moving lead: a lead vehicle may brake at 4\,m/s$^2$. $D_t$: decision-trained gate minus the same gate on a perception-level target ($^\ddagger$one interval reaches zero).}}
\label{tab:survive}
\centering
\scriptsize
\setlength{\tabcolsep}{3pt}

\begin{tabular}{lllll}
\toprule
Result & Standard & One scene removed & Capped & Moving lead \\
\midrule
Harmed share, 850 scenes & 46.9\% [44.6, 49.1] & 46.6\%, 47.1\% & --- & 47.0\% [44.9, 49.1] \\
Harmed share, shortfall only, 850 & 41.6\% [38.5, 44.6] & --- & --- & 42.9\% [39.8, 45.9] \\
Oracle@20, 850 scenes & $+31.4$ [$+25.7$, $+37.2$] & --- & --- & $+28.2$ [$+25.2$, $+31.3$] \\
Gate on $V_i^q$, realized, 228 & $+12.7$ [$+2.1$, $+30.0$] & $+5.9$ [$+1.1$, $+11.3$] & $+5.5$ [$+3.1$, $+8.6$] & $+8.4$ [$+4.6$, $+12.8$] \\
$D_t$, four targets, 228 & $+11.9$ to $+16.1$ & $+6.0$ to $+9.1^\ddagger$ & $+3.9$ to $+7.4$ & $+6.6$ to $+10.7$ \\
Criticality-only gate, 20\% latency & $+9.1$ [$+0.0$, $+22.3$] & $+3.6$ [$-0.8$, $+7.1$] & $+5.2$ [$+2.9$, $+7.6$] & $+8.8$ [$+5.3$, $+11.9$] \\
\bottomrule
\end{tabular}
\end{table}

}

\clearpage
{
\section{Glossary}
\label{app:glossary}

Table~\ref{tab:glossary} collects the terms used throughout the paper.

\begin{table}[!h]
\caption{{Terms used in the paper.}}
\label{tab:glossary}
\centering
\small

\begin{tabular}{@{}p{0.25\linewidth}p{0.70\linewidth}@{}}
\toprule
Term & Meaning \\
\midrule
Cheap and full mode & The two perception modes compared, for example YOLOv8s at 320 and 640\,px; the full mode costs more. \\
Escalation & Running the full mode on an input instead of keeping the cheap output. \\
Downstream system & The fixed decision rule that consumes the detections, such as the braking controller or a planner, together with its loss against the reference state. \\
Reference state & The annotated scene against which each decision's loss is computed; it never enters an allocator's input. \\
Unit & One evaluated input: a frame on nuScenes and KITTI, a state on nuPlan, a token on NAVSIM. \\
Decision value $V_i^q$ & The reduction in the downstream system's loss when input $i$ is escalated, all else fixed (Eq.~\ref{eq:intro-value}); it can be negative. \\
Affected, harmed share & An input is affected if $V_i^q\neq0$; the harmed share is the fraction of affected inputs with $V_i^q<0$. \\
Harm ratio $\rho_q$ & Total harm divided by total benefit (Eq.~\ref{eq:ndg}). \\
All-cheap, all-full & Running every input in the cheap mode, or escalating every input; gains are reported as reductions of the all-cheap loss. \\
Decision oracle & The set of at most $k$ inputs with the largest total decision value, the reference for allocation (Eq.~\ref{eq:oracle}). \\
Oracle@20 & The decision oracle escalating at most 20\% of the inputs, reported as the share of the all-cheap loss it removes. \\
Allocator & A rule that decides which inputs to escalate before escalation, from pre-escalation information only; gates and routers are learned allocators. \\
Deployable and diagnostic signals & A deployable signal uses only pre-escalation information: the input, the cheap output, previous frames, calibration and ego state. A diagnostic signal needs the full output or the reference state and never enters an allocator's input at test time. \\
Selection budget & At most $k$ escalations, reported as a fraction of inputs; the allocator is treated as free. \\
Latency or energy budget & A per-input budget $C_{\mathsf c}+\alpha C_{\mathsf f}$ in milliseconds or millijoules that also charges the allocator's own cost. \\
Realized gain & The total decision value of the escalated inputs, in \% of the all-cheap loss. \\
nDG & Realized gain divided by the oracle's gain at the same budget (Eq.~\ref{eq:ndg}). \\
Cell, track & A cell is one combination of dataset, detector pair, geometry and downstream system; a track fixes the perception modes, the downstream system, the split and the cost profile. \\
Frozen splits & The benchmark's splits, fixed before any baseline was evaluated (Section~\ref{sec:splits}); the tracks and settings evaluated on them are called frozen. \\
Not resolved & Said of a difference whose 95\% interval includes zero. \\
Post hoc & Said of an analysis specified after earlier results had been seen (Appendix~\ref{app:followup}). \\
\bottomrule
\end{tabular}
\end{table}
}

{
\section{Formulation, Cache and Tracks}
\label{app:layers}

\bench{} has three layers, and only the first is needed for the claims about decision value (Table~\ref{tab:layers}). The formulation needs two perception outputs for the same input, one fixed decision rule and a loss against the reference scene; the harmed shares of Table~\ref{tab:ceiling} involve no allocator, budget, cost model or ranking. A downstream system cannot be removed: the value of a detection is defined only through a decision rule and a loss, and of the 79 nuScenes frames that change the losses of two systems, 40 are helped by one and harmed by the other. The released cache is the reusable layer. The tracks fix further conventions so that reported numbers are comparable, and they are one instantiation; Appendices~\ref{app:fu-lead} and~\ref{app:fu-seq} use the cache without them.

\paragraph{What the tracks assume about a solution.} An allocator chooses among perception modes that already exist, it leaves the downstream system unchanged, and its benefit is credited at the current step. Methods that redesign the perception model or train it jointly with the planner are outside what the tracks score.

\begin{table}[!h]
\caption{{The three layers of \bench{}: what each contains and whether the paper's claims depend on it.}}
\label{tab:layers}
\centering
\footnotesize

\begin{tabular}{@{}p{0.14\linewidth}p{0.45\linewidth}p{0.35\linewidth}@{}}
\toprule
Layer & Contents & Role \\
\midrule
Formulation & Decision value (Eq.~\ref{eq:intro-value}): two perception outputs for one input, a fixed decision rule, a loss against the reference scene. For allocation: the decision is made without the expensive output, its cost is charged, and the reference is an oracle at the same budget (Eqs.~\ref{eq:oracle},~\ref{eq:multifid}). & Essential. The claims that value is signed, often negative and specific to the downstream system need nothing else. \\
\addlinespace
Released cache & Each mode's detections for every input, reference states, per-input decision values, measured costs and the downstream evaluators. & Reusable without the track conventions: other requirements and loss weights (Appendix~\ref{app:fu-lead}); decisions before any current-frame perception, against a sequence oracle (Appendix~\ref{app:fu-seq}). \\
\addlinespace
Track conventions & Detectors and resolutions; monocular geometry; the controllers and their loss weights; submission as a score or a mask; the cascade cost model; the fixed action history; the splits. & Incidental, fixed for comparability. The fixed history makes value additive over inputs, so the oracle is exact and needs no simulator; a track then credits one escalation at a time. \\
\bottomrule
\end{tabular}
\end{table}
}

\end{document}

%% file: math_commands.tex
\usepackage{amsmath,amsfonts,bm}

\def\eqref#1{equation~\ref{#1}}

\def\1{\bm{1}}

\DeclareMathAlphabet{\mathsfit}{\encodingdefault}{\sfdefault}{m}{sl}
\SetMathAlphabet{\mathsfit}{bold}{\encodingdefault}{\sfdefault}{bx}{n}

\newcommand{\E}{\mathbb{E}}